\documentclass[lettersize,journal]{IEEEtran}

\ifCLASSINFOpdf
\else
\fi
\usepackage[cmintegrals]{newtxmath}
\usepackage{url}
\usepackage{cite}
\usepackage{amsmath,amsfonts}

\usepackage{amsthm}
\usepackage{mathtools}
\usepackage{algorithmic}
\usepackage[ruled,linesnumbered]{algorithm2e}
\usepackage{booktabs}
\usepackage{multirow}
\usepackage{graphicx}
\usepackage{textcomp}
\usepackage[table,xcdraw]{xcolor}
\usepackage{bm}
\usepackage{enumitem}
\usepackage{subfigure}
\usepackage{url}
\usepackage{hyperref}
\usepackage{siunitx}
\usepackage[capitalize]{cleveref}
\usepackage{subfigure}
\usepackage{afterpage}
\usepackage[font=small]{caption}
\newtheorem{assumption}{Assumption}

\newtheorem{lemma}{Lemma}
\newtheorem{theorem}{Theorem}
\newtheorem{corollary}{Corollary}

\DeclareMathOperator{\KL}{KL}

\usepackage{soul}
\usepackage{thmtools}
\definecolor{hl}{RGB}{205, 232, 248}
\definecolor{lhl}{RGB}{234, 245, 250}
\definecolor{sb_orange_025}{RGB}{255, 223, 195}
\newcommand{\po}{\phantom{0}}

\newcommand{\ie}{\textit{i}.\textit{e}.}

\newcommand{\bfa}{\mathbf{a}}
\newcommand{\bfs}{\mathbf{s}}
\newcommand{\ba}{\mathbf{a}}
\newcommand{\bs}{\mathbf{s}}

\newcommand{\cmmnt}[1]{}
\newcommand{\tr}[1]{}

\newcommand{\transitions}{T}
\newcommand{\bellman}{\mathcal{B}}

\newcommand{\policy}{\pi}
\newcommand{\vpolicy}{{\pi^v_k}}

\newcommand{\behaviorv}{{\pi}_{\beta_k}}
\newcommand{\E}{\mathbb{E}}

\def\rva{{\mathbf{a}}}

\def\rvs{{\mathbf{s}}}

\allowdisplaybreaks[4]

\newboolean{showcomments}
\setboolean{showcomments}{true}
\newcommand{\comment}[1]{\ifthenelse{\boolean{showcomments}}
	{\textcolor{red}{(Comment: #1)}}
	{}
}
\newcommand{\red}[1]{\ifthenelse{\boolean{showcomments}}
	{\textcolor{red}{#1}}
	{}
}
\newcommand{\blue}[1]{\ifthenelse{\boolean{showcomments}}
	{\textcolor{black}{#1}}
	{}
}

\newboolean{showanswers}
\setboolean{showanswers}{true}
\newcommand{\answer}[1]{\ifthenelse{\boolean{showanswers}}
	{\textcolor{blue}{(Answer: #1)}}
	{}
}

\newtheorem{remark}{Remark} 
\newtheorem{definition}{Definition} 
\SetKwRepeat{Do}{do}{while}

\begin{document}

\title{FOVA: Offline Federated Reinforcement Learning 
with Mixed-Quality Data
\vspace{-0pt}}       
\author{
        Nan~Qiao,~\IEEEmembership{Student Member,~IEEE},
        Sheng~Yue,~\IEEEmembership{Member,~IEEE},
        and Ju Ren,~\IEEEmembership{Senior Member,~IEEE},
        and~Yaoxue~Zhang,~\IEEEmembership{Senior Member,~IEEE}
    \IEEEcompsocitemizethanks{
    \IEEEcompsocthanksitem 
    This research was supported in part by the National Natural Science Foundation of China under Grant Nos. 62572496 and 62432004, 
    by the Young Elite Scientist Sponsorship Program by CAST under Contract No. ZB2025-218, 
    by the Shenzhen Science and Technology Program under Grant No. JCYJ20250604175500001, 
    and by a grant from the Guoqiang Institute, Tsinghua University.
    \textit{(Corresponding author: Sheng Yue)}.
    \IEEEcompsocthanksitem Nan Qiao is with the School of Computer Science and Engineering, Central South University, Changsha, Hunan, 410083 China,
    and also 
    with the 
    School of Cyber Science and Technology, Sun Yat-sen University, Shenzhen, 518107 China.
    Email: \texttt{nan.qiao@csu.edu.cn}.
    \IEEEcompsocthanksitem Sheng Yue is with the 
    School of Cyber Science and Technology, Sun Yat-sen University, Shenzhen, 518107 China. Email: \texttt{yuesh5@mail.sysu.edu.cn}.
    \IEEEcompsocthanksitem Ju Ren and Yaoxue Zhang are with the Department of Computer Science and Technology, BNRist, Tsinghua University, Beijing, 100084 China, and also with Zhongguancun Laboratory, Beijing, 100094 China. Emails: \texttt{\{renju, zhangyx\}@tsinghua.edu.cn}.
    }  
}
\maketitle

\begin{abstract}
Offline Federated Reinforcement Learning (FRL), a marriage of federated learning and offline reinforcement learning, has attracted increasing interest recently. 
Albeit with some advancement, we find that the performance of most existing offline FRL methods drops dramatically when provided with mixed-quality data, that is, the logging behaviors (offline data) are collected by policies with varying qualities across clients. 
To overcome this limitation, this paper introduces a new vote-based offline FRL framework, named FOVA. It exploits a \emph{vote mechanism} to identify high-return actions during local policy evaluation, alleviating the negative effect of low-quality behaviors from diverse local learning policies.
Besides, building on advantage-weighted regression (AWR), we construct consistent local and global training objectives, significantly enhancing the efficiency and stability of FOVA.
Further, we conduct an extensive theoretical analysis and rigorously show that the policy learned by FOVA enjoys strict policy improvement over the behavioral policy.
Extensive experiments corroborate the significant performance gains of our proposed algorithm over existing baselines on widely used benchmarks.
Code is available at:
\url{https://sites.google.com/view/fova}.
\end{abstract}


\section{Introduction}
\label{sec:introduction}


\IEEEPARstart{W}{ith} the growing prevalence of Artificial Intelligence of Things (AIoT) applications, a large number of edge devices involving intelligent decision-making have emerged \cite{lin2024awq,khan2022level,lim2020federated,paoi2024}, such as self-driving cars and mobile agents \cite{liu2019edge,ndikumana2020deep,yang2023path}. 
\blue{
To enable decision-making capabilities in edge devices, Federated Reinforcement Learning (FRL) \cite{kiran2021deep, masoudi2020device,yue2024CJE,kaelbling1996reinforcement}  combines Federated Learning (FL) \cite{wang2024dfrd,mcmahan2017communication} and Reinforcement Learning (RL) \cite{agarwal2019reinforcement,offtutorial,matsuo2022deep}. This approach allows for collaborative policy training in a distributed manner without sharing raw data \cite{woo2023blessing,lim2021federated}.}
In a standard FRL setting, each client independently interacts with the environment and updates the policy locally \cite{wang2024momentum,zhang2024finitetime}.
\blue{
These updates are then uploaded to a centralized server for policy aggregation \cite{qi2021federated, wang2020federated, jin2022federated}.}

Unfortunately, owing to environmental uncertainties, interactions with the real world can be costly and risky, such as in healthcare and autonomous driving \cite{levine2020offline,prudencio2023survey}.
Accordingly, \textit{offline FRL} \cite{yue2024federated}, which enables clients to learn the decision policy using only precollected static data with no further online environmental interaction, has garnered increasing attention in recent years \cite{zhou2024federated,rengarajan2023federated,qi2021federated}. Existing studies have attempted to tackle the necessary challenges in offline FRL \cite{park2022federated,wen2023federated,woo2024federated,yue2024federated, rengarajan2023federated,yue2024CJE}. 
One major challenge is the mismatch between the state-action distributions of the updating policy and the offline dataset, easily causing large extrapolation errors during offline training \cite{qi2021federated}. 
To mitigate this problem,
several efforts incorporate conservatism terms into the training loss to ensure the learning policy stays close to the behavior policy
\cite{park2022federated,wen2023federated}. Another critical issue is the limited state-action coverage of each client's local dataset, which can lead to unreliable behavior estimates in unseen environments \cite{woo2024federated}.
Some existing works propose to constrain the distribution of the learning policy around the average visitation distributions of clients to remedy this issue
\cite{shen2023distributed}.
\blue{
Other recent works explore bi-regularization techniques to manage and calibrate these distributional shifts \cite{yue2024federated, rengarajan2023federated}.}

\begin{figure}[tbp]
    \centering
        \subfigure{
                \begin{minipage}[htbp]{0.475\columnwidth}
                    \centering	
                    \label{fig:intro-mixed}
                    \includegraphics[width=\textwidth]{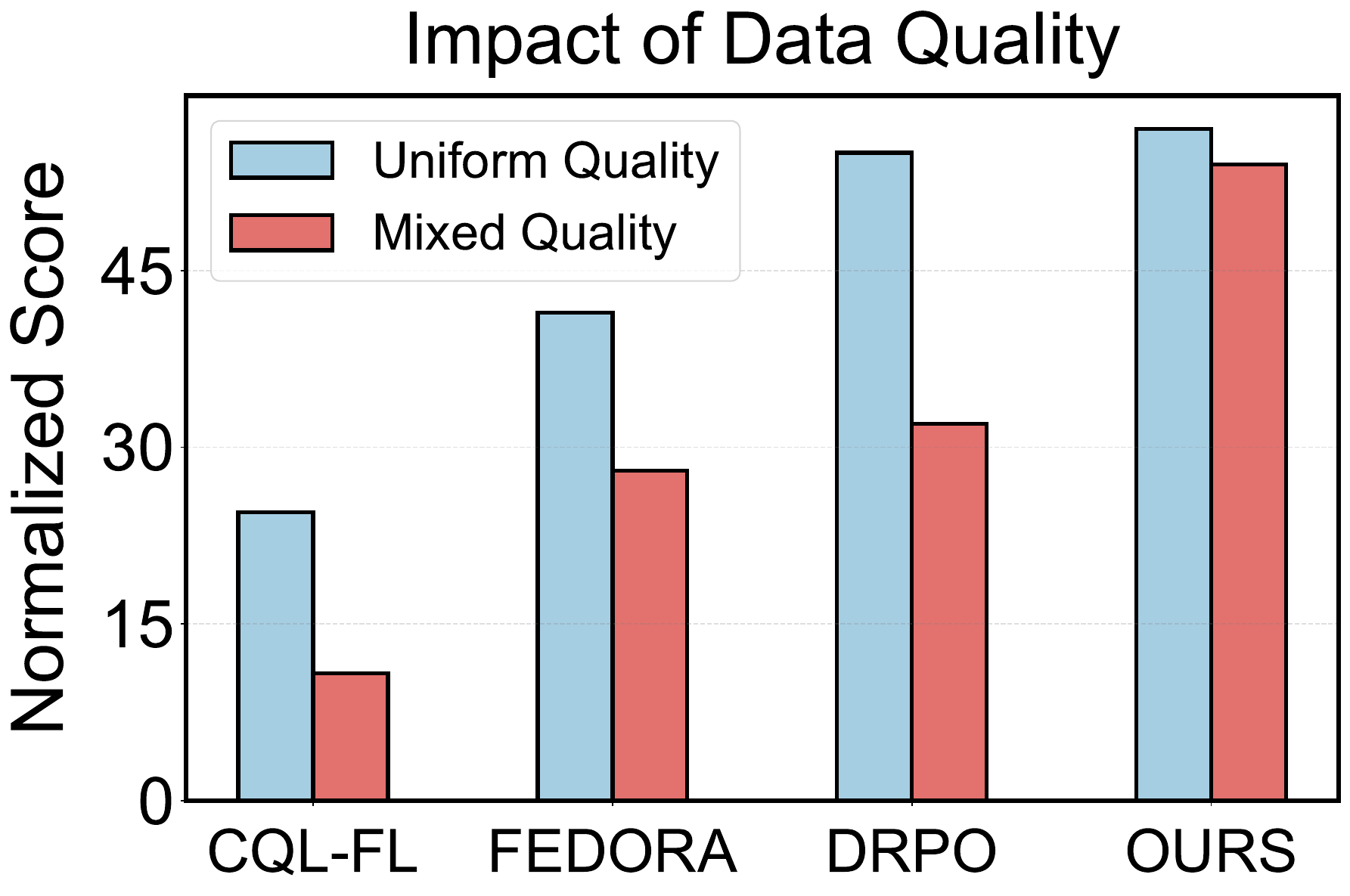}
              	\end{minipage}
        }
        \hspace{-0.02\textwidth}
        \subfigure{
                \begin{minipage}[htbp]{0.475\columnwidth}
                    \centering	
                    \label{fig:intro-incon}
                    \includegraphics[width=\textwidth]{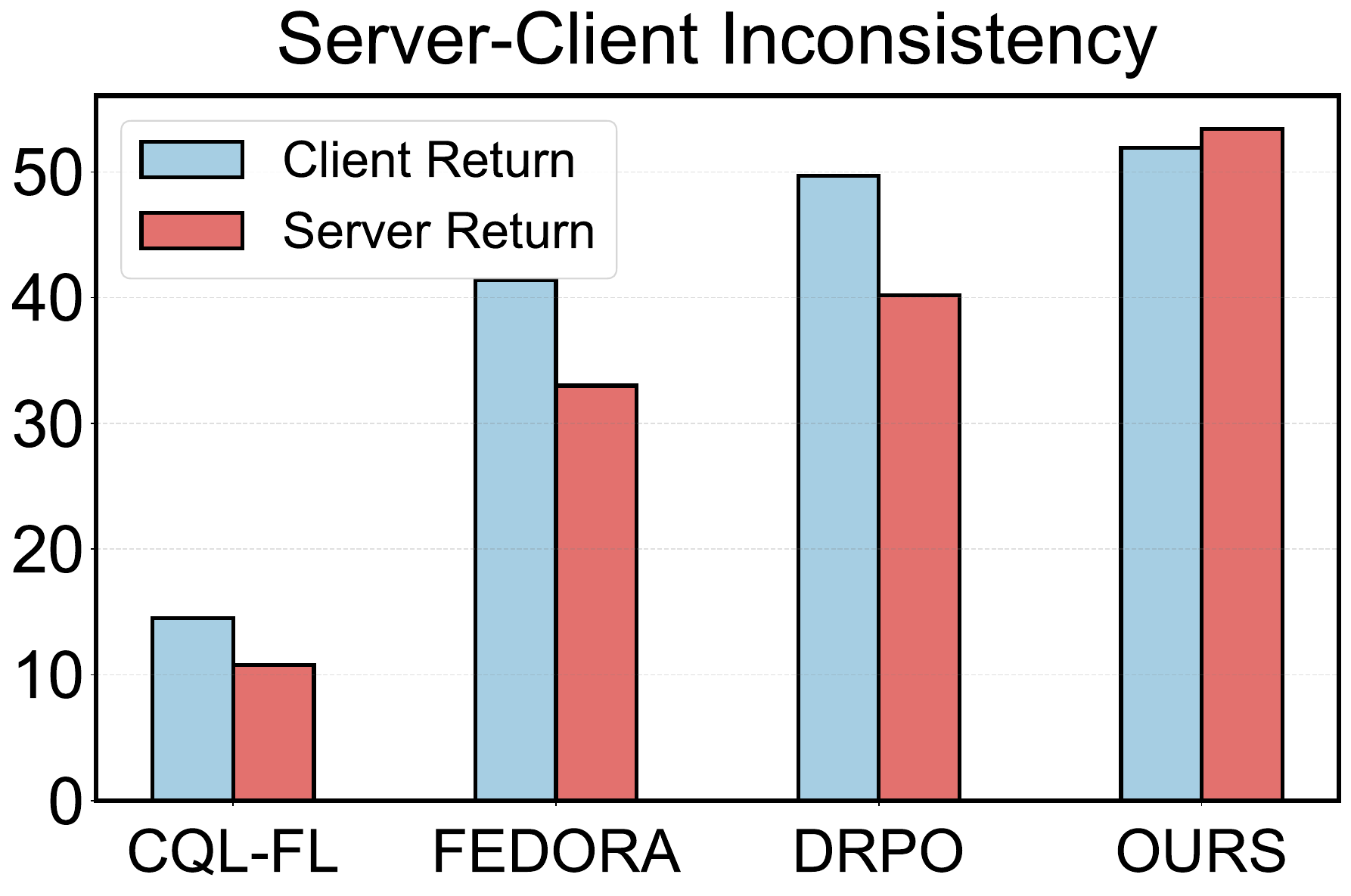}
              	\end{minipage}
        }
        
\caption{
Results of CQL-FL \cite{park2022federated}, FEDORA \cite{rengarajan2023federated}, DRPO \cite{yue2024federated} and the proposed method on Gym locomotion tasks. Left: 
Performance comparison between mixed-quality and uniform-quality (commonly used as evaluated benchmarks in the existing literature) cases. 
Right: 
Performance comparison between the average client-side (local policies) return and server-side (global policy) return.
 }
 \vspace{-17pt}
 \label{fig:intro}
\end{figure}

\blue{
Albeit with these empirical successes, we find that most existing offline FRL methods suffer from 
significant performance degradation with \textit{mixed-quality} data. This occurs because logging behaviors (offline data) are collected by policies with varying quality across different clients.}
As depicted in Fig. \ref{fig:intro-mixed}, the performance of existing approaches with mixed-quality data drops by approximately 40\% and exhibits more instability, in sharp contrast to those with uniform quality (see Section II in the supplementary file for more results) \cite{qiao2025fova_supp}.  
The rationale is that, in the mixed-quality case, some clients have low-quality offline data, which inevitably induces low-quality learning policies during the local updating phase. Due to the lack of an effective mechanism to evaluate policy quality, these low-quality policies would `contaminate' the global policy during the aggregation, ultimately resulting in the degradation of overall performance.

In addition, the existing methods are typically limited by \textit{inconsistent} optimization objectives. 
That is, in local updating, each client updates the local policy for improvement over its local logging behaviors (varying qualities across clients).
\blue{
However, in the global updating, they simply aggregate the local policies by averaging them as in standard FL.
}
This would lead to large biases in optimizing the global policy.  To corroborate the negative impact of this inconsistency, we conducted a series of experiments comparing the average client and server returns across existing algorithms.
As illustrated in Fig.~\ref{fig:intro-incon}, during training, the performance of the aggregated policy is consistently inferior to the average of local updating policies (see Section I in the supplementary file for more results) \cite{qiao2025fova_supp}.
This server-client objective misalignment would significantly affect the training efficiency and effectiveness. 
To overcome these limitations, this paper introduces a new offline FRL method, named
\textit{\underline{F}ederated \underline{O}ffline reinforcement learning with \underline{V}ote mechanism and \underline{A}dvantage-weighted regression} (FOVA), 
which encompasses two key components: the vote mechanism and the advantage-weighted regression.
Specifically, the vote mechanism is designed to 
select the behavior with the maximum Q-value from the global, local, and local behavioral policies. It is thereby capable of identifying high-return actions during local policy
evaluation, alleviating the negative effect of low-quality behaviors
from diverse local learning policies.
Further, leveraging the technique of advantage-weighted regression (AWR), we devise consistent and tractable local and global policy improvement objectives, largely improving training efficiency and stability. 
Building on this, our method significantly surpasses baseline performance.

In a nutshell, our main contributions are summarized below:
\begin{itemize}
    \item
    \blue{
    We propose a new offline FRL method, FOVA, which can be efficiently implemented in parallel across clients, with mixed-quality behavioral data.
    FOVA can mitigate the negative effects of low-quality behaviors across clients via a vote mechanism, while significantly enhancing the training efficiency and stability policy by inconsistent optimization objectives.
    } 
    \item 
    
    We conduct extensive theoretical analysis of the proposed method. We prove that FOVA learns a conservative Q-function, 
    enabling the learned policy to safely explore out-of-distribution state-action pairs from local data. Additionally, FOVA achieves strict performance improvements regardless of the quality of local datasets.
    
    \item 
    Through extensive simulations, we demonstrate that FOVA significantly outperforms existing methods on standard offline RL benchmarks, achieving average score improvements of 20\%–50\%. Furthermore, FOVA attains improvements of up to 56\%, in mixed-quality data cases.
\end{itemize}

The remainder of the paper is organized as follows. Section \ref{sec:RelatedWork} briefly reviews related work. Section \ref{sec:preliminaries} introduces the preliminaries, including FRL, offline FRL, and the system model. In Section \ref{sec:proposed_method}, we describe the proposed vote mechanism and AWR method to tackle mixed-quality data and inconsistent optimization objective problems, respectively. Section \ref{sec:improve_bp} presents the theoretical analysis. Section \ref{sec:simulation} showcases the simulation results, , followed by the discussion of limitations and future work in Section \ref{sec:future-work}, and the concluding remarks in Section \ref{sec:conclusion}.


\section{Related Work}\label{sec:RelatedWork}

\blue{As a key enabling technology, FRL has become increasingly important in the field of edge intelligence \cite{yu2020deep,mills2019communication,zhang2024distributed}.}
Early studies primarily focused on online FRL, where intelligent agents interact with the environment \cite{huang2024multi,wu2024fedcache,wu2024netllm,wang2023uav}.
However, with the development of offline RL techniques, which allow training with offline data without the need for environment interaction, offline FRL has garnered growing research attention  \cite{rengarajan2023federated,park2022federated}.


\textbf{Online FRL.} 
\blue{
FL was first introduced with the Federated Averaging algorithm, enabling collaborative model training without sharing data, thus addressing privacy concerns in distributed settings \cite{mcmahan2017communication}. 
}
\blue{
Initially focused on supervised learning, recent work has shifted towards applying FL to online RL, driven by the increasing demand for intelligent decision-making systems that can interact with dynamic environments \cite{wang2024dfrd,hebert2022fedformer,lim2021federated}.}
Some works in FRL have focused on practical applications, such as combining Proximal Policy Optimization with Federated Averaging on IoT devices \cite{lim2020federated} and adapting Deep Deterministic Policy Gradient for autonomous driving scenarios \cite{liang2022federated}.
Additionally, recent research has delved into theoretical aspects, addressing challenges like adversarial attack resilience \cite{wu2021byzantine}, environmental and task heterogeneity \cite{wang2023federated}, and the complexities of sample and communication under asynchronous or online sampling \cite{woo2023blessing}.
\blue{
Despite these innovations, many FRL methods have focused primarily on online RL, where interactions with the environment are often costly and risky \cite{xie2022fedkl,woo2024federated}.}
This has prompted increased attention to offline RL \cite{sutton2018reinforcement,prudencio2023survey}.

\blue{\textbf{Offline RL.} Offline RL focuses on learning policies from static datasets, addressing challenges like distribution shifts, where the state-action visitation distribution of the learned policy diverges from that of the behavior policy used to generate the data \cite{levine2020offline,peng2019advantage,kostrikov2021iql,kumar2019stabilizing,fujimoto2021minimalist,kumar2020conservative,lyu2022mildly,an2021uncertainty}. Such discrepancies can significantly degrade policy performance. 
To counteract this, offline RL typically employs conservatism regularization, ensuring that the learned policy does not overfit to out-of-distribution actions. Several methods achieve this by keeping the learned policy close to the behavior policy \cite{peng2019advantage,kostrikov2021iql,kumar2019stabilizing,fujimoto2021minimalist} or by penalizing excessive optimism regarding out-of-distribution actions \cite{kumar2020conservative,lyu2022mildly,an2021uncertainty}. 
These approaches have been further adapted to combine the privacy-preserving advantages of FL with offline RL techniques, leading to the development of offline FRL \cite{park2022federated,woo2024federated,yue2024federated,rengarajan2023federated}
}

\textbf{Offline FRL.}
\blue{
Offline FRL has recently gained significant attention due to its unique challenges and potential.
}
Park et al. and Yue et al. extended FRL to offline data scenarios, utilizing conservative regularization terms to mitigate extrapolation errors caused by the lack of environmental interactions \cite{park2022federated,yue2024CJE}. 
Woo et al. designed the sample-efficient FedLCB-Q method on the server side, addressing data coverage issues due to privacy constraints that prevented data sharing \cite{woo2024federated}.
Yet, they overlook the distributional differences between local policy and global policy, which have been shown to cause significant performance degradation.
Additionally, Yue et al. and Rengarajan et al. identified distribution shifts between local policies and global policy, using regularization terms to constrain the differences in policy probability distributions \cite{yue2024federated,rengarajan2023federated}. 
Nevertheless, they neglected the fact that both local and global policies are easily influenced by low-quality logging behaviors, which would be further exacerbated during continuous local training and global aggregation, ultimately leading to unsatisfactory policy performance. 
Moreover, existing work exhibits inconsistency between the optimization objective of the local policy and the actual optimization target of the global policy.
Our approach tackles these problems by employing the vote mechanism and advantage-weighted regression method, respectively.

\section{Preliminaries}\label{sec:preliminaries}
In this section, we first introduce preliminaries for FRL and offline FRL, followed by the system model and problem formulation.

\subsection{Federated Reinforcement Learning}

In RL, the environment is modeled as a Markov Decision Process (MDP), defined by the tuple $\mathcal{M} = (\mathcal{S}, \mathcal{A}, T, r, \mu_0, \gamma)$. Here, $\mathcal{S}$ is the state space, $\mathcal{A}$ is the action space, $T$ represents state transition probabilities, $r$ is the reward function, $\mu_0$ is the initial state distribution, and $\gamma \in (0,1)$ is the discount factor.

\blue{
In the \textit{actor-critic} methods, a policy (denoted as $\pi(\ba|\bs)$) is explicitly trained to maximize the expected return $J(\mathcal{M},\pi) := \frac{1}{1-\gamma}\mathbb{E}_{(\bs,\ba)\sim d_{\mathcal{M}}^\pi(\bs,\ba)}[r(\bs,\ba)]$, where $d_{\mathcal{M}}^\pi$ represents the state-action visitation distribution, through alternating phases of policy evaluation and policy improvement.
We assume $|r(\bs,\ba)|\le r_{\mathrm{max}}$}.
\blue{During policy evaluation, the Q-function $Q^\pi$ is updated using the Bellman operator, $\mathcal{B}^\pi Q = r + \gamma P^\pi Q.$
Specifically, $P^\pi$ denotes the transition probability matrix under policy $\pi$, defined as $P^\pi Q(\bs, \ba) = \E_{\bs' \sim T(\cdot | \bs, \ba), \ba' \sim \pi(\cdot|\bs')} [Q(\bs', \ba')].$}
In the policy improvement phase, the policy is refined to favor actions with higher expected Q-values.

FRL carries out this process in a distributed manner, leveraging the paradigm of FL, without exchanging local data.
The objective of FRL is to accelerate policy search and enhance sampling and communication efficiency through client collaboration while maintaining strict adherence to information privacy requirements.


\subsection{Offline Federated Reinforcement Learning}

The objective of offline FRL is to learn a global policy $\bar\pi$ from offline data $\{\mathcal{D}_k\}_{k \in \mathcal{K}}$ distributed across clients without interacting with the environment. Each offline dataset $\mathcal{D}_k$ is collected using a behavior policy $\pi_{\beta_k}$, with its discounted marginal state distribution denoted as $d_{\pi_{\beta_k}}$.
The offline mixed-data over all clients can therefore
be regarded as induced by a mixture of behavior policies $\{\pi_{\beta_k}\}_{k\in\mathcal{K}}$ \cite{jia2024offline}.
We characterize the data quality of client $k$ by the
expected return $J(\mathcal{M},\pi_{\beta_k})$. 

\blue{Furthermore, offline FRL faces significant challenges due to distribution shifts, which include discrepancies between the distributions of local policy \(\pi_k\) and the behavior policy \(\pi_{\beta_k}\), as well as variations in behavior policy distributions across different clients }\cite{woo2024federated,zhou2024federated,rengarajan2023federated}.
To overcome these challenges, current approaches design dual regularization during local updates, that is,
\begin{align}
    \max_{\pi_k}\mathbb{E}_{\bs\sim\mathcal{D}_k,\ba\sim\pi_k(\cdot|\bs)}\big[\hat{Q}^{\pi_k}(\bs,\ba)\big]
    - \alpha_1 D(\pi_k,\policy_{\beta_k}) - \alpha_2 D(\pi_k,\bar{\pi}),
    \nonumber
    \label{eq:local_updating_pi}
\end{align}
where $\alpha_1$ and $\alpha_2$ are positive coefficients, and $D(\cdot,\cdot)$ denote a divergence measure such as the total variance (TV) distance or the L2-norm distance for the constrained policy distribution shift \cite{yue2024federated, rengarajan2023federated}. 
These regularization terms attempt to strike a balance between alignment with the local behavior policy $\pi_k$ and the global policy $\bar\pi$ aggregated in the server.

However, the dual regularization approaches address the aforementioned distribution shift without taking into account the quality of local datasets. This results in poor performance in the case of mixed-quality behavioral data. 
Moreover, existing methods only seek to find multiple local policies instead of optimizing the global policy generated from server aggregation. The global policy should have been the optimization objective of offline FRL. This leads to the problem of inconsistent optimization objectives. 
In Appendix \ref{app:app-challenges}, we delve deeper into these problems.

\begin{figure}[t!]
    \setlength{\abovecaptionskip}{-0cm} 
    \setlength{\belowcaptionskip}{-0cm} 
    \centering
	\includegraphics[width=0.4\textwidth]{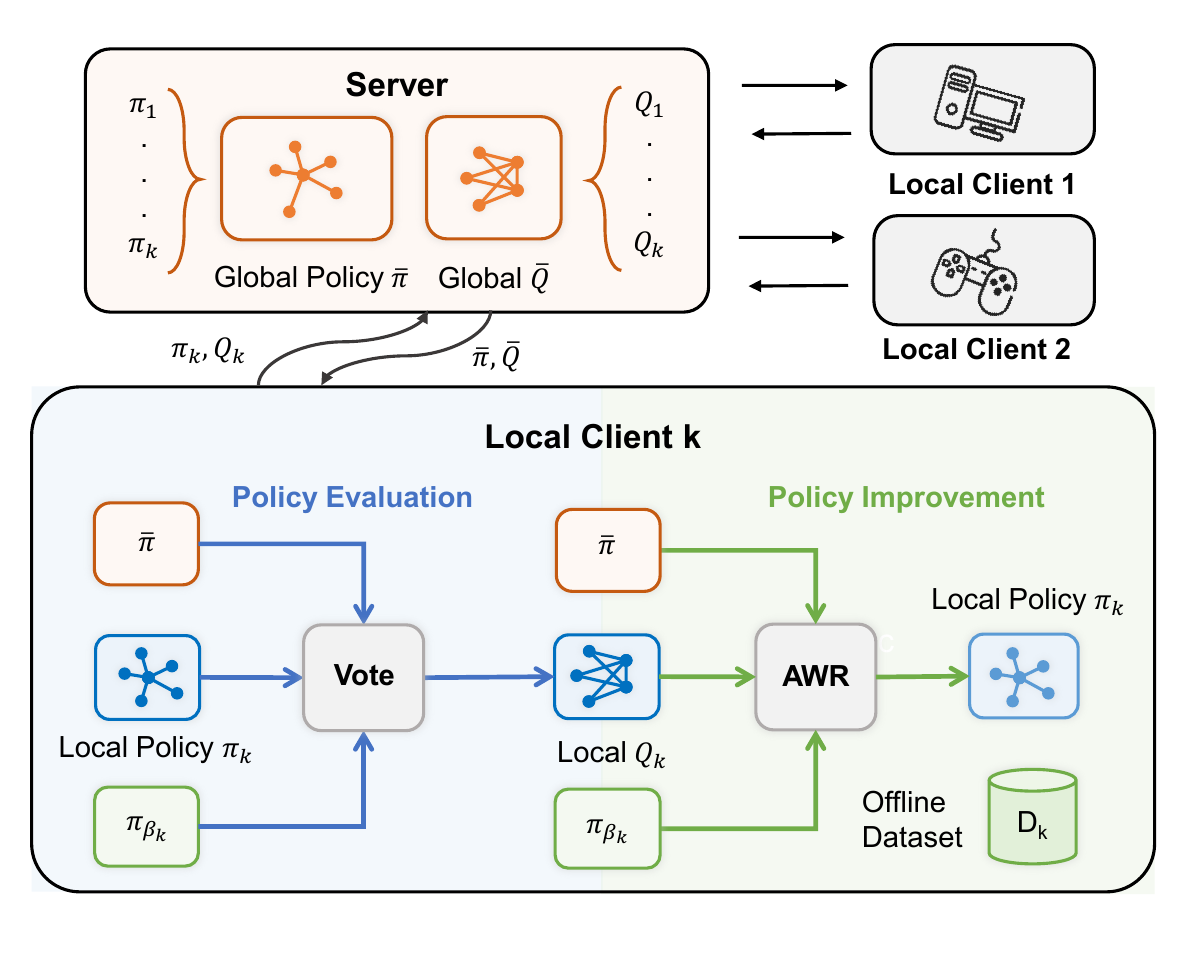}
    \caption{Illustration of offline FRL training procedure,
    demonstrating how clients collaborate training based on our FOVA method.}
    \label{fig:system-model}
\end{figure}
 
\subsection{System Model}\label{subsec:system_model}
The system model depicted in Fig.~\ref{fig:system-model} consists of a collection of clients, indexed by $\mathcal{K}$, representing $K$ entities that interface with a central server. Each client, denoted by $k \in \mathcal{K}$, possesses an individual offline dataset $\mathcal{D}_k = \{(s_i, a_i, r_i, s'_i, a'_i)\}_{i=1}^{D_k}$ collected using its own behavior policy $\pi_{\beta_k}$. Specifically, during the local update phase, we have designed a vote mechanism and an AWR method.
In the local policy evaluation phase, we implement the vote mechanism to identify and select better behavior across $\bar{\pi}$, $\pi_k$, and $\pi_{\beta_k}$, thereby training the local $Q_k$ function. In the local policy improvement phase, we develop the AWR method, which fully utilizes $\bar{\pi}$, $Q_k$, and $\pi_{\beta_k}$ to train the local policy function $\pi_k$. Subsequently, the parameters $\pi_k$ and $Q_k$ generated in the client $k$ are transmitted to the central server. The server then performs global aggregation to obtain the global parameters $\bar{\pi}$ and $\bar{Q}$, which are distributed back to the local clients.
For clarity, we have compiled the key notations used in this paper and presented them in Table~\ref{table:notations}.
\begin{table}[ht]
\small
    \caption{Key Notations}
    \label{table:notations}
    \centering
    \renewcommand\arraystretch{1.25}
    \resizebox{\columnwidth}{!}{
        \begin{tabular}{l|l}
            \hline
            \textbf{Notation} & \textbf{Definition} \\ \hline
            \(M\) & Markov Decision Process \((S, A, T, R, \mu_{0}, \gamma)\) \\
            \(S\) & state space of the MDP \\
            \(A\) & action space of the MDP \\
            \(T\) & state transition probabilities in the MDP \\
            \(R\) & reward function in the MDP \\
            \(\mu_{0}\) & initial state distribution in the MDP \\
            \(\gamma\) & discount factor in the MDP \\
            \(\pi(\bfa | \bfs)\) & probability of taking action \(\bfa\) under the input state \(\bfs\) \\
            \(Q^{\pi}(\bfs,\bfa)\) & Q function of policy \(\pi\) at state \(\bfs\) and action \(\bfa\)\\
            \(V^{\pi}(\bfs)\) & value function of policy \(\pi\) at \(\bfs\)\\
            \(A(\bfs, \bfa)\) & advantage function, calculated by \(Q(\bfs, \bfa)-V(\bfs)\) \\
            \(\mathcal{B}^{\pi}\) & actual Bellman operator under \(\pi\) \\
            \(k\) & the index of client \\
            \(D_{k}\) & offline dataset of client \(k\) by \(\pi_{\beta_{k}}\) \\
            \(\pi_{\beta_{k}}(\bfa | \bfs)\) & behavior policy for \(D_{k}\) \\
            \(d_{\pi_{\beta_{k}}}(\bfs)\) & discounted marginal state distribution of \(\pi_{\beta_{k}}(\bfa | \bfs)\) \\
            \(\pi_{k}(\bfa | \bfs)\) & local policy of client \(k\) \\
            \(\overline{\pi}(\bfa | \bfs)\) & global policy obtained in server aggregation \\
            \(\pi_{k}^{v}(\bfa | \bfs)\) & vote policy for max Q value behavior \\
            \(\hat{Q}^{\tau}(\bfs,\bfa)\) & estimated $Q(\bfs,\bfa)$ function at iteration \(\tau\) \\
            \(\hat{\mathcal{B}}^{\pi}\) & empirical Bellman operator under \(\pi\) \\
            \(\beta\) & positive temperature parameter in AWR \\
            \(\lambda\) & positive weighting parameter in AWR \\
            \(Z(\bfs)\) & partition function in AWR \\
            \hline
        \end{tabular}
    }
\end{table}

\section{Proposed Method}\label{sec:proposed_method}
In this section, 
we introduce FOVA, detailing two primary mechanisms in local updates: the vote mechanism and the AWR method. 
These approaches address the problems observed in existing methods and provide a more effective framework for policy learning in offline FRL. 
Additionally, we provide theoretical analyses for each of the two methods.

\subsection{Local Policy Evaluation with Vote Mechanism}\label{subsec:vp}

As illustrated before, existing methods often experience performance degradation in scenarios with mixed-quality data due to a lack of the mechanism to assess behavior from policies. 
To this end, we design the vote policy \(\vpolicy\), which 
selects the best behavior from mixed-quality policies. Specifically, vote policy \(\vpolicy\) chooses 
the behavior with the max Q-value from \(\pi_k\), \(\pi_{\beta_k}\), and \(\bar\pi\), which is defined as follows:
\begin{equation}
\label{eq:vote_policy}
\pi^v_k(\bfa|\bfs)={\arg\max}_{\pi^{\prime}_k\in\{\pi_k,\pi_{\beta_k},\bar{\pi}_k\}}  Q_{\bfa \sim \pi^{\prime}_k(\bfa|\bfs)}(\bs, \ba) ,~~ \forall \bfs.
\end{equation}
Notably, our method does not require training a new policy $\vpolicy$. Instead, it selects the optimal behavior and its corresponding Q-value by identifying the maximum Q-value.
This can be implemented with just a few lines of code. 
\blue{From an expectation perspective, the expected Q-value under the state–action distribution induced by a policy recovers the previously defined return $J(\mathcal{M},\pi)$, which we use to measure data quality. Therefore, choosing the action with the largest Q-value at each state can be viewed as locally selecting behaviors from higher-value policies.}
By integrating evaluations of additional \(\pi_{\beta_k}\) and \(\bar\pi\), this approach effectively enhances Q-function performance.
The Q-learning with vote policy is:
\begin{align}
\label{eq:vql}
\hat Q^{\tau+1} \leftarrow \arg \min_{Q}~ {\mathbb{E}_{\bfs, \bfa, \bfs^{\prime} \sim \mathcal{D}_{k}}\left[\left(Q(\bfs, \bfa)-\hat{\mathcal{B}}^{\red{\pi^v_k}} \hat Q^{\tau}(\bfs, \bfa)\right)^{2}\right]}.
\end{align}
However, such a Q-learning update fails to effectively address the problem of extrapolation error. Extrapolation error poses a significant challenge in local policy evaluation because offline RL training lacks interaction with the environment \cite{prudencio2023survey}.
Consequently, there is a mismatch between the distribution of state-action pairs accessed by the current policy and those sampled from the dataset \cite{levine2020offline}.
To mitigate this problem, we incorporate conservative Q-learning (CQL) regularization terms into the Q-learning update \cite{kumar2020conservative}.
These terms effectively penalize Q-values for actions outside the dataset policy, serving as a classical regularization approach in offline RL. Additionally, we integrate the vote policy into the CQL regularization terms.
To this end, the vote-based conservative Q-learning (VCQL) is as follows (changes from CQL in \red{red}):
\begin{align}
\label{eq:vcql}
&\hat Q^{\tau+1} \leftarrow \arg \min_{Q}~ {\frac{1}{2} \mathbb{E}_{\bfs, \bfa, \bfs^{\prime} \sim \mathcal{D}_{k}}\left[\left(Q(\bfs, \bfa)-\hat{\mathcal{B}}^{\red{\pi^v_k}} \hat Q^{\tau}(\bfs, \bfa)\right)^{2}\right]}\nonumber\\
&+{{\alpha}\left(\mathbb{E}_{s \sim \mathcal{D}_{k}, \bfa \sim \red{\pi^v_k(\bfa|\bfs)}}[Q(\bfs, \bfa)]-\mathbb{E}_{\bfs, \bfa \sim \mathcal{D}_{k}}[Q(\bfs, \bfa)]\right)},
\end{align}
where $\mathbb{E}_{\bfs, \bfa, \bfs^{\prime} \sim \mathcal{D}_{k}}\left[\hat{\mathcal{B}}^{\pi^v_k} Q^{\tau}_{k}(\bfs, \bfa)\right] = \gamma\ \mathbb{E}_{\mathbf{\mathbf{a'} \sim \pi^v_k (\cdot \mid \mathbf{s'})}}\left[ Q^{\tau}_{k}\left(\mathbf{s'}, \mathbf{a'} \right)\right]+r(\mathbf{s}, \mathbf{a})$.
\blue{Such vote policy significantly modifies the policy evaluation function of CQL, potentially rendering the original conservatism guarantee in \cite[Theorem 3.2]{kumar2020conservative} invalid.} 
Thus, we provide a novel theoretical analysis of conservatism in Appendix \ref{subsec:conservatism}.


\subsection{Local Policy Improvement with AWR
}\label{subsec:awr}
As mentioned in Section \ref{sec:introduction}, current efforts suffer from a decline in policy performance because of the inconsistent optimization objective.
Moreover, due to privacy constraints, the server cannot access offline datasets to optimize the global policy directly, making it challenging to address this problem on the server side.
To this end, an intuitive strategy lies in directly optimizing the global policy on local policy updates. 
The objective is mathematically expressed as follows:
\begin{align}
\label{eq:P-Advatage-of-Global}
    &~~\mathbb{E}_{\rvs\sim d_{\bar\pi_k}(\rvs), \rva \sim \bar\pi_k(\rva|\rvs)}\left[\hat{A}(\rvs, \rva)\right],
\end{align}
where advantage \( \hat A(\rvs, \rva) = \hat{Q}(\rvs, \rva) - \hat{V}(\rvs) \) is a widely recognized metric to evaluate policy performance \cite{neumann2008fitted}.
$\hat{Q} = \lim_{\tau \to \infty} \hat{Q}^{\tau}$, $\hat{V}(\bs) = \E_{\vpolicy(\ba|\bs)}[\hat{Q}(\bs, \ba)]$, 
and $\bar{\pi}_k=\frac{1}{K}\sum\nolimits^{K}_{i=1}\pi_i$ is aggregated in the same way as the global policy, whereas updated and optimized at the local end. 
We slightly abuse notation by using \(\hat{Q}\) and \(\hat{V}\) to denote \(\hat{Q}^{\vpolicy}\) and \(\hat{V}^{\vpolicy}\), respectively.
Subsequently, we design the objective as maximizing an advantage function with respect to policy $\pi_k$:
\begin{align}
    \arg \max_{\pi_k}  & \mathbb{E}_{\rvs\sim d_{\bar\pi_k}(\rvs), \rva \sim \bar\pi_k(\rva|\rvs)}\left[\hat Q(\rvs, \rva) - \hat V(\rvs)\right] 
    \label{eq:AWRProblem}\\
    \textrm{s.t.} \quad &  
    \mathbb{E}_{\rvs\sim d_{\bar\pi_k}(\rvs)}\left[\mathrm{D_{KL}} \left(\bar\pi_k(\cdot |\rvs) , \pi_{\beta_k}(\cdot |\rvs) \right)\right] \leq \epsilon ,
    \label{eq:AWRCons1}\\
    &\int_\rva \ \bar\pi_k(\rva | \rvs) d\rva = 1,~~\forall \rvs,
    \label{eq:AWRCons2}
    \\
    & \bar\pi_k = \frac{1}{K}\sum\nolimits^{K}_{i=1}\pi_i.
    \label{eq:AWRCons3}
\end{align}
The constraint in Eq.~\eqref{eq:AWRCons1} ensures that the global policy $\bar\pi_k$ remains closely aligned with the local behavioral policy $\pi_{\beta_k}$, thereby reducing extrapolation errors during the policy improvement process. Eq.~\eqref{eq:AWRCons2} ensures policy $\bar{\pi}_k$ normalization.
Eq.~\eqref{eq:AWRCons3} describes $\bar{\pi}_k$ as defined previously.
Further, we convert the constraint \eqref{eq:AWRCons1} into a Lagrangian constraint with a coefficient \(\beta\), which is a common transformation in RL \cite{kostrikov2021iql,peng2019advantage}.
Combining the fact of $\mathbb{E}_{\rvs\sim d_{\bar\pi_k}(\rvs)}\left[\mathrm{D_{KL}} \left(\bar\pi_k(\cdot |\rvs) , \pi_{\beta_k}(\cdot |\rvs) \right)\right]=\int_{\rvs} d_{\bar\pi_k}(\rvs) \mathrm{D_{KL}} \left( \bar\pi_k(\cdot |\rvs), \pi_{\beta_k}(\cdot |\rvs) \right) d\rvs$,
we have
\begin{align}
    \arg \max_{\pi_k} \quad & \left( \int_\rvs  d_{\bar\pi_k}(\rvs) \int_\rva \bar\pi_k(\rva | \rvs) \left[\hat Q(\rvs,\bar\pi(\rva|\rvs)) - \hat V(\rvs)\right] \ d\rva \ d\rvs  \right)\nonumber\\
    &+ \beta \left(\epsilon - \int_{\rvs}  d_{\bar\pi_k}(\rvs) \mathrm{D_{KL}} \left(\bar\pi_k(\cdot |\rvs) , \pi_{\beta_k}(\cdot |\rvs) \right) d\rvs \right) \nonumber\\
    \textrm{s.t.} \quad & \int_\rva \ \bar\pi_k(\rva | \rvs) \ d\rva = 1, \quad \forall \ \rvs,
    \nonumber\\
    & \bar\pi_k = \frac{1}{K}\sum\nolimits^{K}_{i=1}\pi_i,
\end{align}
which provides a form of reasonable transition,
following the method in \cite{peng2019advantage}.
Additionally, since our goal is to derive an improved aggregated policy, we directly consider \(\bar\pi_k\) as the target variable.
\begin{align}
    \arg \max_{\bar\pi_k} \quad &  \int_\rvs  d_{\bar\pi_k}(\rvs) \int_\rva \bar\pi_k(\rva | \rvs) \left[\hat Q(\rvs,\bar\pi(\rva|\rvs)) - \hat V(\rvs)\right] \ d\rva \ d\rvs  \nonumber\\
    &+ \beta \left(\epsilon - \int_{\rvs}  d_{\bar\pi_k}(\rvs) \mathrm{D_{KL}} \left(\bar\pi_k(\cdot |\rvs) , \pi_{\beta_k}(\cdot |\rvs) \right) d\rvs \right) \nonumber\\
    \textrm{s.t.} \quad & \int_\rva \ \bar\pi_k(\rva | \rvs) \ d\rva = 1, \quad \forall \ \rvs.
\end{align}
Rearranging terms, we derive the Lagrangian of this constrained optimization problem over $\bar\pi_k$, formulated as:
\begin{align}
\label{eq:lagrange-barpik}
    \mathcal{L}(\bar\pi_k, \beta, \theta) &=   \int_\rvs  d_{\bar\pi_k}(\rvs) \int_\rva \bar\pi_k(\rva | \rvs) \left[\hat Q(\rvs,\bar\pi(\rva|\rvs)) - \hat V(\rvs)\right]  d\rva  d\rvs \nonumber\\
    &+ \beta \left(\epsilon - \int_{\rvs}  d_{\bar\pi_k}(\rvs) \mathrm{D_{KL}} \left(\bar\pi_k(\cdot |\rvs) , \pi_{\beta_k}(\cdot |\rvs) \right) d\rvs \right)\nonumber\\
    & + \int_{\rvs} \theta_\rvs \left(1 - \int_\rva \ \bar\pi_k(\rva | \rvs) d\rva \right) d\rvs,
\end{align}
where $\beta$ and $\bm\theta = \{\theta_\rvs \ | \ \forall \rvs \in \mathcal{S}\}$ represent the Lagrange multipliers.
Differentiating $\mathcal{L}(\bar\pi_k, \beta,\theta)$ with respect to $\bar\pi_k$ and solving for the optimal policy $\bar\pi_k^*$, we obtain
\begin{align}
    \frac{\partial \mathcal{L}}{ \partial \bar\pi_k(\rva | \rvs)} &= \  d_{\bar\pi_k}(\rvs) \left(\hat Q(\rvs,\bar\pi(\rva|\rvs)) - \hat V(\rvs) \right)  - \beta  d_{\bar\pi_k}(\rvs) - \theta_\rvs \nonumber\\
    &- \beta   d_{\bar\pi_k}(\rvs) \ \mathrm{log} \bar\pi_k(\rva|\rvs) + \beta  d_{\bar\pi_k}(\rvs) \mathrm{log} \pi_{\beta_k}(\rva|\rvs).
\end{align}
Taking $\frac{\partial \mathcal{L}}{ \partial \bar\pi_k(\rva | \rvs)}=0$ yields
\begin{equation}
\begin{aligned}
    \mathrm{log}{ \bar\pi_k(\rva|\rvs)}= \frac{1}{\beta}\left( \hat Q(\rvs,\bar\pi(\rva|\rvs)) - \hat V(\rvs) \right) + \mathrm{log} \pi_{\beta_k}(\rva|\rvs) + C,\nonumber
\end{aligned}
\end{equation}
where $C= \left(- \frac{1}{ d_{\bar\pi_k}(\rvs)}\frac{\theta_\rvs}{\beta} - 1 \right)$. Rearranging terms, we obtain
\begin{align}
\label{eq:21}
\blue{
    \bar\pi_k(\rva|\rvs)
    = \pi_{\beta_k}(\rva|\rvs) \mathrm{exp}\left(\frac{1}{\beta}\left(\hat Q(\rvs,\bar\pi(\rva|\rvs)) - \hat V(\rvs) \right) + C \right).}
\end{align}
Given that $\int_\rva \bar\pi_k(\rva|\rvs) \ d\rva = 1$, the second exponential term serves as the partition function $Z(\rvs)$, which normalizes the conditional action distribution.
\begin{align}
\label{eq:22}
    Z(\rvs) 
    &= \int_{\rva'} \pi_{\beta_k}(\rva'|\rvs) \mathrm{exp}\left(\frac{1}{\beta}\left(\hat Q(\rvs,\bar\pi(\rva|\rvs)) - \hat V(\rvs) \right) \right) d\rva'\\
    &= \mathrm{exp}\left(\frac{1}{ d_{\bar\pi_k}(\rvs)}\frac{\theta_\rvs}{\beta} + 1 \right).\nonumber
\end{align}
Substituting Eq.~\eqref{eq:22} in Eq.~\eqref{eq:21}, we derive the optimal of $\bar\pi_k(\rva|\rvs)$ as follow:

\begin{align}
\label{eq:AWRLagrangianOPT}
 \blue{
     \bar\pi_k^*(\rva|\rvs) = \frac{\pi_{\beta_k}(\rva|\rvs)}{Z(\rvs)}  \ \mathrm{exp}\left(\frac{1}{\beta}\left(\hat Q(\rvs, \bar\pi(\rvs)) - \hat V(\rvs) \right) \right).
     }
\end{align}
In summary, we obtain Eq.~\eqref{eq:AWRLagrangianOPT} as a closed-form expression to a constrained optimization problem \eqref{eq:AWRProblem} on global policy through such scaling and equivalent transformations. 
Next, we project $\bar\pi_k^*$ onto the manifold of parameterized policies to mitigate the distribution shift between local and global policies \cite{yue2024federated,shlens2014notes}.  
\begin{align}
\label{eq:AWRUpdate}
    & \arg\min_{\pi_k}  \mathbb{E}_{\rvs \sim \mathcal{D}_k} \left[ \mathrm{D_{KL}} \left(\bar\pi_k^*(\cdot  | \rvs) \middle|\middle| \pi_k(\cdot | \rvs)\right) \right] = \\
     & \arg\max_{\pi_k}  \mathbb{E}_{\rvs,\rva \sim \mathcal{D}_k}  \left[ \mathrm{log} \pi_k(\rva | \rvs)  \mathrm{exp}\left( \frac{1}{\beta} \left( \hat Q(\rvs, \bar\pi(\rvs)) - \hat V(\rvs) \right) \right) \right]\nonumber
\end{align}
where $\beta$ acts as a positive temperature parameter. For large values of $\beta$, the objective function behaves akin to behavioral cloning, while for smaller values, it aims to maximize the Q function. Previous studies have demonstrated that this objective function learns a policy that maximizes the Q-value while adhering to distribution constraints \cite{kostrikov2021iql,peng2019advantage,peters2007reinforcement}. 
Furthermore, the goal of local policy improvement is to optimize the local policy while constraining the distribution shift between the local policy and the global policy, as \cite{yue2024federated}.

\begin{align}
\label{eq:awr}
\max_{\pi_k} ~~
 &\lambda \mathbb{E}_{\rvs,\rva \sim \mathcal{D}_k}  \left[ \mathrm{log} \pi_k(\rva | \rvs)  \mathrm{exp}\left( \frac{1}{\beta} \left(\hat Q(\rvs, \bar\pi(\rvs)) - \hat V(\rvs) \right) \right) \right] \nonumber\\
+&\mathbb{E}_{\rvs\sim\mathcal{D}_k,\rva\sim\pi_k(\rva|\rvs)}\big[\hat{Q}(\rvs,\rva)\big],
\end{align}
where $\lambda$ is a positive weighting parameter. 
Notably, $\pi_k$ achieves strict policy improvement over the local behavior policy.  
We delve into a theoretical analysis of this finding in Appendix \ref{sec:improve_bp}.
Specifically, the argument passes through an empirical–true MDP comparison, where Lemma~\ref{lemma:JMDP-new} bounds the return in the empirical MDP $\widetilde{M}$ relative to the underlying MDP $M$ and thereby enables performance transfer. Building on this bridge, Theorems~\ref{thm:policy-improvement-global-app} and \ref{thm:policy-improvement-local-app} deliver strict policy improvement guarantees for FOVA. 

\subsection{Server Aggregation}\label{subsec:server}
Following the conventional pattern of FL, during the server aggregation phase, we adopt linear aggregation as follows: 
\begin{align}
    \label{eq:aggregation}
    \bar\pi\leftarrow\frac{1}{K}\sum\nolimits^{K}_{k=1}\pi_k,
    ~~~~~~\bar{Q}\leftarrow\frac{1}{K}\sum\nolimits^{K}_{k=1} Q_k.
\end{align}


We outline the FOVA algorithm designed for offline FRL, with the pseudocode outlined in Algorithm 1.
The FOVA training code can be readily implemented on top of the classical CQL algorithm \cite{kumar2020conservative} by modifying fewer than ten lines of code in the OfflineRL-Kit library \footnote{\url{https://github.com/yihaosun1124/OfflineRL-Kit}}.
This algorithm follows a structured process, integrating computations across both the server and client sides. The initial values for $\alpha$ are derived from Lemma \ref{thm:vcql_underestimates}, while the initial values for $\beta$ and $\lambda$ are based on Theorem \ref{thm:policy-improvement-local-app} and Theorem \ref{thm:policy-improvement-global-app}, respectively.

\begin{algorithm}[h]
    \caption{FOVA}
    \label{alg:fova}
    \LinesNumbered
    \KwIn{
    $\alpha$, $\beta$, $\lambda$, $\{\mathcal{D}_k\}^\mathcal{K}_{k=1}$, $\bar\pi$ and $\bar Q$}
        \For{$t=0$ \KwTo $T$}{ \tcp{Server side}
            {Server distributes $\bar\pi$ and $\bar Q$ to clients}\\
            \For{$k=1$ \KwTo $K$}{  \tcp{Client side}
                Client $k$ receives $\bar Q$ and $\bar\pi$\\
                    Train $Q_k$ and $\pi_k$ by Eq.~\eqref{eq:vcql} and ~\eqref{eq:awr}\\
                Send $Q_k$, $\pi_k$ to the server
            }
            {Server receives $\{Q_k\}^{\mathcal{K}}_{k=1}$, and $\{\pi_k\}^{\mathcal{K}}_{k=1}$}\\
            Calculate $\bar Q$ and $\bar\pi$ by  Eq.~\eqref{eq:aggregation}\
        }
    \KwOut{Global policy ${\bar\pi}$.}
\end{algorithm}

On the client side, for each training iteration, each client samples a minibatch of transition data from the local dataset $\mathcal{D}_k$. Subsequently, they update $Q_k$ and $\pi_k$ using the vote policy and AWR methods during both the policy evaluation and policy improvement phases. After completion of local updates, each client sends their updated $Q_k$ and $\pi_k$ back to the server.
On the server side, upon receiving all updated local policies and Q-functions, global aggregation is performed to update $\bar{Q}$ and $\bar{\pi}$.
The updated global policy is then redistributed to all clients for the next round of local updates. 
This iterative process continues until convergence, ultimately yielding the global policy $\bar{\pi}$ as the output of the algorithm.

\blue{Because client datasets are non-i.i.d., Section III in the supplementary file formalizes client MDPs, quantifies cross-client heterogeneity, and derives a global improvement bridge via advantage decomposition and averaging \cite{qiao2025fova_supp}. We then establish a local, regularized policy-improvement lower bound with error terms governed by $\lambda$ and $\beta$, culminating in a global safe-improvement theorem.}

\begin{table*}[!t]
\caption[]{Averaged scores on Gym locomotion and Maze2D tasks over five random seeds.}
\label{tab:compare}
\small
\setlength{\tabcolsep}{5.5pt}
\centering
\begin{tabular}{clcccccccccc}
\toprule
\multicolumn{1}{c}{\text{Quality}} & \text{Task} & \text{Fed-TD3BC} & \text{CQL-FL} & \text{FEDORA} & \text{DRPO} & \text{FOVA (ours)}  \\
\midrule
\multirow{4}{*}{\rotatebox[origin=c]{90}{\shortstack{Medium\\-replay}}}
& HalfCheetah & $322.86$\textcolor{black!60}{\po$\pm\po46.08$} & $3599.23$\textcolor{black!60}{\po$\pm\po759.31$} & $1182.48$\textcolor{black!60}{\po$\pm\po474.17$} & \textbf{5699.47}\textcolor{black!60}{\textbf{\po\bm{$\pm$}\po172.51}} & 5318.01\textcolor{black!60}{\po$\pm\po295.25$}  \\
& Hopper      & $\po31.64$\textcolor{black!60}{\po$\pm\po3.39\po$} &  $\po474.79$\textcolor{black!60}{\po$\pm\po172.86$} & {$\po716.10$}\textcolor{black!60}{\po$\pm\po218.12$} & $\po674.67$\textcolor{black!60}{\po$\pm\po154.54$} &
\textbf{1159.54}\textcolor{black!60}{\textbf{\po\bm{$\pm$}\po359.07}}  \\
& Walker2d    & $216.47$\textcolor{black!60}{\po$\pm\po19.86$} & $1260.72$\textcolor{black!60}{\po$\pm\po190.79$} & $\po158.90$\textcolor{black!60}{\po$\pm\po154.24$} & {$1167.41$}\textcolor{black!60}{\po$\pm\po415.84$} & \textbf{1951.92}\textcolor{black!60}{\textbf{\po\bm{$\pm$}\po303.05}}  \\
& Ant         & $\po70.71$\textcolor{black!60}{\po$\pm\po28.51$} & {$1771.87$}\textcolor{black!60}{\po$\pm\po305.98$} & $1053.39$\textcolor{black!60}{\po$\pm\po328.26$} & $1648.21$\textcolor{black!60}{\po$\pm\po202.69$} & \textbf{1801.35}\textcolor{black!60}{\textbf{\po\bm{$\pm$}\po326.84}}  \\
\midrule
\multirow{4}{*}{\rotatebox[origin=c]{90}{Medium}}
& HalfCheetah & $566.15$\textcolor{black!60}{\po$\pm\po32.60$} & $4625.33$\textcolor{black!60}{\po$\pm\po665.95$} & $4221.63$\textcolor{black!60}{\po$\pm\po701.02$} & \textbf{6125.08}\textcolor{black!60}{\textbf{\po\bm{$\pm$}\po254.08}} & 5886.32\textcolor{black!60}{\po$\pm\po183.66$}  \\
& Hopper      & $\po19.64$\textcolor{black!60}{\po$\pm\po0.91\po$} & $\po480.78$\textcolor{black!60}{\po$\pm\po162.10$} & {$1780.10$}\textcolor{black!60}{\po$\pm\po211.34$} & $1736.08$\textcolor{black!60}{\po$\pm\po191.77$} & \textbf{1961.35}\textcolor{black!60}{\textbf{\po\bm{$\pm$}\po306.26}}  \\
& Walker2d    & $271.89$\textcolor{black!60}{\po$\pm\po23.19$} & $2244.43$\textcolor{black!60}{\po$\pm\po275.81$} & $1055.82$\textcolor{black!60}{\po$\pm\po540.66$} & {$2773.79$}\textcolor{black!60}{\po$\pm\po391.93$} & \textbf{2919.78}\textcolor{black!60}{\textbf{\po\bm{$\pm$}\po414.77}}  \\
& Ant         & $~66.56$\textcolor{black!60}{\po$\pm\po33.19$} & $3313.17$\textcolor{black!60}{\po$\pm\po371.59$} & $~699.09$\textcolor{black!60}{\po$\pm\po289.26$} & {$3406.45$}\textcolor{black!60}{\po$\pm\po575.70$} & \textbf{3572.45}\textcolor{black!60}{\textbf{\po\bm{$\pm$}\po478.81}}  \\
\midrule
\multirow{4}{*}{\rotatebox[origin=c]{90}{\shortstack{Medium\\-expert}}}
& HalfCheetah & $235.92$\textcolor{black!60}{\po$\pm\po24.27$} & $1855.34$\textcolor{black!60}{\po$\pm\po562.41$} & $5595.15$\textcolor{black!60}{\po$\pm\po900.75$} & {$6093.59$}\textcolor{black!60}{\po$\pm\po735.31$} & \textbf{7055.40}\textcolor{black!60}{\textbf{\po\bm{$\pm$}\po1022.5}}  \\
& Hopper      & $968.65$\textcolor{black!60}{\po$\pm\po11.90$} & $\po555.48$\textcolor{black!60}{\po$\pm\po242.02$} & $\po997.79$\textcolor{black!60}{\po$\pm\po357.30$} & {$1240.89$}\textcolor{black!60}{\po$\pm\po349.85$} & \textbf{1793.44}\textcolor{black!60}{\textbf{\po\bm{$\pm$}\po240.07}} \\
& Walker2d    & $\po14.84$\textcolor{black!60}{\po$\pm\po8.68$\po} & $\po755.33$\textcolor{black!60}{\po$\pm\po333.56$} & {$2520.77$}\textcolor{black!60}{\po$\pm\po858.47$} & $2270.17$\textcolor{black!60}{\po$\pm\po298.43$} & \textbf{2919.78}\textcolor{black!60}{\textbf{\po\bm{$\pm$}\po490.15}}   \\
& Ant         & $186.96$\textcolor{black!60}{\po$\pm\po23.51$} & $2764.01$\textcolor{black!60}{\po$\pm\po298.89$} & $2174.40$\textcolor{black!60}{\po$\pm\po205.89$} & {$2703.34$}\textcolor{black!60}{\po$\pm\po404.04$} & \textbf{3163.63}\textcolor{black!60}{\textbf{\po\bm{$\pm$}\po344.55}}  \\
\midrule
\multirow{4}{*}{\rotatebox[origin=c]{90}{Expert}}
& HalfCheetah & $315.81$\textcolor{black!60}{\po$\pm\po13.12$} & $8554.64$\textcolor{black!60}{\po$\pm\po844.09$} & {$9317.24$}\textcolor{black!60}{\po$\pm\po252.19$} & $7317.01$\textcolor{black!60}{\po$\pm\po1055.1$} & \textbf{11419.6}\textcolor{black!60}{\textbf{\po\bm{$\pm$}\po643.87}}  \\
& Hopper      & $349.29$\textcolor{black!60}{\po$\pm\po19.60$} & $\po504.82$\textcolor{black!60}{\po$\pm\po183.77$} & $1219.33$\textcolor{black!60}{\po$\pm\po140.07$} & {$2141.96$}\textcolor{black!60}{\po$\pm\po431.44$} & \textbf{3095.18}\textcolor{black!60}{\textbf{\po\bm{$\pm$}\po277.05}}  \\
& Walker2d    & $202.14$\textcolor{black!60}{\po$\pm\po40.59$} & $1600.82$\textcolor{black!60}{\po$\pm\po217.40$} & $2233.94$\textcolor{black!60}{\po$\pm\po588.13$} & {$2786.76$}\textcolor{black!60}{\po$\pm\po514.46$} & \textbf{3991.86}\textcolor{black!60}{\textbf{\po\bm{$\pm$}\po435.61}}  \\
& Ant         & $678.03$\textcolor{black!60}{\po$\pm\po131.7$} & $2508.69$\textcolor{black!60}{\po$\pm\po611.32$} & $\po801.41$\textcolor{black!60}{\po$\pm\po272.25$} & {$3646.48$}\textcolor{black!60}{\po$\pm\po407.18$} & \textbf{3548.27}\textcolor{black!60}{\textbf{\po\bm{$\pm$}\po551.96}}  \\
\midrule
\multicolumn{2}{c}{\textit{Locomotion avg.}} & $282.34$ & $2304.33$ & $2232.97$ & $3214.21$ & \textbf{3853.30} \textcolor{black!60}{\bm{$(\uparrow 19.89\%)$}}  \\
\midrule
\multirow{3}{*}{\rotatebox[origin=c]{90}{Sparse}}
& Umaze & $22.46$\textcolor{black!60}{\po$\pm\po2.06$} & $27.35$\textcolor{black!60}{\po$\pm\po6.51$} & $22.41$\textcolor{black!60}{\po$\pm\po7.49$} & $33.59$\textcolor{black!60}{\po$\pm\po3.05$} & \textbf{\po56.68}\textcolor{black!60}{\textbf{\po\bm{$\pm$}\po5.26\po}} \\
& Medium & $13.05$\textcolor{black!60}{\po$\pm\po4.54$} & $27.51$\textcolor{black!60}{\po$\pm\po7.50$} & $\po9.94$\textcolor{black!60}{\po$\pm\po3.47$} & $47.90$\textcolor{black!60}{\po$\pm\po13.3$} & \textbf{\po57.71}\textcolor{black!60}{\textbf{\po\bm{$\pm$}\po11.04}}\\ 
& Large & $\po1.27$\textcolor{black!60}{\po$\pm\po0.41$} & $\po6.50$\textcolor{black!60}{\po$\pm\po1.10$} & $17.35$\textcolor{black!60}{\po$\pm\po5.05$} & $22.14$\textcolor{black!60}{\po$\pm\po7.38$} & \textbf{\po44.70}\textcolor{black!60}{\textbf{\po\bm{$\pm$}\po11.13}} \\
\midrule
\multirow{3}{*}{\rotatebox[origin=c]{90}{Dense}}
& Umaze & $36.11$\textcolor{black!60}{\po$\pm\po3.27$} & $44.80$\textcolor{black!60}{\po$\pm\po11.8$} & $\po62.04$\textcolor{black!60}{\po$\pm\po14.15$} & $66.92$\textcolor{black!60}{\po$\pm\po5.58$} & \textbf{\po79.84}\textcolor{black!60}{\textbf{\po\bm{$\pm$}\po8.19\po}} \\
& Medium & $15.77$\textcolor{black!60}{\po$\pm\po3.72$} & $31.93$\textcolor{black!60}{\po$\pm\po8.26$} & \textbf{119.21}\textcolor{black!60}{\textbf{\po\bm{$\pm$}\po7.21\po}} & $98.94$\textcolor{black!60}{\po$\pm\po8.69$} & {$104.95$}\textcolor{black!60}{\po$\pm\po14.67$} \\
& Large & $13.94$\textcolor{black!60}{\po$\pm\po3.62$} & $17.81$\textcolor{black!60}{\po$\pm\po4.18$} & $\po42.82$\textcolor{black!60}{\po$\pm\po5.81\po$} & $29.09$\textcolor{black!60}{\po$\pm\po5.53$} & \textbf{105.45}\textcolor{black!60}{\textbf{\po\bm{$\pm$}\po7.32\po}} \\
\midrule
\multicolumn{2}{c}{\textit{Mazed2D avg.}} & $17.1$ & $25.98$ & $45.63$ & $49.76$ & \textbf{74.88 \textcolor{black!60}{\bm{$(\uparrow 50.48\%)$}}}\\  
\bottomrule
\end{tabular}
\end{table*}

\section{Simulation}\label{sec:simulation} 

In this section, we evaluate our proposed algorithm by exploring several key questions. Firstly, we compare the overall performance of our method with that of the baselines.
Secondly, we investigate how the consistent optimization objective plays a role in improving the overall performance. 
Thirdly, we examine whether our method effectively solves the performance degradation issues that arise in scenarios involving mixed-quality data. 
Finally, we explore the impact of the vote mechanism, the Advantage Weighted Regression (AWR) method, and the critical parameters on the performance of our algorithm.

\begin{figure}[h]
	\centering
    \subfigure[Halfcheetah.]{\label{fig:gym-halfcheetah}\includegraphics[width=0.11\textwidth,height=0.14\textwidth]{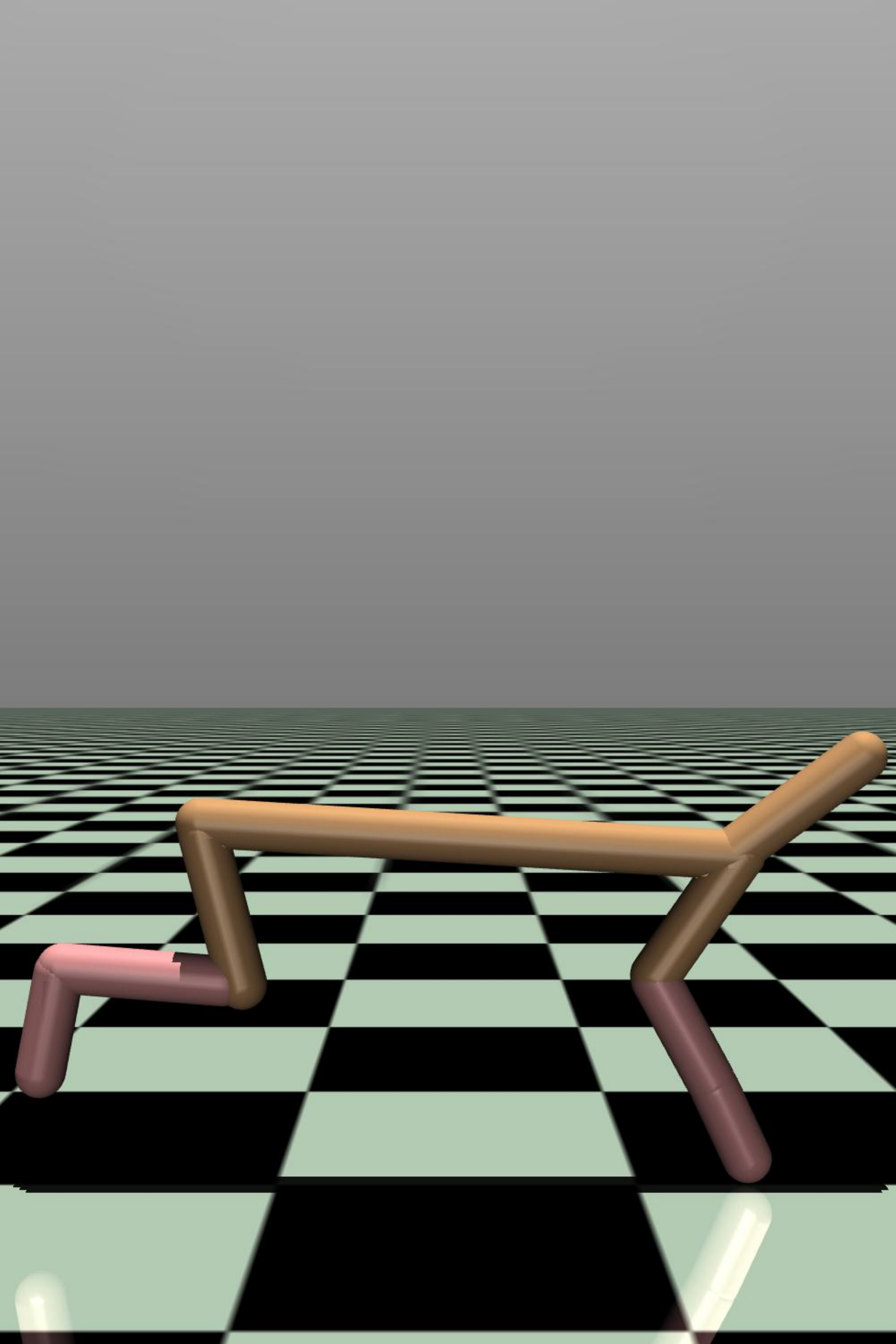}}
    \subfigure[Hopper.]{\label{fig:gym-hopper}\includegraphics[width=0.11\textwidth,height=0.14\textwidth]{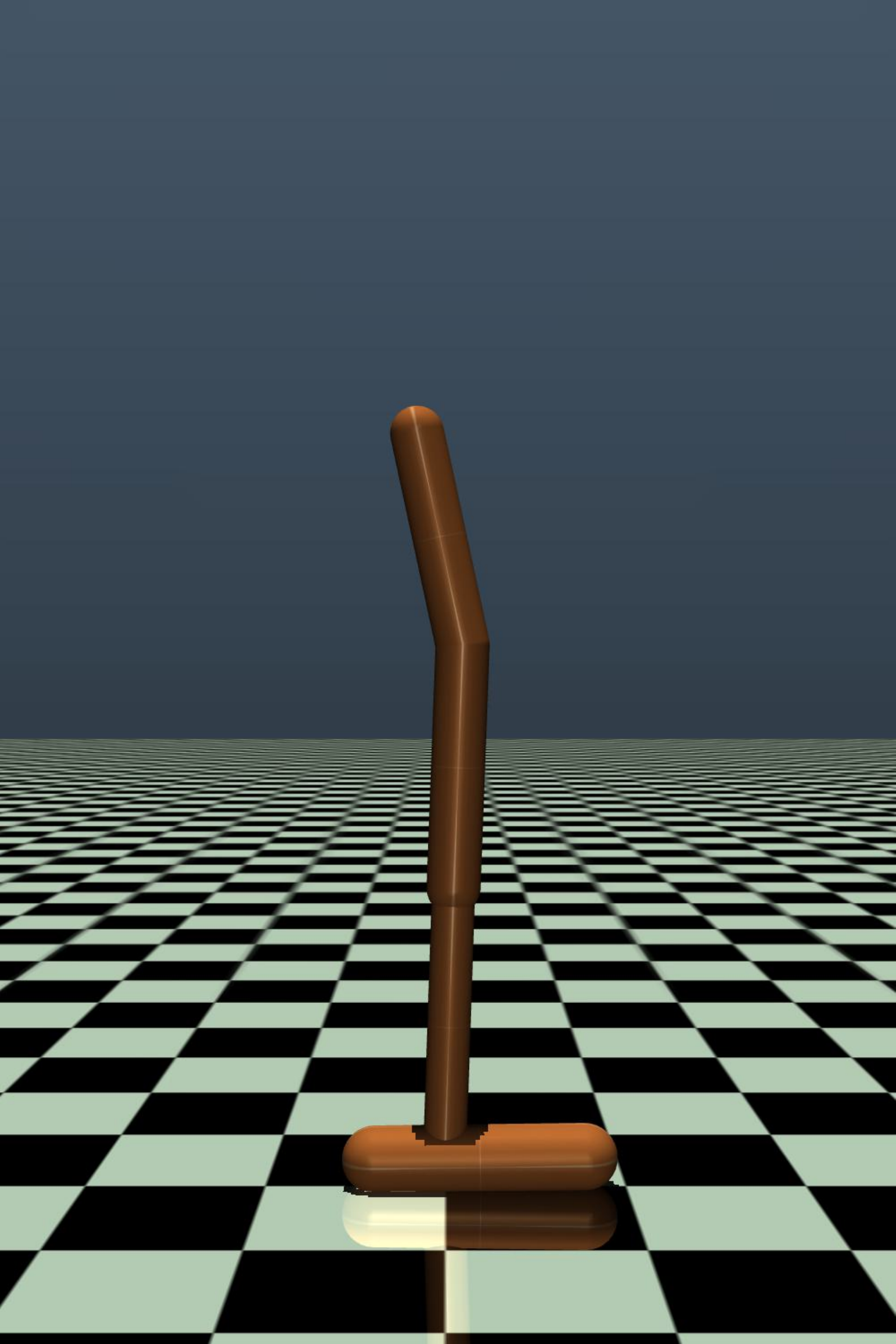}}
    \subfigure[Walker2d.]{\label{fig:gym-walker2d}\includegraphics[width=0.11\textwidth,height=0.14\textwidth]{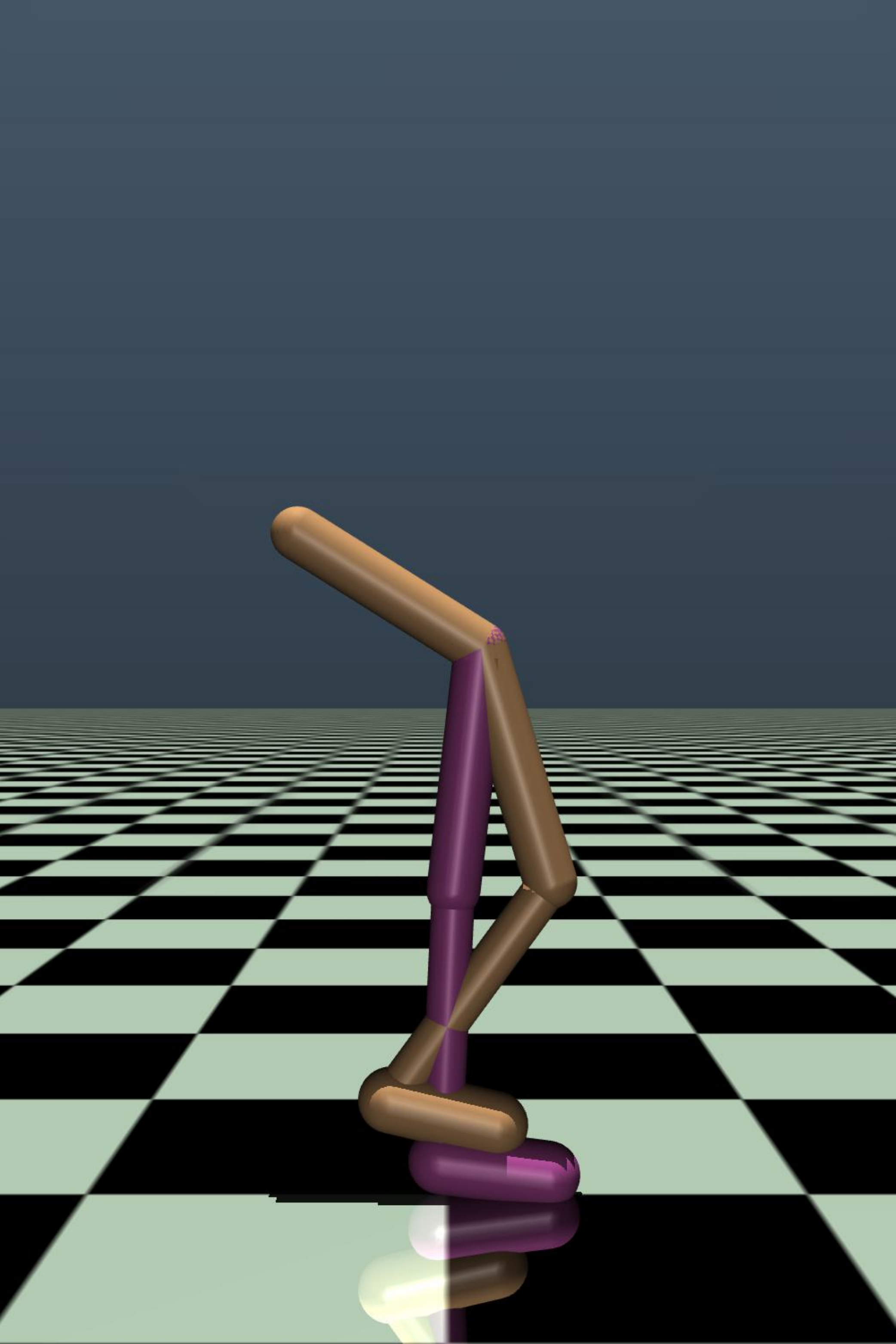}}
    \subfigure[Ant.]{\label{fig:gym-ant}\includegraphics[width=0.11\textwidth,height=0.14\textwidth]{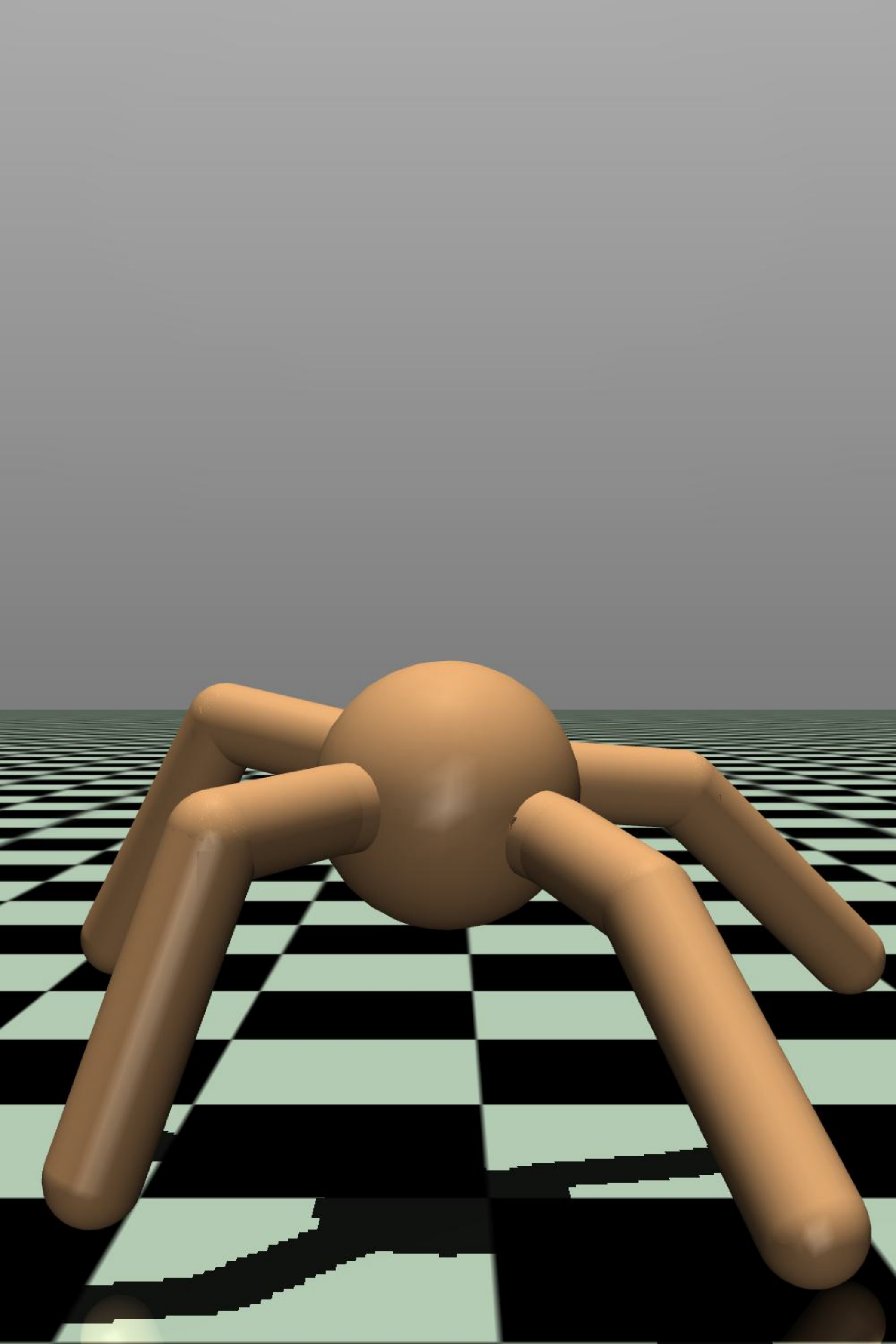}}

	\centering
    \subfigure[Umaze.]{\label{fig:maze-umaze}\includegraphics[width=0.11\textwidth]{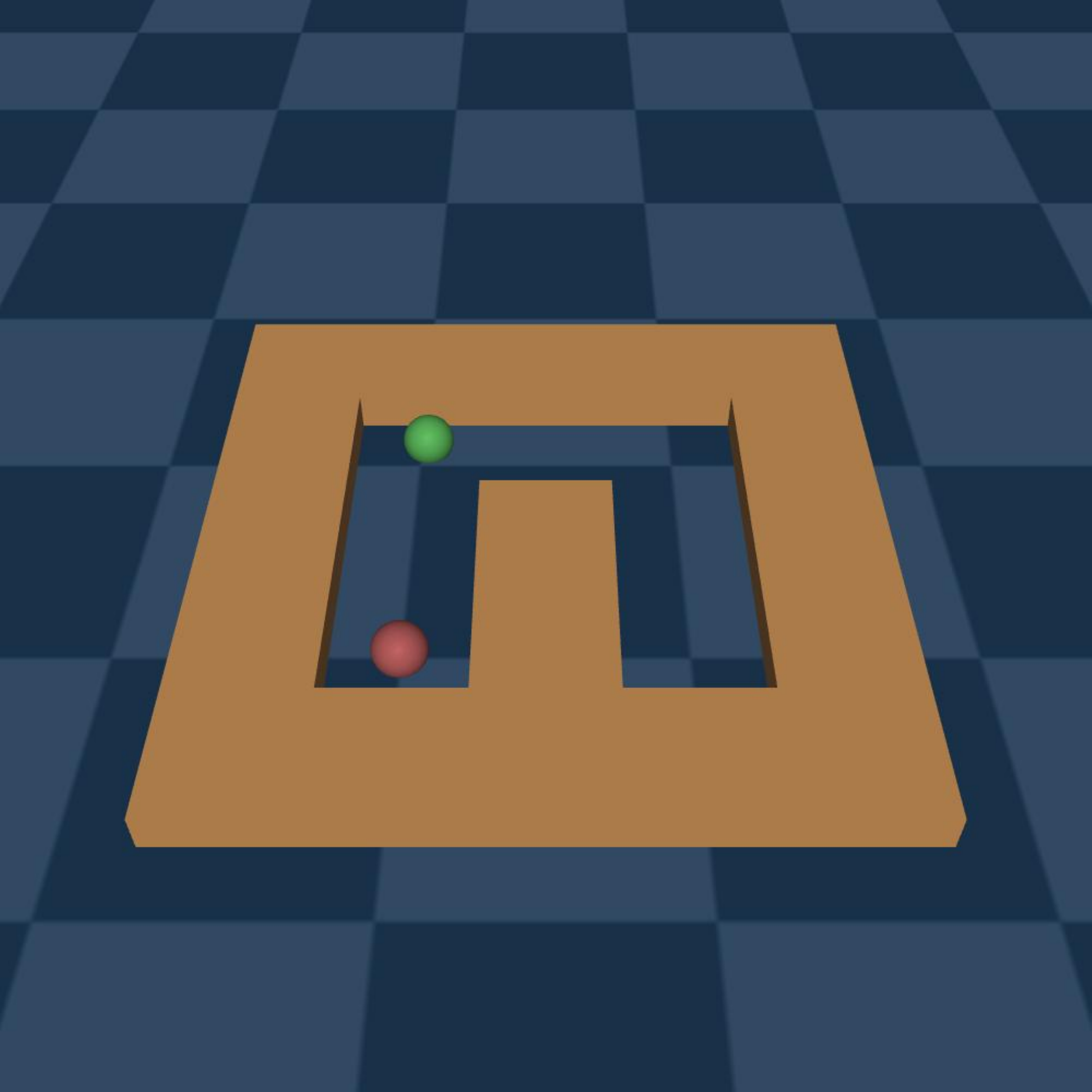}}
    \subfigure[Medium.]{\label{fig:maze-medium}\includegraphics[width=0.11\textwidth]{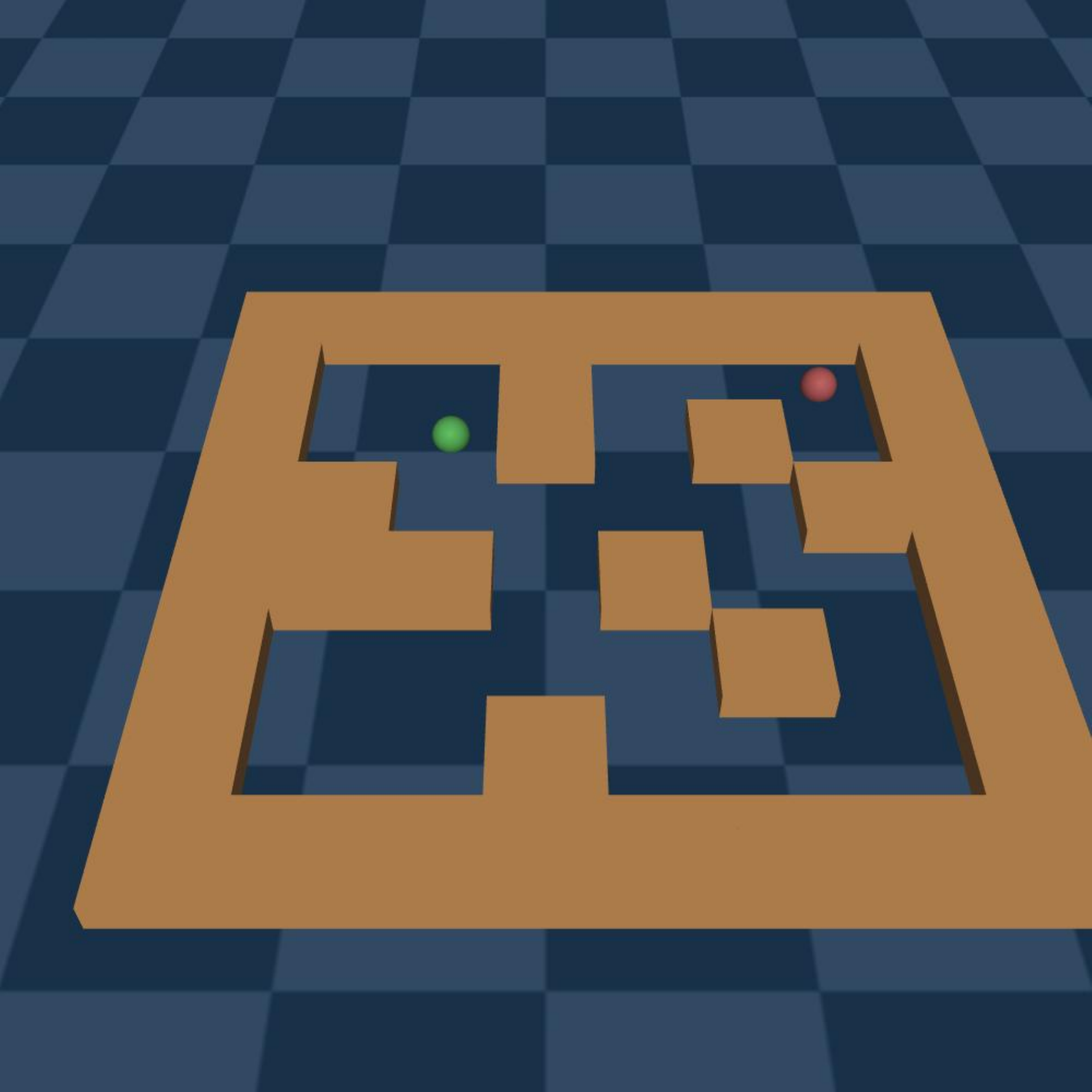}}
    \subfigure[Large.]{\label{fig:maze-large}\includegraphics[width=0.11\textwidth]{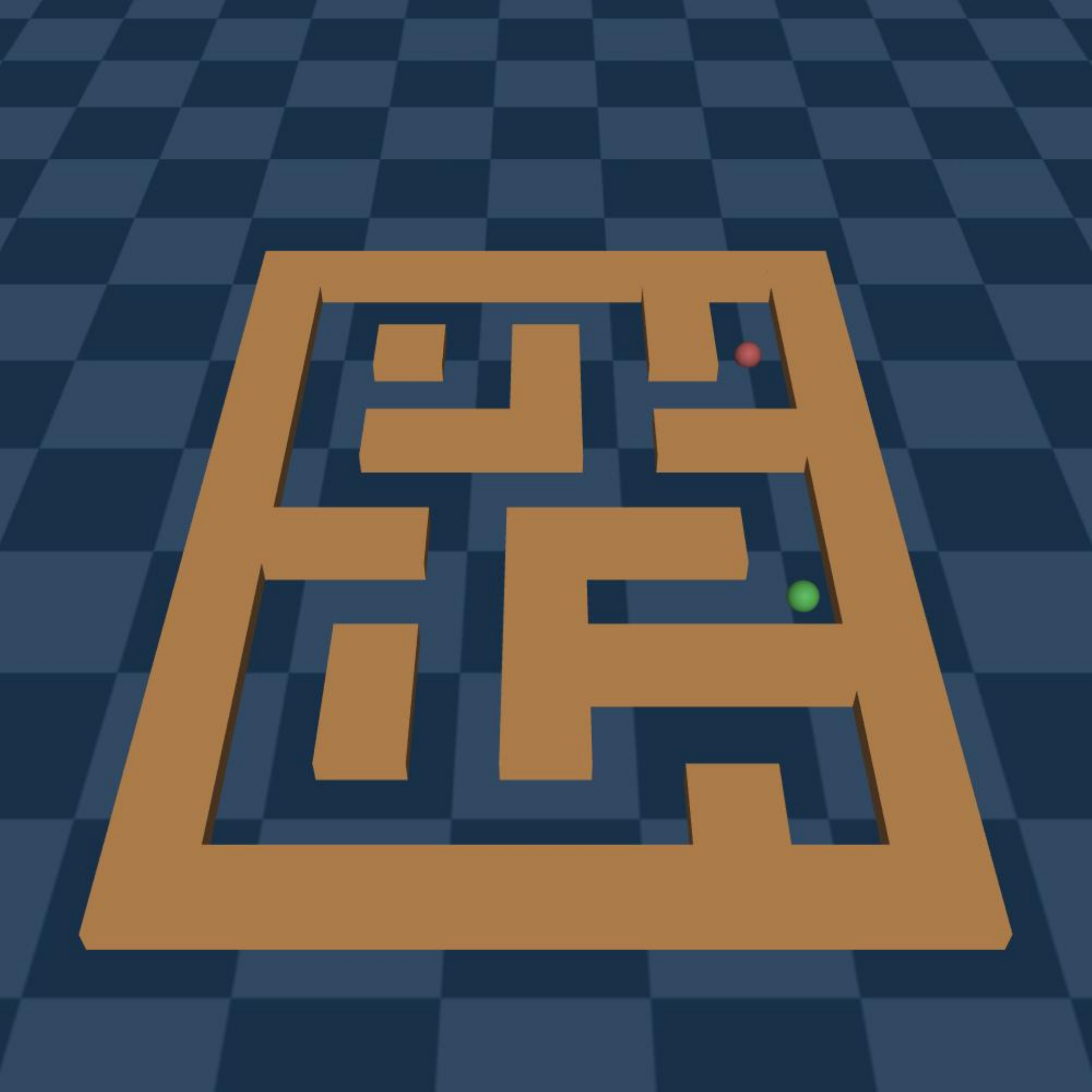}}
\caption{Gym locomotion and Maze2D tasks.} 
\label{fig:gym}
\end{figure}

\subsection{Experiment Setup}
\label{subsec:Exp-Setup}
\textbf{Environments and Baselines}.
We evaluate the performance of FOVA using datasets from the widely recognized offline RL benchmark D4RL \footnote{\url{https://github.com/rail-berkeley/d4rl}}
\cite{fu2020d4rl,todorov2012mujoco}.
This benchmark is built upon the Gym locomotion and Maze2D simulation environments, as shown in Fig.~\ref{fig:gym}. 
The Gym locomotion encompasses four challenging continuous robotic control tasks (namely "HalfCheetah", "Hopper", "Walker2d" and "Ant"), involving four different dataset-quality types (expert, medium-expert, medium, and medium-replay). The Gym The Maze2D domain is a navigation task that requires a 2D agent to reach a fixed goal location (including three types of maps, namely "umaze", "medium", and "large"), as well as two types of datasets-quality (dense reward and sparse reward).
In addition, we compare FOVA\footnote{\url{https://sites.google.com/view/fova/}} with the following baselines:

\begin{itemize} 
\item \textit{Fed-TD3BC}: Combining FedAvg with TD3BC \cite{fujimoto2021minimalist}, this method performs TD3BC updates locally on clients with a specified number of gradient steps. The policy parameters are then aggregated on the server.

\item \textit{CQL-FL}: This approach modifies the CQL \cite{kumar2020conservative} rules locally on the clients through a fixed number of gradient steps \cite{park2022federated}. Subsequently, the updated parameters are linearly aggregated on the server.

\item \textit{FEDORA} is an open-source offline FRL algorithm \cite{rengarajan2023federated}. We reproduce FEDORA by following their paper and utilizing the provided open-source implementation \footnote{\url{https://github.com/DesikRengarajan/FEDORA}}.

\item \textit{DRPO} represents the state-of-the-art offline FRL algorithm \cite{yue2024federated}. Our reproduction of DRPO is based on their paper and the open-source code available \footnote{\url{https://github.com/HansenHua/DRPO-INFOCOM24}}.
\end{itemize}


\textbf{Implementation.}
We implement all experiments in PyTorch~2.1.2 on Ubuntu~20.04.4 LTS with four NVIDIA GeForce RTX~3090 GPUs. The Actor and Critic are ReLU-based multilayer perceptrons (MLPs) with three hidden layers of 256 units; the Critic model applies Layer Normalization after each hidden layer to avoid feature rank collapse \cite{yue2023understanding}. We train with Adam (batch size 256), using learning rates of $1\!\times\!10^{-4}$ for the Actor and $3\!\times\!10^{-4}$ for the Critic, and set the discount factor to $\gamma=0.99$. Unless otherwise stated, we use default hyperparameters $\lambda=\beta=5$, and choose $\alpha$ consistent with CQL~\cite{kumar2020conservative}. All reported numbers are averaged over 3–5 runs.
Some reported results are normalized scores provided by the D4RL benchmark to measure the performance of the learned policy score compared to random and expert scores \footnote{\url{https://github.com/Farama-Foundation/D4RL/}}. The normalized score is calculated as:
\[
\mathrm{normalized\ score}
=100\times\frac{\mathrm{score}-\mathrm{random\ score}}{\mathrm{expert\ score}-\mathrm{random\ score}}\, .
\]

\subsection{Experiment Results}\label{subsec:Exp-Evaluation}

\paragraph{Overall Performance}
We conducted a comprehensive comparison of our FOVA algorithm against leading methods, including DRPO and FEDORA, as presented in Table \ref{tab:compare}. Each experiment was repeated five times using different random seeds to ensure the robustness of our results, with the highest performance scores highlighted for clarity. Utilizing the standard D4RL benchmark tests across Gym locomotion and Maze2D datasets, FOVA achieved significant performance improvements, with score enhancements of up to 20\% and 50\% respectively. Additionally, FOVA attained optimal or near-optimal performance in nearly all evaluated tasks, demonstrating its effectiveness and reliability compared to existing state-of-the-art algorithms.

\begin{figure}[h]
    \vspace{-6pt}
	\centering
    \hspace{-0.04\textwidth}
        \subfigure[Impact of datasize on baselines.]{\label{fig:incon-datasize}\includegraphics[width=0.302\textwidth]{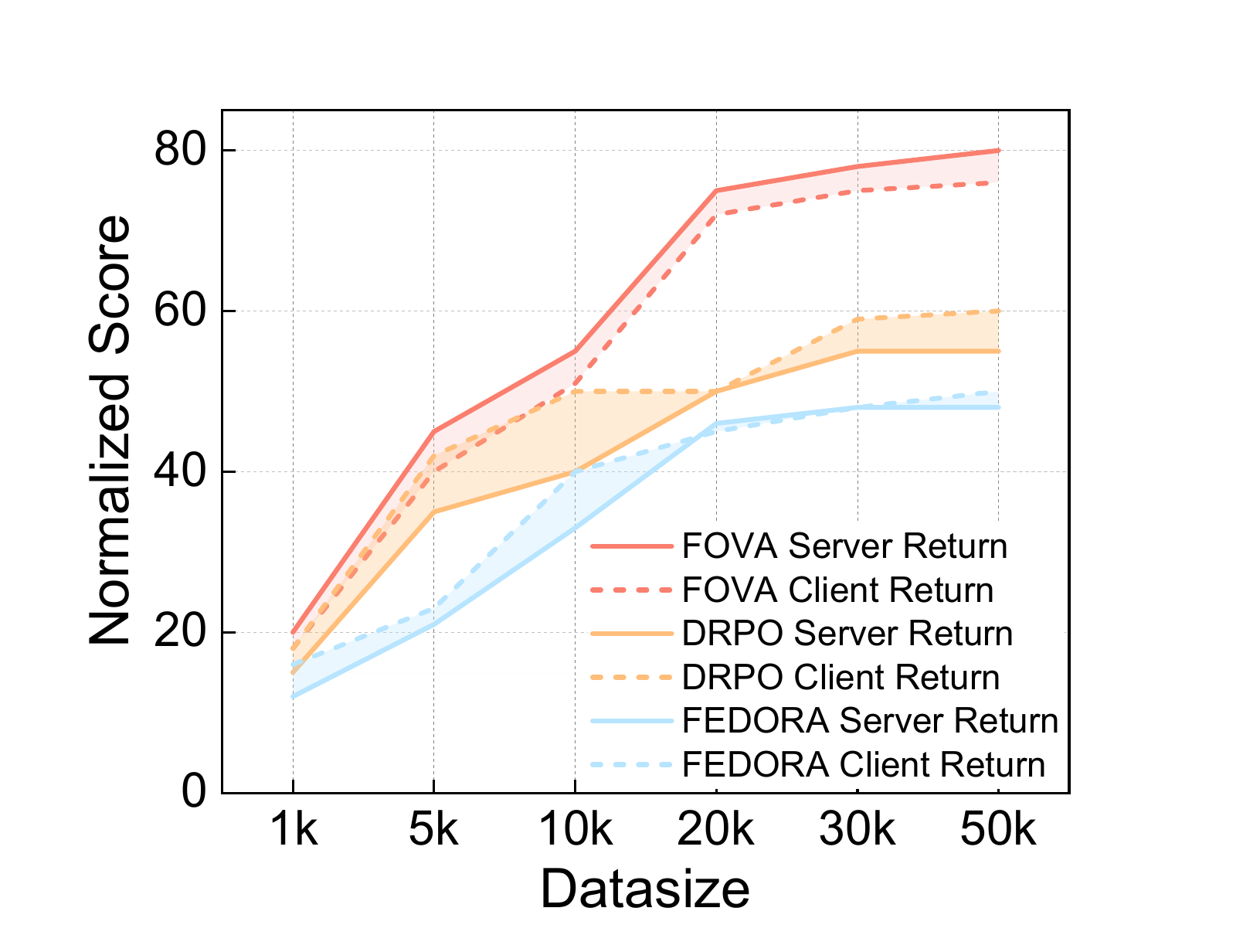}}
        \hspace{-0.04\textwidth}
	\subfigure[FOVA in Hopper-medium-replay.]{\label{fig:incon-cover}\includegraphics[width=0.245\textwidth]{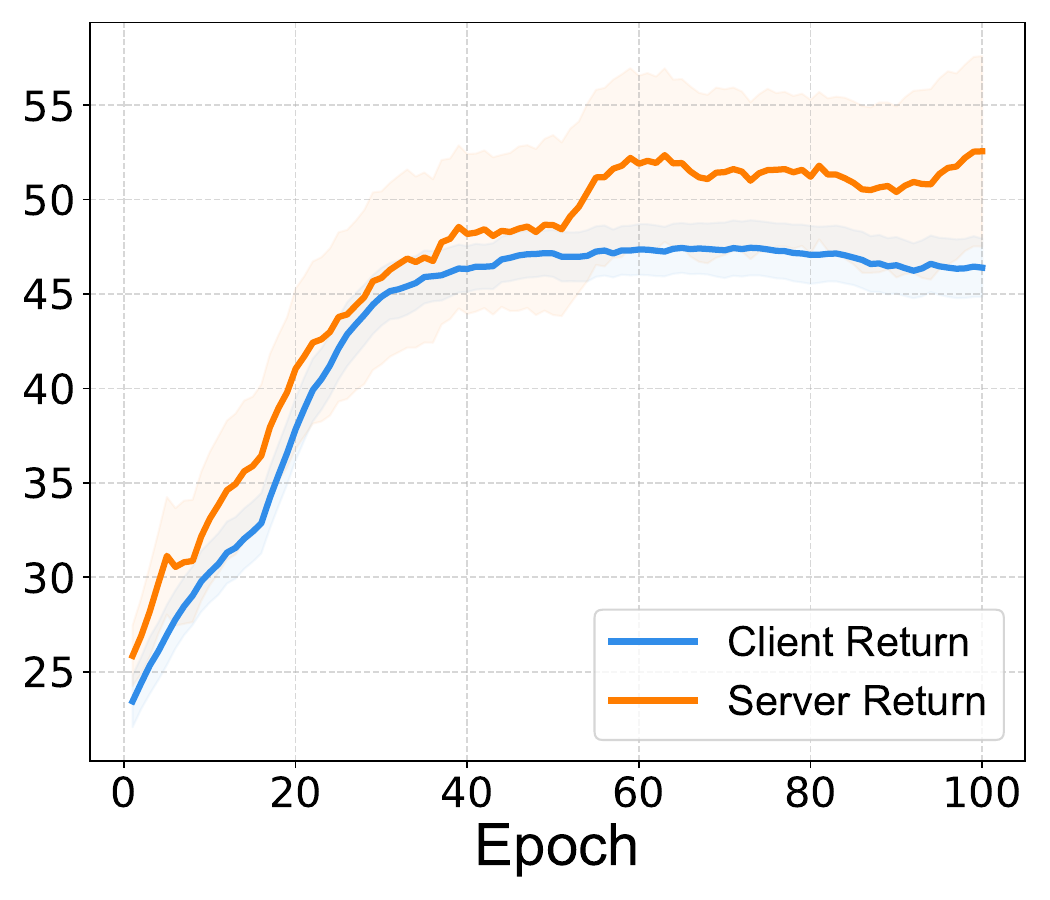}}
	\caption{Performance comparison between server-side policy and client-side policy.}
        \label{fig:moti-fova}
    \vspace{-0pt}
\end{figure}


\begin{figure*}[tbp]
    \centering 
    \subfigure{\label{fig:hoer}\includegraphics[width=0.244\textwidth]{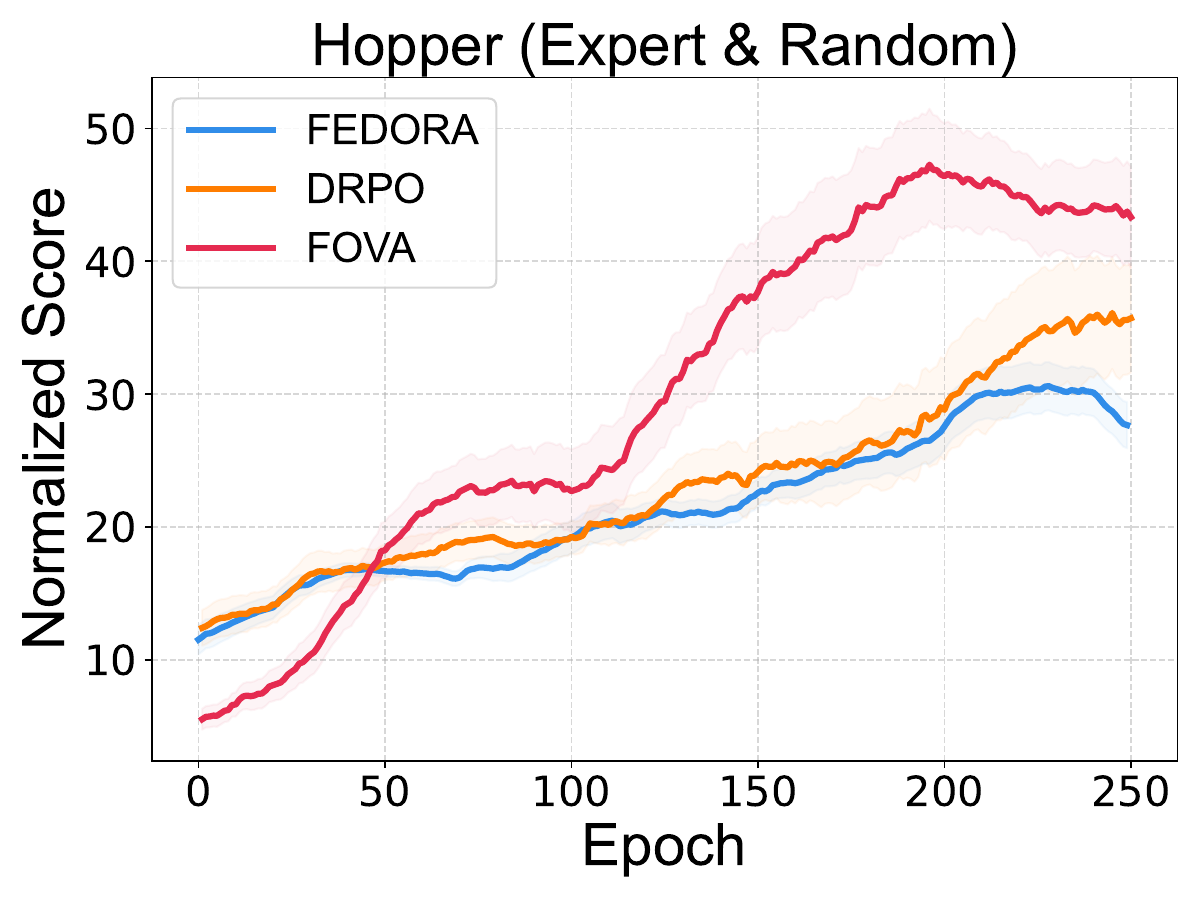}}
    \subfigure{\label{fig:hamr}\includegraphics[width=0.244\textwidth]{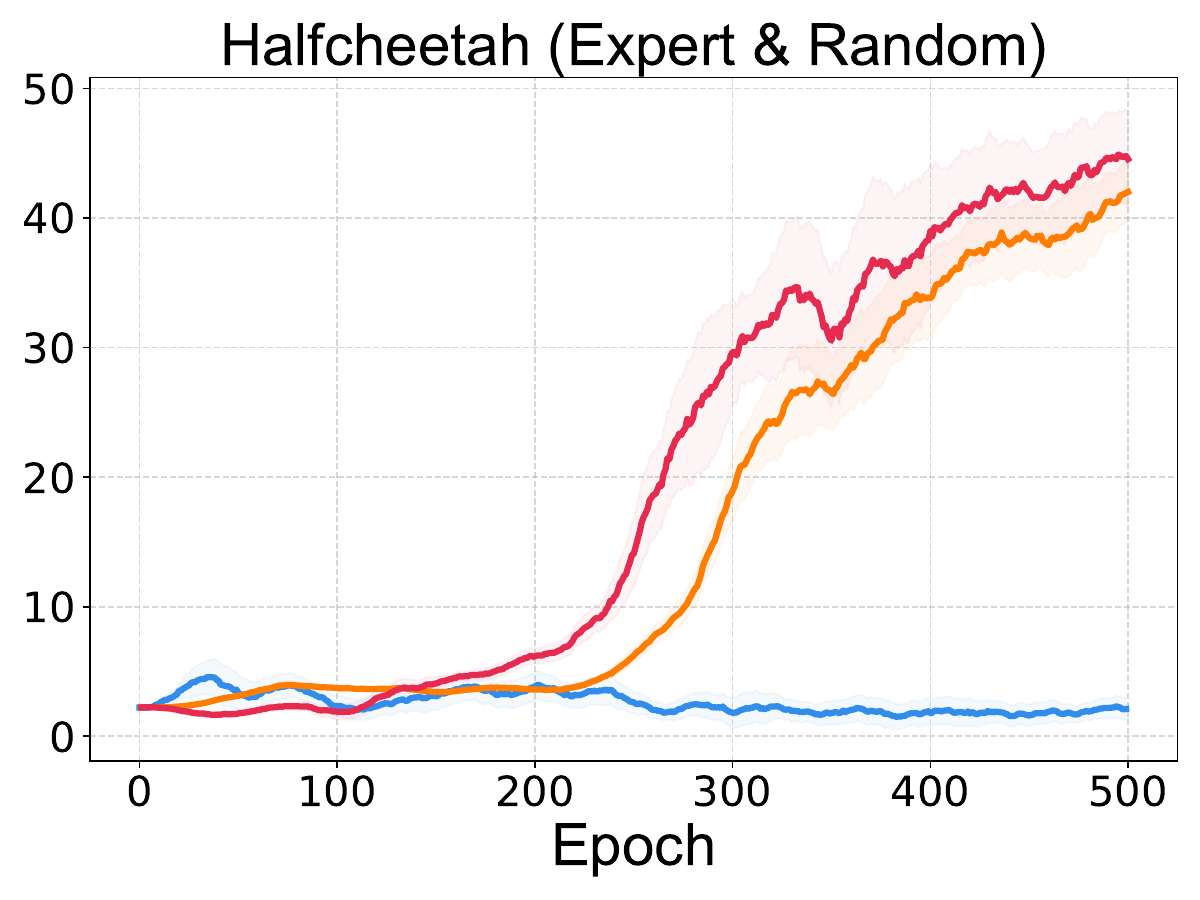}}
    \subfigure{\label{fig:anter}\includegraphics[width=0.244\textwidth]{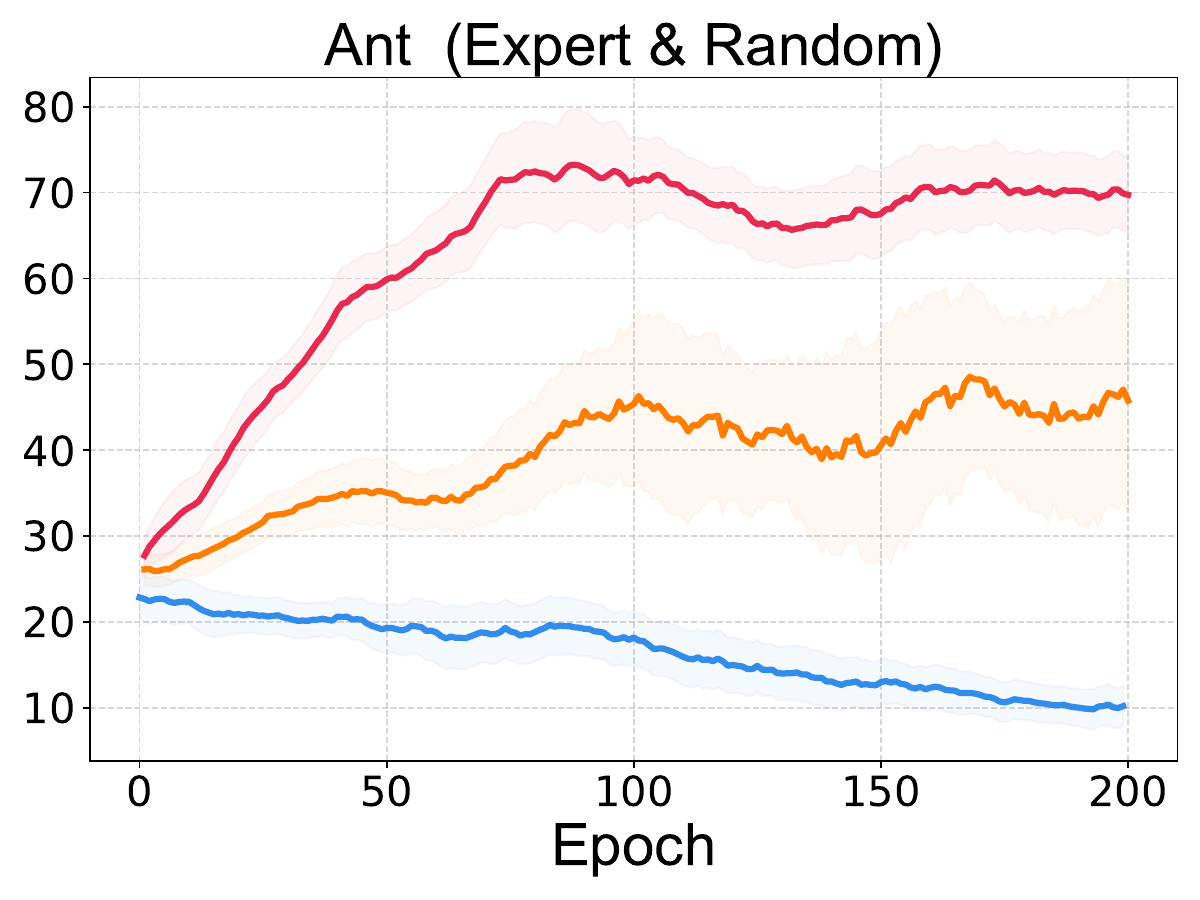}}
    \subfigure{\label{fig:wamr}\includegraphics[width=0.244\textwidth]{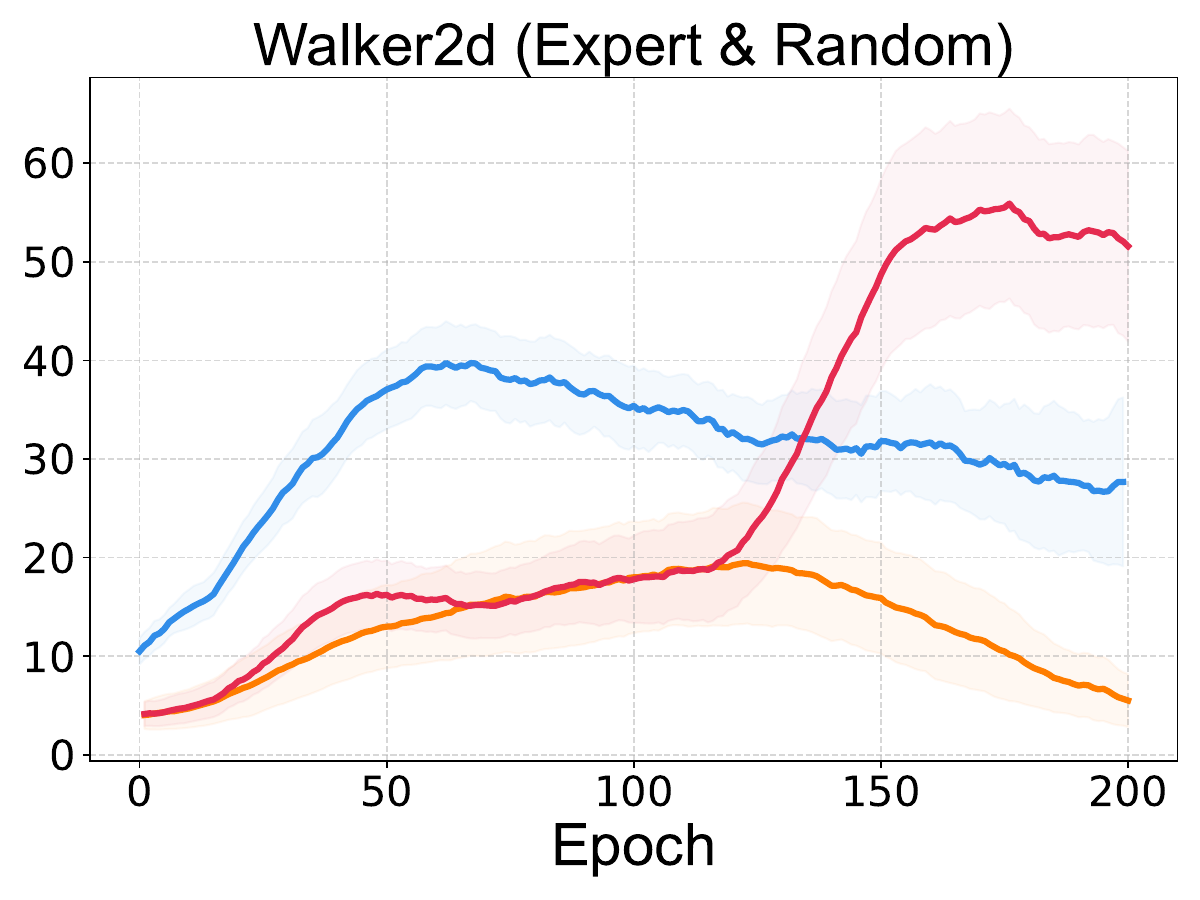}}
    \subfigure{\label{fig:hoem}\includegraphics[width=0.244\textwidth]{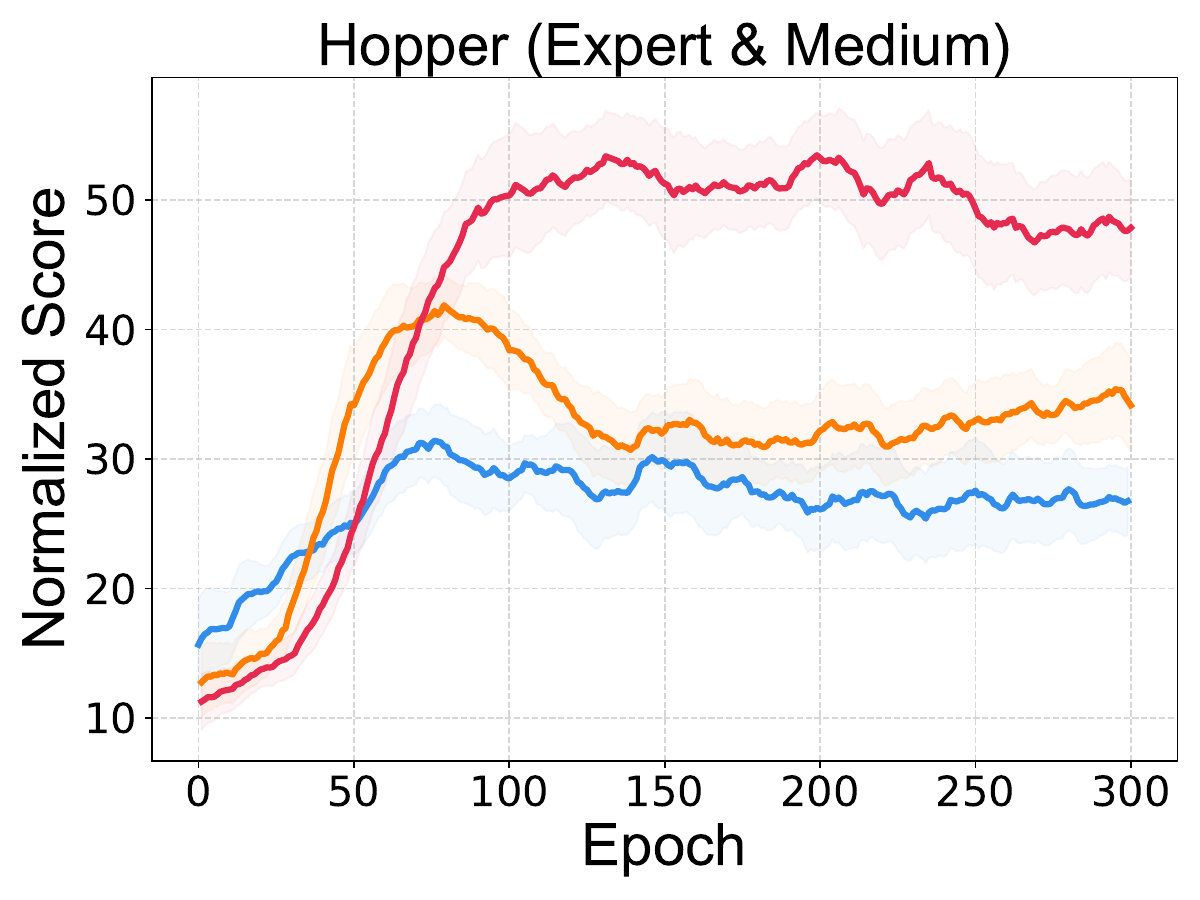}}
    \subfigure{\label{fig:hame}\includegraphics[width=0.244\textwidth]{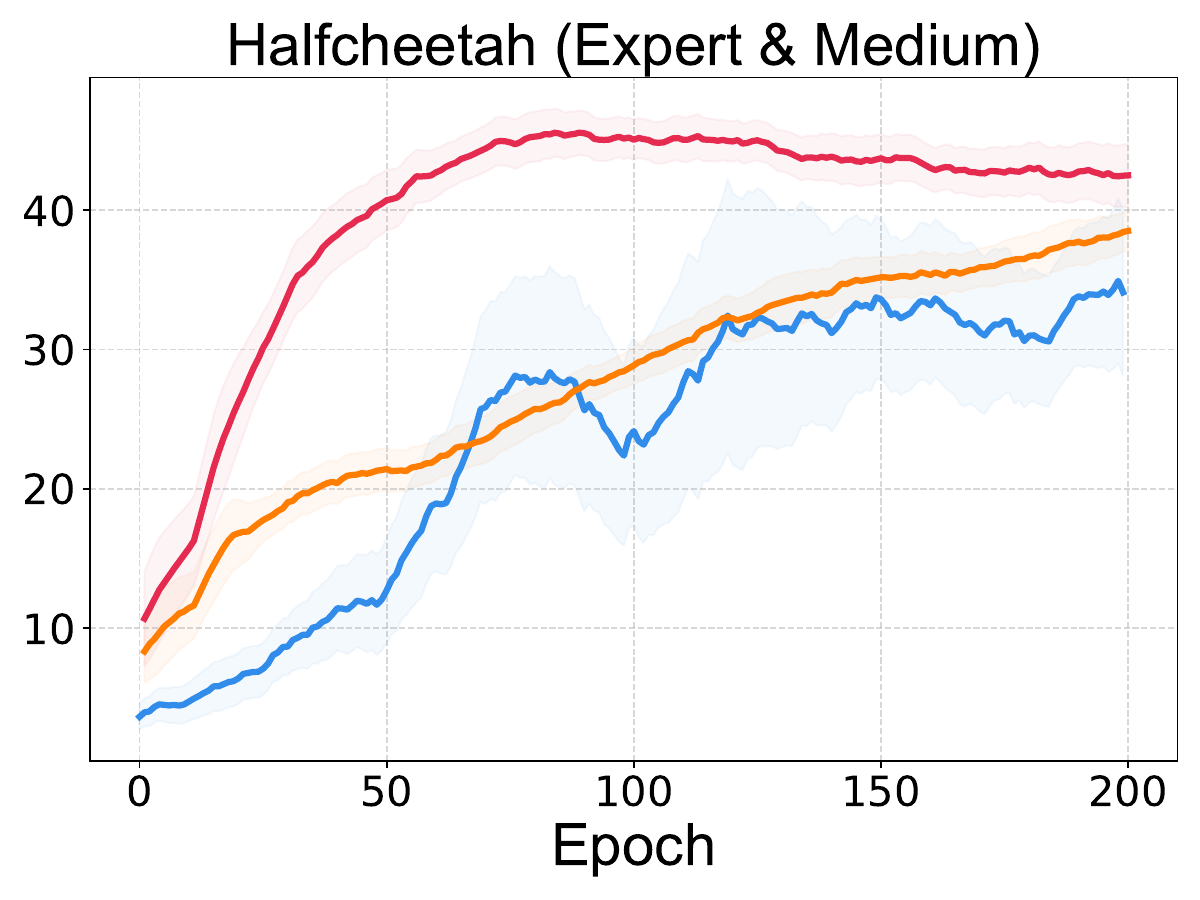}}
    \subfigure{\label{fig:antem}\includegraphics[width=0.244\textwidth]{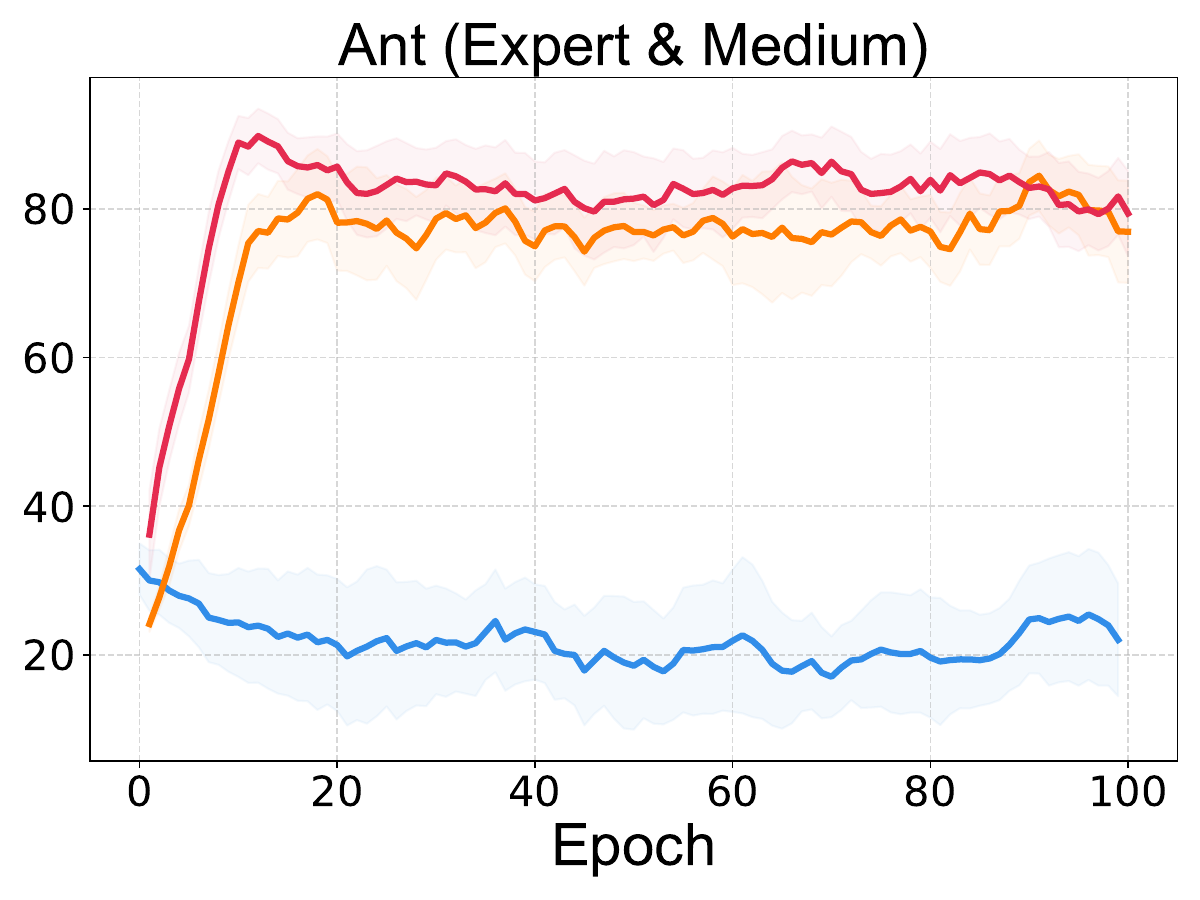}}
    \subfigure{\label{fig:wame}\includegraphics[width=0.244\textwidth]{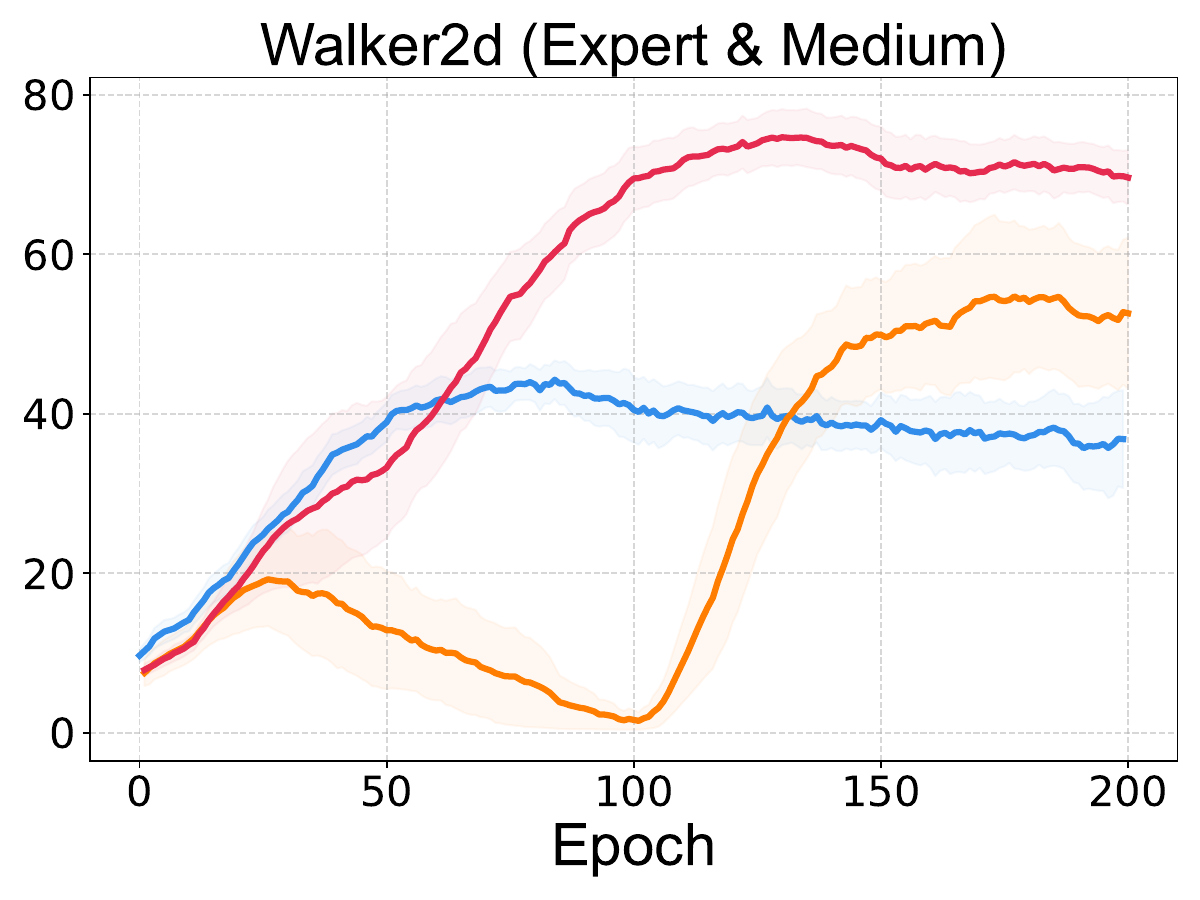}}
    \caption{Performance of FOVA against baselines.}
    \label{fig:compare}
\end{figure*}

\begin{table*}[h]
\centering
\small
\caption{FOVA Performance under Dynamic Data Quality Shifts.}
\blue{
\setlength{\tabcolsep}{14pt}
\rowcolors{2}{gray!10}{white}
\begin{tabular}{l rr rr rr rr}
\toprule
\textbf{Task} &
\multicolumn{2}{c}{\textbf{FOVA (Vote+L2)}} &
\multicolumn{2}{c}{\textbf{FOVA (Vote)}} &
\multicolumn{2}{c}{\textbf{FOVA (L2 Only)}} &
\multicolumn{2}{c}{\textbf{FOVA (Vanilla)}} \\
\cmidrule(lr){2-3}\cmidrule(lr){4-5}\cmidrule(lr){6-7}\cmidrule(lr){8-9}
& PER $\uparrow$ & BWT $\uparrow$
& PER $\uparrow$ & BWT $\uparrow$
& PER $\uparrow$ & BWT $\uparrow$
& PER $\uparrow$ & BWT $\uparrow$ \\
\midrule
Walker2d    & \textbf{3014.26} & \textbf{150.32} & 2784.72 & -190.18 & 2773.25 & -19.67 & 2629.79 & 144.28 \\
Ant         & \textbf{3085.95} & 186.23          & 3028.13 & \textbf{252.32} & 2728.50 & 114.14 & 2581.31 & 234.30 \\
Halfcheetah & \textbf{4701.41} & \textbf{-17.74} & 4654.85 & -248.30 & 4623.82 & -70.94 & 4468.63 & -461.14 \\
Hopper      & \textbf{1676.17} & 218.52          & 1550.06 & \textbf{251.07} & 1529.72 & 237.12 & 1411.74 & 134.76 \\
\midrule
\textbf{Avg.} & \textbf{3119.45} & \textbf{134.34} & 3004.44 & 16.23 & 2913.82 & 65.16 & 2772.87 & 13.05 \\
\bottomrule
\end{tabular}
\label{table:cl}}
\end{table*}

\begin{figure*}[!t]
    \centering
    \subfigure[FOVA vs. DRPO in Walker2d task with expert- and random-quality data.]{
        \label{fig:Walker2d-fova}
        \includegraphics[width=0.220\textwidth]{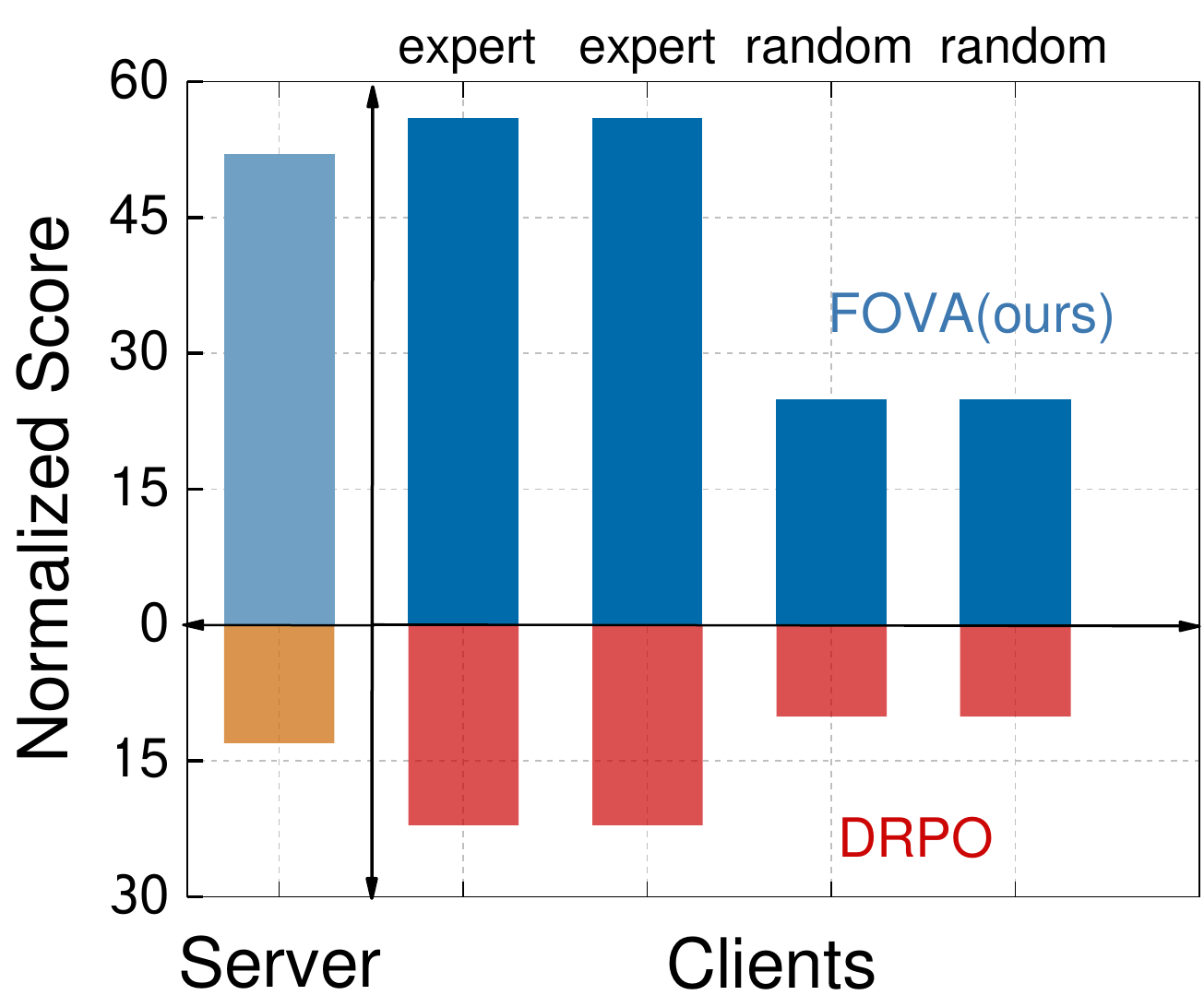}
    }
    \subfigure[Impact of number of clients in mixed-quality data.]{
        \label{fig:mix-K}
        \includegraphics[width=0.20\textwidth]{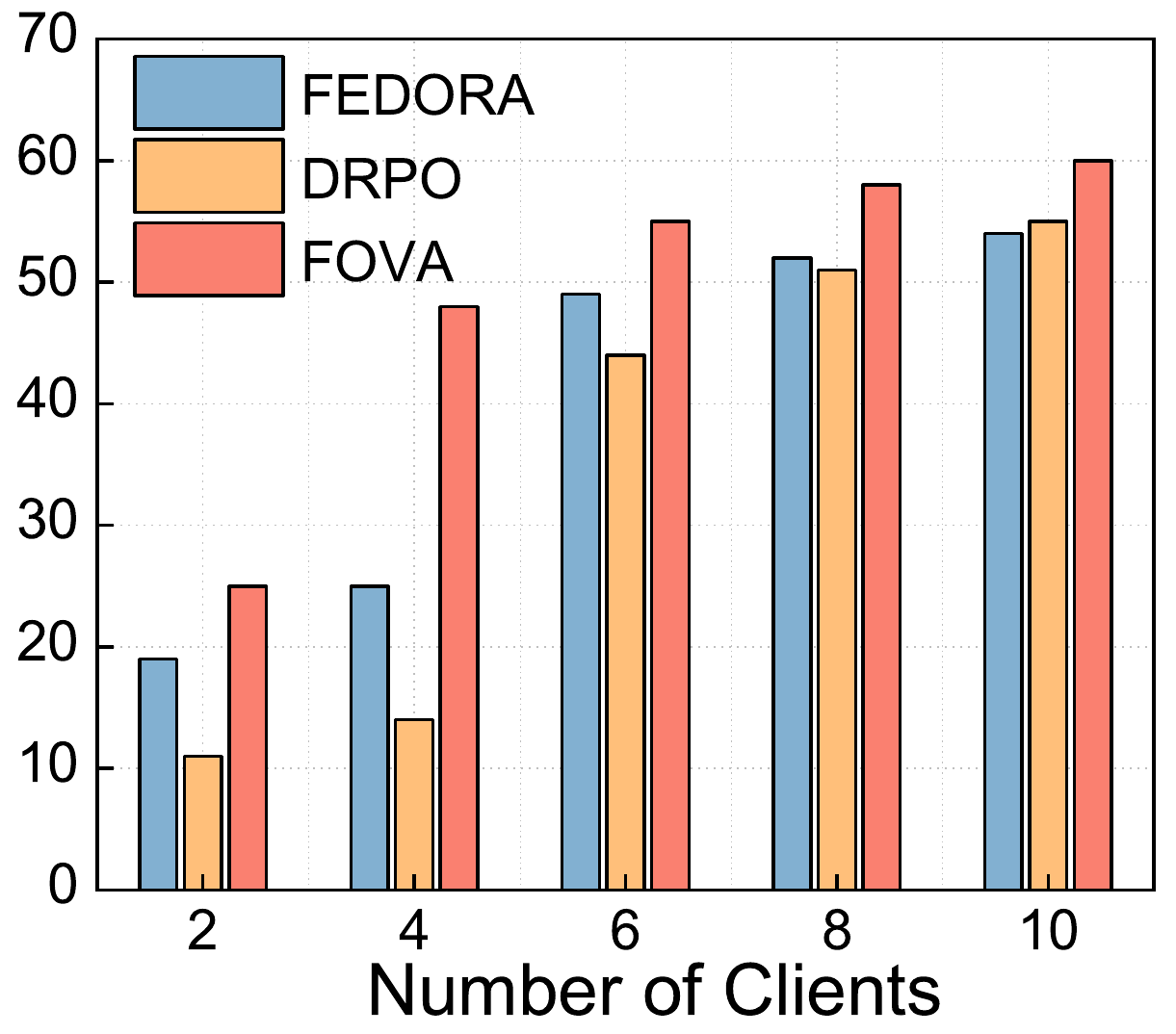}
    }
    \subfigure[Normalized score of Hopper under mixed-quality data.]{
        \label{fig:intro-hopper}
        \includegraphics[width=0.240\textwidth]{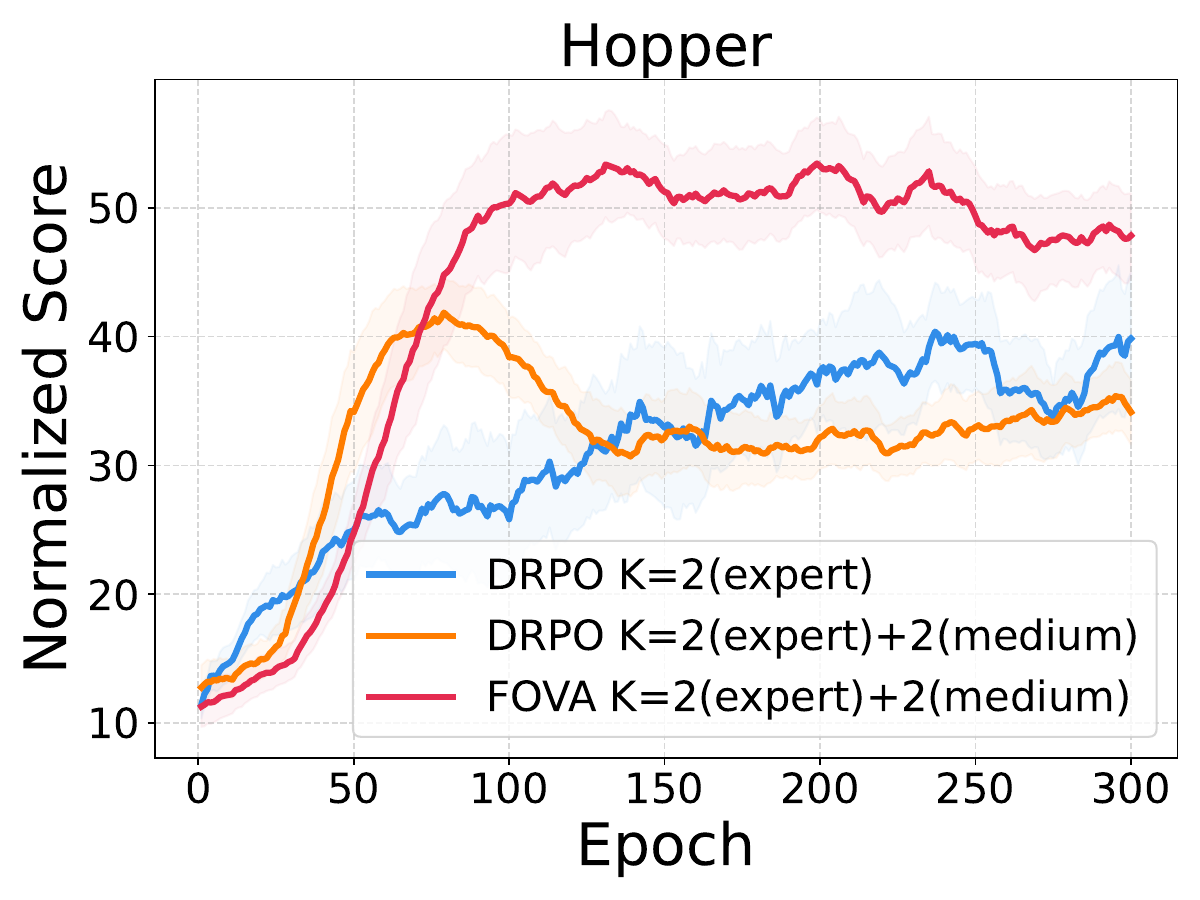}
    }
    \subfigure[Normalized score of Walker2d under mixed-quality data.]{
        \label{fig:intro1-walker2d}
        \includegraphics[width=0.240\textwidth]{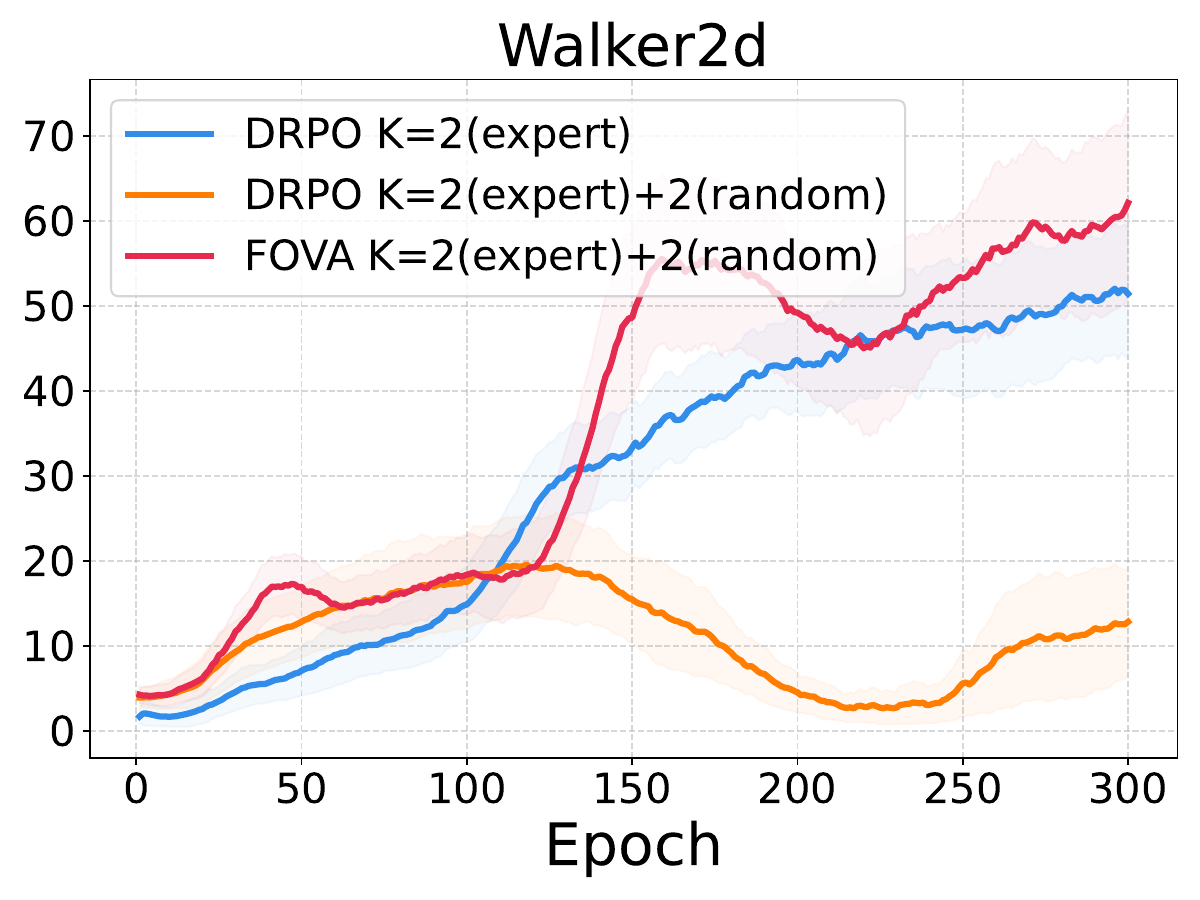}
    }
    \caption{Performance comparison on Gym-MuJoCo tasks with mixed-quality data.}
    \label{fig:add-heter}
\end{figure*}

\paragraph{Consistent Optimization Objective}
To evaluate the role of the consistent optimization objective in our method compared to the baseline methods, we utilize the Walkerd-medium-expert dataset to analyze the differences between overall server-side return and average client-side return, as illustrated in Fig.~\ref{fig:incon-datasize}. In our experiments, we fixed the number of clients at five and varied the data size from 1,000 to 50,000 state-action pairs. The results demonstrate that FOVA consistently outperforms the client-side return in terms of server-side policy performance. Although FOVA does not always surpass the baseline methods in client-side policies, its server-side policies frequently exceed those of the corresponding baseline methods. This observation suggests that the baseline methods may allocate computational resources toward enhancing client-side policies, whereas FOVA focuses on directly optimizing the final target server-side policy. Consequently, FOVA effectively maintains a consistent optimization objective, leading to superior performance in this domain.

In Fig.~\ref{fig:incon-cover}, we evaluate our FOVA method within the Hopper-medium-replay to address the inconsistency issue presented in Fig.~\ref{fig:intro-incon} of Section \ref{sec:introduction}. The experimental settings are consistent with those in Fig.~\ref{fig:intro-incon}. Notably, unlike existing approaches, our method maintains the server policy return consistently higher than the average client return throughout training iterations. This demonstrates that FOVA reliably focuses on server policy training, resulting in performance that surpasses current methods.

In general, a consistent optimization objective always leads to a better server policy as the output policy.

\begin{table*}[h]
\centering
\small
\caption{GPU Memory and Runtime Comparison}
\blue{
\setlength{\tabcolsep}{10pt}
\label{table:gpu_comparison}
\begin{tabular}{lcccccc}
\toprule
 & \textbf{Fed-TD3BC} & \textbf{Fed-CQL} & \textbf{FEDORA} & \textbf{DRPO} & \textbf{FOVA (w/o vote)} & \textbf{FOVA (w/ vote)} \\
\midrule
\textbf{GPU Mem.} (MB)       & 346   & 428   & 484   & 524  & 452     & 522     \\
\textbf{Runtime} (s/epoch)   & 5.54  & 7.61  & 14.52 & 9.84 & 9.60    & 9.73    \\
\bottomrule
\end{tabular}}
\end{table*}

\begin{figure*}[h]
	\centering
	\subfigure[The impact of $\lambda$ and $\beta$]{\label{fig:parameter}\includegraphics[width=0.214\textwidth]{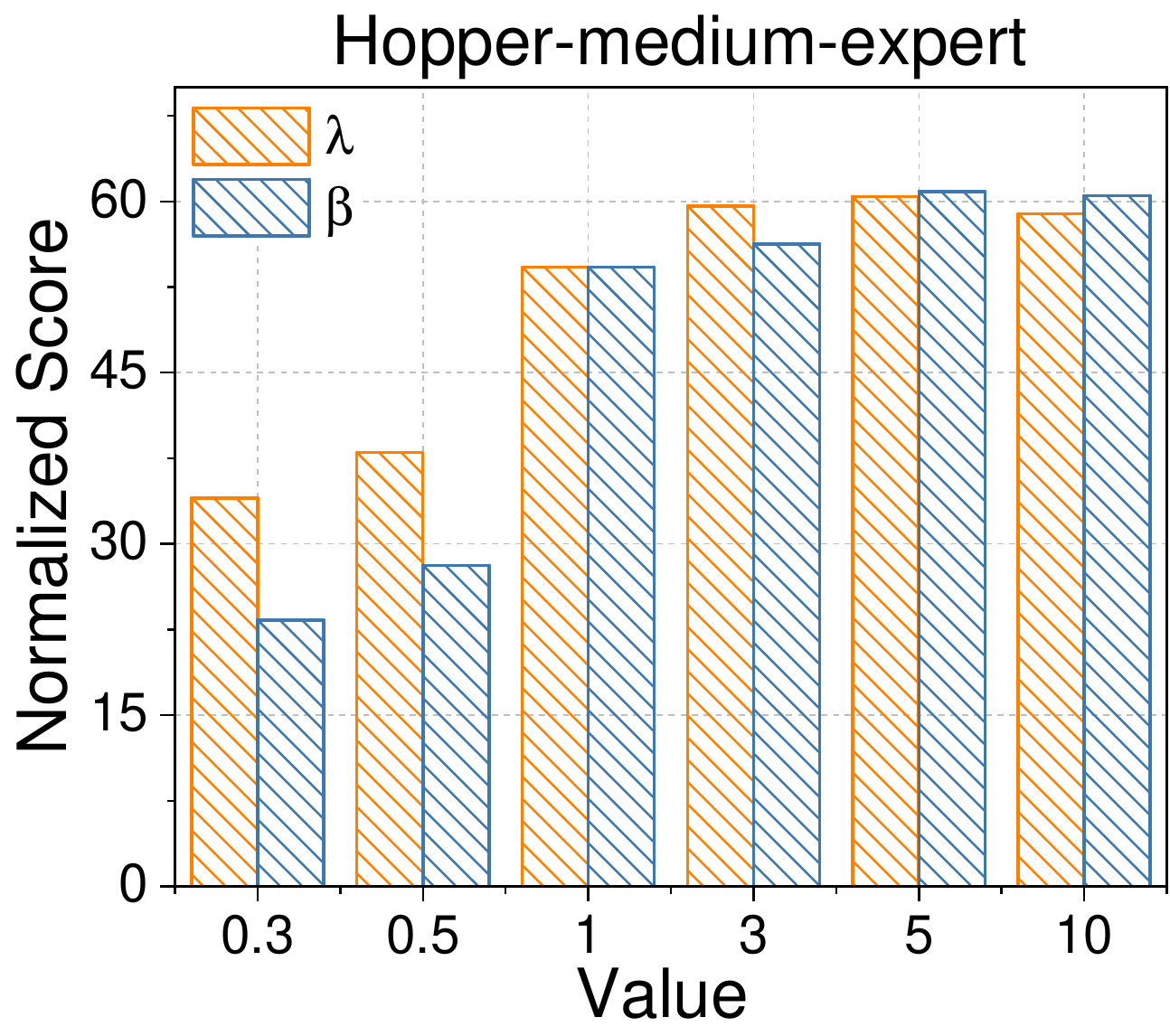}}
        \hspace{0.01\textwidth}
	\subfigure[The impact of $K$]{\label{fig:noc}\includegraphics[width=0.203\textwidth]{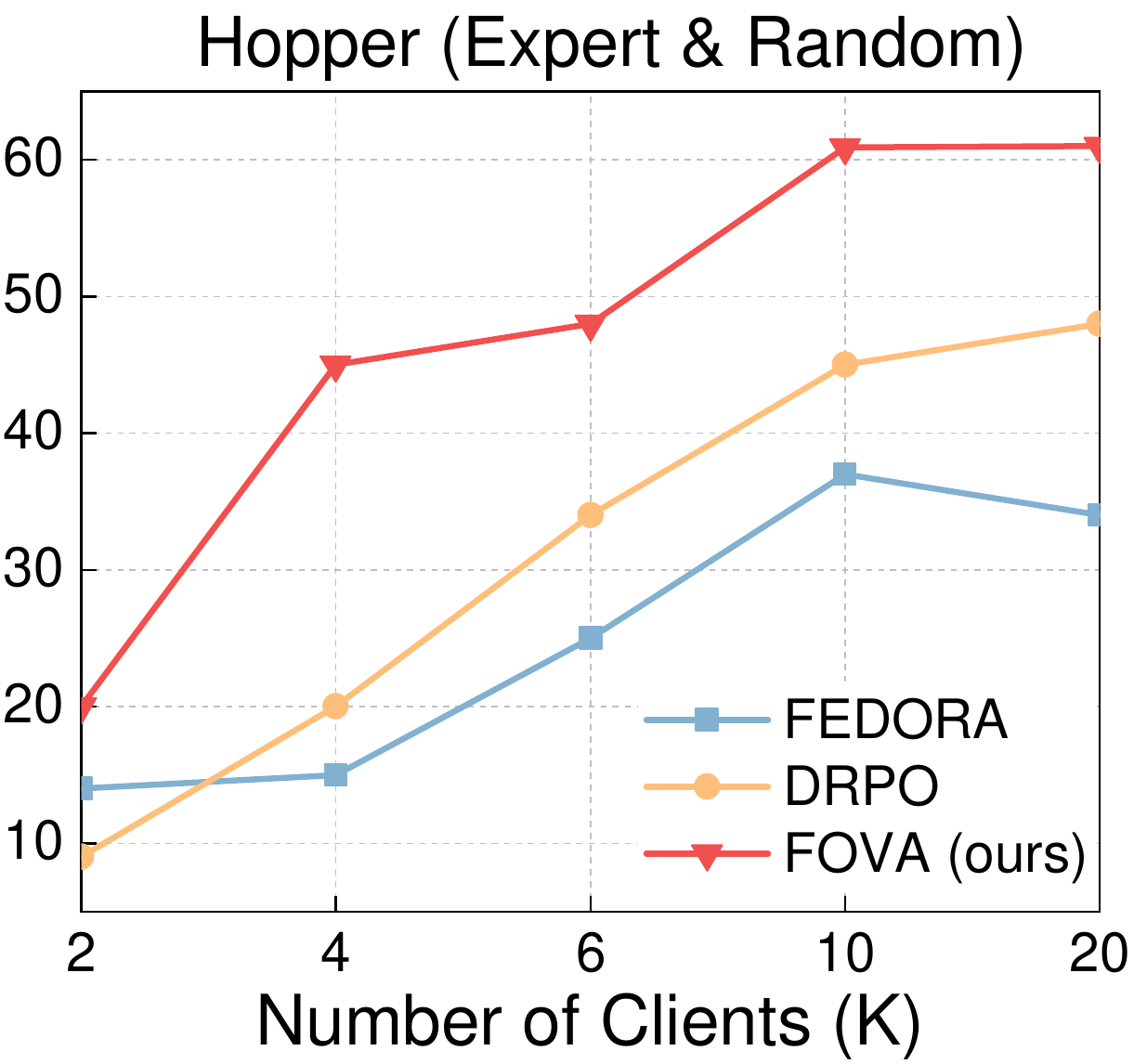}}
	\subfigure[The impact of datasize]{\label{fig:datasize}\includegraphics[width=0.252\textwidth]{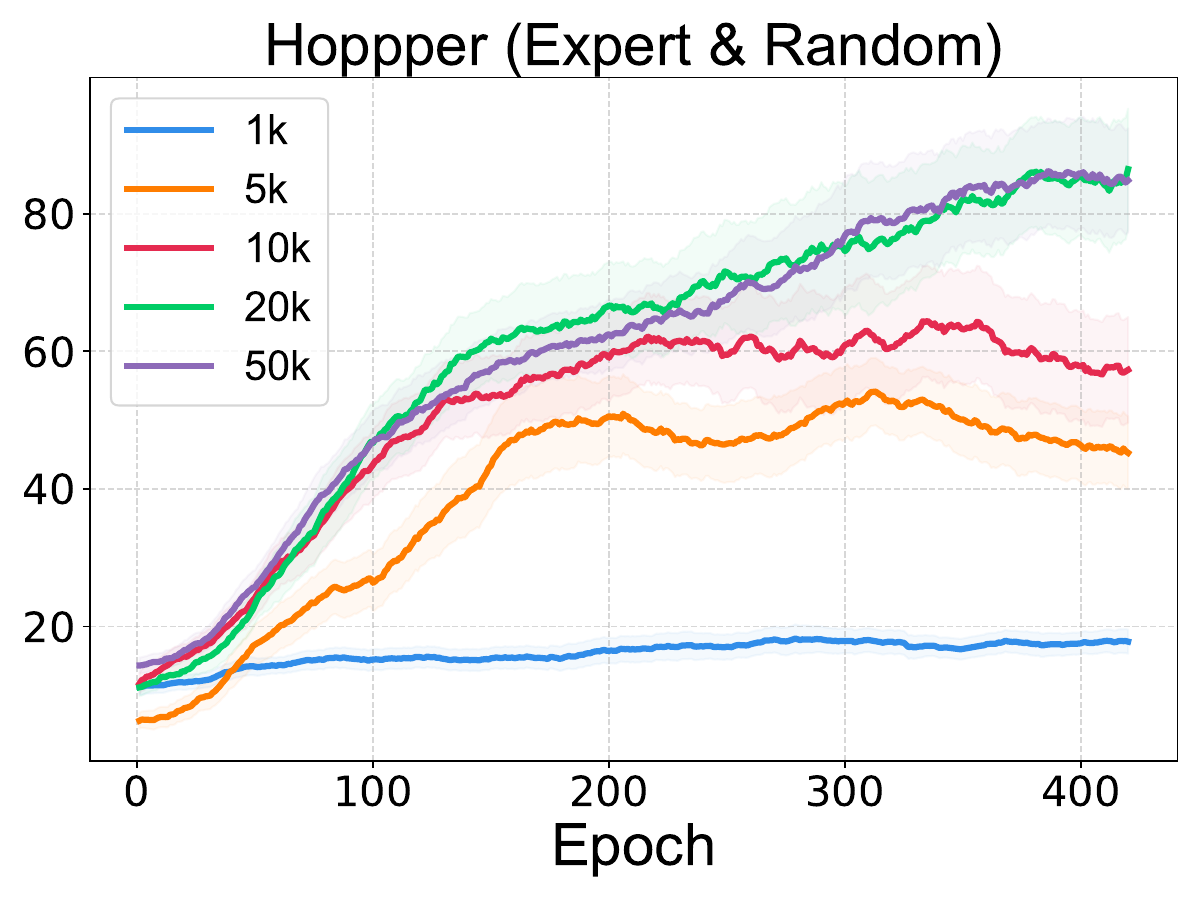}}
	\hspace{0.00\textwidth}
	\subfigure[Ablation study]{\label{fig:Ablation}\includegraphics[width=0.252\textwidth]{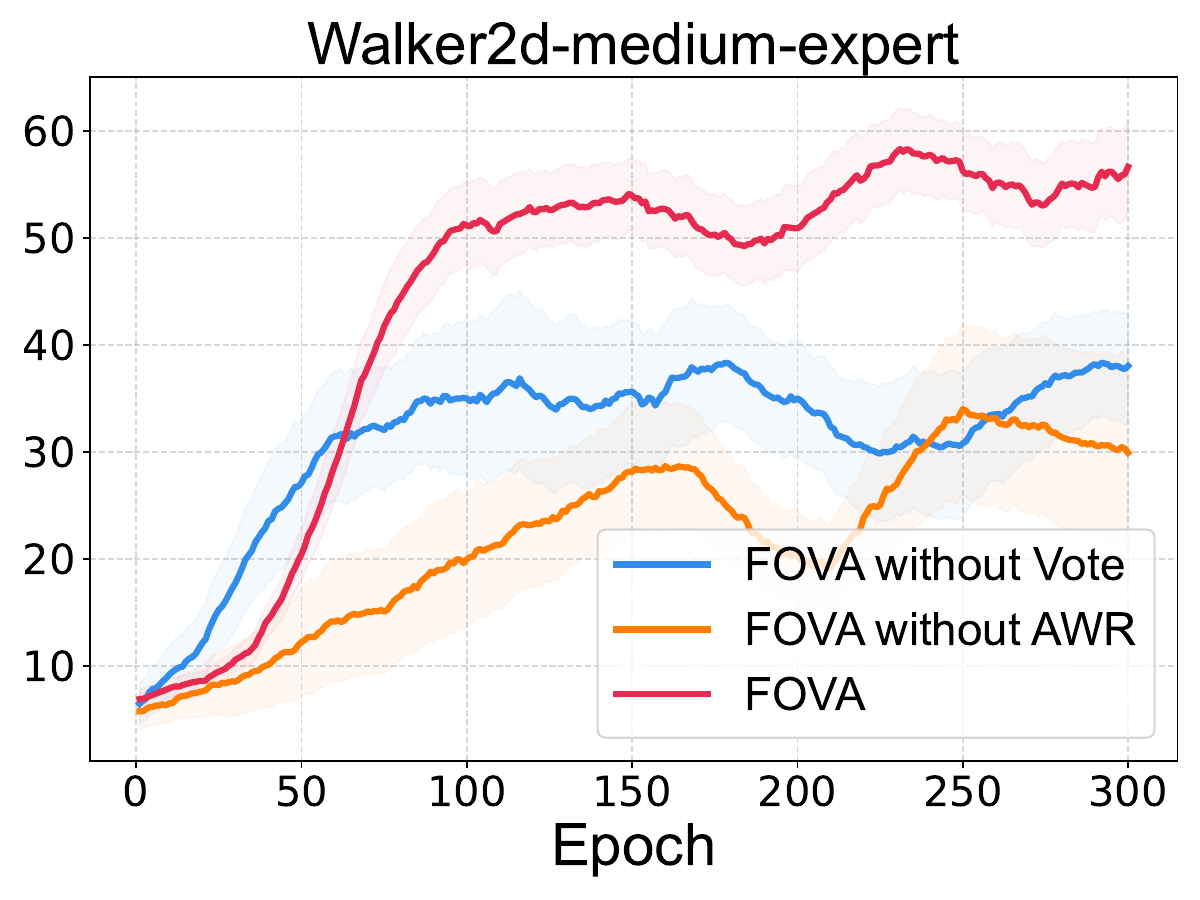}}
        \hspace{0.01\textwidth}
    \vspace{-0pt}
	\caption{\centering Impact of parameter and ablation study.}
 \label{fig:impact}
\end{figure*}

\paragraph{Mixed-quality dataset} 
First, we conducted eight experiments using mixed-quality data, configuring two clients with high-quality (Expert) datasets and two clients with low-quality (Medium/Random) datasets. FEDORA and DRPO were selected as comparison algorithms due to their strong performance in Table \ref{tab:compare}. As illustrated in Fig.~\ref{fig:compare}, the learning curves demonstrate that FOVA’s vote mechanism effectively selects superior policies in mixed-quality environments, leading to more efficient and stable learning progress. 
Compared to state-of-the-art methods, our algorithm's final performance improved by approximately 12\%–56\%, strongly supporting FOVA as an effective solution for offline FRL challenges.

Second, we revisit the experiment from Section \ref{sec:introduction}, which involves 1 server, half of clients with high-quality (Expert) data, and another half of clients with low-quality (Random) data. We carry out two experiments to explore how FOVA solves the mixed-quality data case. depicted in Fig.~\ref{fig:Walker2d-fova}, maintaining the same setup as in Fig.~\ref{fig:intro-mixed} and \ref{fig:moti-heter}, although FOVA's performance varies between expert and random clients.
Clearly, FOVA's random clients can distinguish good and bad policies, thus performing better, and the expert clients effectively avoid being contaminated by low-quality client policies. This is due to the evaluation and discrimination capabilities brought by the vote mechanism.
Furthermore, we compared FOVA with DRPO and FEDORA when the number of clients ranges from 2 to 10. FOVA outperforms them for all values of $K$, as shown in Fig.~\ref{fig:mix-K}.

Third, we further explore a mixed scenario in the context of FL. As depicted in Fig.~\ref{fig:intro-hopper} and \ref{fig:intro1-walker2d}, there are three lines in red, yellow, and blue:
(i) The blue line represents the application of the state-of-the-art (SOTA) algorithm DRPO on expert-quality datasets. These datasets are distributed across two clients, with each client having 5,000 state-action pairs.
(ii) The yellow line shows the performance of DRPO on a mixed-quality dataset configuration. This configuration consists of two expert-quality datasets and two additional medium-or random-quality datasets. Each client also has 5,000 state-action pairs.
(iii) The red line indicates that our proposed method (OURS) outperforms DRPO under the same configuration as in (ii).
Based on the above-described experimental settings, we conducted tests on both the hopper and walker2d datasets.
As illustrated in Fig.~\ref{fig:intro-hopper} and \ref{fig:intro1-walker2d}, the addition of mixed-quality datasets leads to a significant decline in the overall performance of the baselines. The reason is that mixed-quality datasets may contain low-quality local data. During the local updates processes, such data can result in the generation of low-quality local policies. Without an effective mechanism to evaluate policy qualities, it becomes difficult to identify these low-quality local policies. These low-quality policies may contaminate other policies, ultimately causing a degradation in overall performance.

\blue{
To further validate our method under scenarios with mixed-quality data, 
we investigate a non-stationary regime where data quality varies across time. 
Specifically, we adopt a sequence of datasets 
\{\textit{random, random, expert, expert, replay, replay, medium, medium}\} across all four clients.
Following continual learning practices \cite{derakhshani2021kernel,gai2023oer}, 
we report Average Performance (PER) and Backward Transfer (BWT), 
where higher values indicate better results. 
{PER and BWT are defined as
$
\mathrm{PER}=\tfrac{1}{K}\sum_{k=1}^{K}a_{K,k}, \quad
\mathrm{BWT}=\tfrac{1}{K-1}\sum_{k=1}^{K-1}\big(a_{K,k}-a_{k,k}\big),
$
where $a_{i,j}$ denotes the performance on dataset $j$ after training up to dataset $i$.}
As shown in Table~\ref{table:cl}, non-stationary data quality induces forgetting effects, 
reflecting the challenge of continually adapting to heterogeneous client data. 
Our Vote module alleviates this degradation by aggregating client policy preferences, 
thereby stabilizing the global policy update. 
In addition, we introduce an L2 regularization term on the $Q$-function 
to penalize deviations from the previously trained $Q$-function:
$
\mathcal{L}_{\text{reg}} = \alpha_2\,\mathbb{E}_{\bfs,\bfa\sim \mathcal{D}_k}\big[\|Q(\bfs,\bfa)-Q_{\text{prev}}(\bfs,\bfa)\|^2\big].
$
The combination of Vote and L2 regularization achieves the best trade-off 
between stability and plasticity. 
These results demonstrate that FOVA effectively adapts to dynamically changing data quality, 
maintaining both overall PER and BWT.
}


{
\paragraph{Impact of the Critical Methods and Parameters}
To answer the question regarding the influence of the vote mechanism, AWR method, and key hyperparameters on our FOVA algorithm, we conducted a series of experiments, the results of which are depicted in Fig. \ref{fig:impact}.

We perform ablation experiments on the Walker2d-medium-expert dataset, comparing our FOVA method with versions of FOVA that lack the vote mechanism and AWR, respectively. As shown in Fig. \ref{fig:Ablation}, FOVA outperforms both modified versions, indicating that both the vote mechanism and AWR play crucial roles in enhancing performance.

Additionally, the impact of the critical hyperparameters $\lambda$ and $\beta$ is shown in Fig. \ref{fig:parameter}. In our study, we varied one hyperparameter at a time, fixing the other at the value of 1.  The results demonstrate that the algorithm's performance consistently improves $\lambda$ and $\beta$ increases until it reaches optimal performance. This aligns with our analysis in Theorem \ref{thm:policy-improvement-local-app} and Theorem \ref{thm:policy-improvement-global-app}.
Furthermore, to show the impact of the number of clients ($K$), we fixed the data size at 5k, with half of the clients using expert datasets and the other half using random datasets, tested on the Hopper dataset. As shown in Fig.~\ref{fig:noc}, the performance of FOVA improves with the increase in the number of clients and surpasses the baseline methods across all client numbers.
To demonstrate the impact of dataset size on learning speed, we used different numbers of state-action pairs in each agent, ranging from 1k to 50k, while keeping the number of agents at 4. As shown in Fig.~\ref{fig:datasize}, FOVA's performance improves with the increase in pre-collected data.
}

To quantify the computational overhead of the vote mechanism, we benchmark six offline FRL algorithms on the Adroit suite under identical settings and report GPU memory/runtime in Table~\ref{table:gpu_comparison}. FOVA with vote uses slightly more resources than FOVA without vote due to additional $Q$-value evaluations (522\,MB, 9.73\,s/epoch vs.\ 452\,MB, 9.60\,s/epoch), yielding a negligible $<\!2\%$ runtime increase. While simpler baselines such as Fed-TD3BC (346\,MB, 5.54\,s/epoch) and Fed-CQL (428\,MB, 7.61\,s/epoch) are cheaper, they typically achieve inferior performance; FEDORA is the most demanding (484\,MB, 14.52\,s/epoch). Overall, FOVA’s overhead is modest relative to its gains.

\section{Limitations and Future Work}\label{sec:future-work}
\blue{We note that the vote mechanism's requirement for multiple Q-value evaluations incurs non-negligible GPU memory overhead. Furthermore, certain hyperparameters currently lack adaptive tuning capabilities. These limitations motivate two concrete directions for future investigation: developing memory-efficient vote through selective value estimation or distributed computation, and designing an adaptive hyperparameter tuning approach tailored to offline FRL.} 

\blue{In addition, we plan to enhance privacy protection in FOVA by exploring techniques such as differential privacy and secure multi-party computation to ensure data confidentiality during aggregation. Another important direction is adapting FOVA to dynamic environments by developing algorithms that can detect and respond to changes in real time, incorporating memory or recurrent structures to handle sequential dependencies. Furthermore, we will work on improving the scalability of FOVA to cope with large-scale FL with many clients and extensive datasets, optimizing communication and computational efficiency for broader practical applications.}


\section{Conclusion}\label{sec:conclusion}
\blue{In this paper, we introduce FOVA, an offline FRL framework that effectively addresses the challenges of mixed-quality data.} By employing vote mechanism and AWR, FOVA significantly improves overall performance. Theoretical analysis shows that FOVA guarantees policy improvement, and extensive experiments demonstrate its superiority over existing methods, particularly in scenarios with mixed-quality data.

\begin{appendices}

\blue{\section{Theoretical Analysis}\label{sec:improve_bp}}
\blue{
In this section, we provide a theoretical analysis of FOVA.  
We begin by examining the conservatism of our approach. This component addresses the extrapolation error that naturally arises in offline RL training during local updates \cite{peng2019advantage,kumar2020conservative,kostrikov2021iql}.  
Next, we show that our offline FRL method, FOVA, guarantees strict policy improvement. The guarantee is established through Theorems~\ref{thm:policy-improvement-global-app} and~\ref{thm:policy-improvement-local-app}.  
To prepare for these results, we first present Lemma~\ref{lemma:JMDP-new}, which bounds the return in the empirical MDP $\widetilde{M}$ relative to the true MDP $M$.  
We then prove the two policy improvement theorems. Our proofs are partly inspired by prior frameworks for policy improvement in offline RL \cite{kumar2020conservative,yu2021combo,yue2024federated,lin2022modelbased}.  }

\subsection{Theoretical Analysis on Conservatism}\label{subsec:conservatism}
\blue{
Similar to \cite{auer2008near,osband2017posterior,cql}, we assume concentration properties
of the reward function and the transition dynamics, as follows:}
\begin{assumption}
\label{assu:cql}
    For all $ \bfs, \bfa \in \mathcal{D} $, the reward function $ r $ and the transition probabilities $ \transitions $ satisfy the following with high probability greater than $ 1 - \delta $:
    \begin{align}
        \Phi^r(\bfs, \bfa) &= |r - r(\bfs, \bfa)| \leq 
        C^{\mathrm{vcql}}_{r, \delta} /\sqrt{|\mathcal{D}(\bfs, \bfa)|}, \nonumber\\
        \Phi^T(\bfs'|\bfs, \bfa) &= |\hat{\transitions}(\bfs'|\bfs, \bfa) - \transitions(\bfs'|\bfs, \bfa)| 
        \leq
        C^{\mathrm{vcql}}_{\transitions, \delta}/\sqrt{|\mathcal{D}(\bfs, \bfa)|}\nonumber,
    \end{align}
where $C_{r,\delta}$ and $C_{T,\delta}$ are constants that depend on $\delta$ with a $\sqrt{\log(1/\delta)}$ dependency.
\end{assumption}
\blue{
The Assumption states a concentration property that is actually considered on each client $k$. For any $(s,a)$ in $\mathcal{D}$, the estimation errors of the reward and transition relative to the true $r,T$ shrink at the rate $1/\sqrt{|\mathcal{D}(s,a)|}$.
\footnote{ For simplicity, we drop the subscript $k$ since this section focuses on local training and does not rely on cross-client differences, such as $\mathcal{D}_k$, $r_k$ and $\hat T_k$.
}
The constants $C_{r,\delta}$ and $C_{T,\delta}$ capture noise, model capacity, and the $\log(1/\delta)$ dependence, while the $\ell_1$ norm measures the discrepancy between transition distributions.
Such an assumption is common in RL and ensures the soundness of the analysis \cite{bartlett2012regal,lin2022modelbased,lyu2022mildly,yu2021combo,yue2024CJE,yue2024federated,auer2008near,osband2017posterior,cql}. In practice, it is particularly suitable for offline federated RL where clients share the same underlying dynamics (e.g., fleets, robot swarms, chain-store recommendation, or edge-device control), rewards are bounded or normalized, and each client has sufficient coverage of its frequent state-action pairs, enabling stable and efficient aggregation.}

To avoid any trivial bound, we assume that the cardinality of a state-action pair in $D$ is non-zero \cite{yu2021combo,kumar2020conservative,lyu2022mildly,lin2022modelbased}. When $(\bfs,\bfa) \notin D$, we have ${\vert D(\bfs,\bfa) \vert}/{\vert D \vert} \geq \delta$ \cite{yue2024federated}. 
Leveraging the assumption, we can bound the gap between the empirical and actual Bellman operators for policy $ \pi^v_k $:
\begin{align}
\label{eq:assu}
    &\left|\hat{\bellman}^\vpolicy \hat{Q}^{\tau}  - {\bellman}^\vpolicy \hat{Q}^{\tau} \right| \\
    \leq& \Phi^r(\bfs, \bfa)
    + \gamma \Big| \sum_{\bfs'} \mathbb{E}_{\vpolicy(\bfa'|\bfs')}\left[\hat{Q}^{\tau}(\bfs', \bfa')\right] \cdot \Phi^T(\bfs'|\bfs, \bfa) \Big|\nonumber\\
    \mathop{\le}^{(a)}& 
    \left(C^{\mathrm{vcql}}_{r, \delta} + 2\gamma r_{\max} C^{\mathrm{vcql}}_{\transitions, \delta} /{(1 - \gamma)}\right)/\sqrt{|\mathcal{D}(\bfs, \bfa)|}\nonumber, 
\end{align}
where (a) holds based on the fact that \(\hat Q^{\tau}(\bs,\ba) \leq \frac{r_{\max}}{1-\gamma},~\forall \rvs, \rva   \) \cite{kumar2020conservative,agarwal2019reinforcement}.
Next, we show that the $\hat{Q}^{\vpolicy}:=\lim_{\tau \to \infty} \hat{Q}^{\tau}$ learned by iterating Eq.\eqref{eq:vcql} lower-bounds the true Q-function.

\begin{lemma}[]
\label{thm:vcql_underestimates}
The value of the policy under the Q-function from Eq.~\eqref{eq:vcql}, $\hat{V}^\vpolicy(\bs) = \E_{\vpolicy(\ba|\bs)}[\hat{Q}^{\vpolicy}(\bs, \ba)]$, lower-bounds the true value of the policy obtained via exact policy evaluation, $V^\vpolicy(\bs) = \E_{\vpolicy(\ba|\bs)}[Q^\vpolicy(\bs, \ba)]$, for all $\bs \in \mathcal{D}$, according to:

\begin{equation*}
 \hat{V}^\vpolicy(\bs) \leq V^\vpolicy(\bs) - \alpha \underbrace{\left[\Gamma_k^{-1} \E_{\vpolicy}\left[\frac{\vpolicy}{\policy_\beta} - 1 \right] \right](\bs)}_{(a) \ge 0} + \underbrace{\left[ \Gamma_k^{-1} C_{v}\right](\bs)}_{\mathrm{(b) sampling~error}},
\end{equation*}
where $\Gamma_k = \left(I - \gamma P^{\vpolicy} \right)$ is non-negative entries and $C_{v}=\frac{C^{{vcql}}_{r, T, \delta} r_{\max}}{(1- \gamma) \sqrt{|\mathcal{D}|}}$.
For all $\bs \in \mathcal{D}$,~
if $\alpha > \frac{C^{{vcql}}_{r, T, \delta} r_{\max}}{1 - \gamma} \cdot \max_{\bs \in \mathcal{D}} \frac{1}{|\sqrt{|\mathcal{D}(\bs)|}} \cdot \left[\sum_{\ba} \vpolicy(\ba|\bs) (\frac{\vpolicy(\ba|\bs)}{\policy_\beta(\ba|\bs))} - 1)\right]^{-1}$
,~ $\hat{V}^\vpolicy(\bs) \leq {V}^\vpolicy(\bs)$, with probability $\geq 1 - \delta$. When $\hat{\bellman}^\vpolicy = {\bellman}^\vpolicy$, then any $\alpha > 0$ guarantees $\hat{V}^\vpolicy(\bs) \leq V^\vpolicy(\bs), \forall \bs \in \mathcal{D}$.
\end{lemma}
\begin{proof}
Building on the Assumption \ref{assu:cql} laid, we extend our analysis to demonstrate that VCQL.
In order to start, we prove this theorem in the absence of sampling error, drawing on the foundation of \cite[Theorem 3.2]{kumar2020conservative}.
The absence of sampling error assumes $\hat{\bellman}^\vpolicy = \bellman^\vpolicy$.
Setting the derivative of the objective function defined in Eq.~\eqref{eq:vcql} to zero, we obtain
\begin{equation}
    \forall ~\bs, \ba, \tau ~~ \hat{Q}^{\tau+1} (\bs, \ba) = \bellman^\vpolicy \hat{Q}^{\tau}(\bs, \ba) - \alpha \left[\frac{\vpolicy(\ba|\bs)}{\behaviorv(\ba|\bs)} - 1 \right].
    \label{eqn:q_function_modified_eval}
\end{equation}
The value $\hat{V}^{\tau+1}$ is underestimated, because:
\begin{align}
\label{eqn:value_recursion}
    \hat{V}^{\tau+1}(\bs) 
    = \bellman^\vpolicy \hat{V}^{\tau} (\bs) - \alpha D_{\textit{VCQL}}(\bs),
\end{align}
where 
$D_{\textit{VCQL}}(\bs): = \sum_{\ba} \vpolicy(\ba|\bs) \left[\frac{\vpolicy(\ba|\bs)}{\behaviorv(\ba|\bs)} - 1 \right]$ is a non-negative term according to
\begin{align}
\label{eq:DVCQL}
    D_{\textit{VCQL}}(\bs)
    &= \sum_{\ba} \vpolicy(\ba|\bs) \left[\frac{\vpolicy(\ba|\bs)}{\behaviorv(\ba|\bs)} - 1 \right]\\
    = & \sum_{\ba} (\vpolicy(\ba|\bs) - \behaviorv(\ba|\bs) + \behaviorv(\ba|\bs)) \left[\frac{\vpolicy(\ba|\bs)}{\behaviorv(\ba|\bs)} - 1 \right]\nonumber\\
    = &
    \sum_{\ba} (\vpolicy(\ba|\bs) - \behaviorv(\ba|\bs)) \left[ \frac{\vpolicy(\ba|\bs) - \behaviorv(\ba|\bs)}{\behaviorv(\ba|\bs)} \right]\nonumber\\
    & +\sum_{\ba} \behaviorv(\ba|\bs) \left[\frac{\vpolicy(\ba|\bs)}{\behaviorv(\ba|\bs)} - 1 \right]\nonumber\\
    \mathop{=}^{(i)}&
    \sum_{\ba}  \frac{\left(\vpolicy(\ba|\bs) - \behaviorv(\ba|\bs) \right)^2}{\behaviorv(\ba|\bs)} ~ +~ 0 \ge  0,
\end{align}
where $(i)$ holds because we have $\sum_{\ba} \vpolicy(\ba|\bs) = \sum_{\ba} \behaviorv(\ba|\bs) = 1$. This implies that $\sum_{\ba} \behaviorv(\ba|\bs) \left[\frac{\vpolicy(\ba|\bs)}{\behaviorv(\ba|\bs)} - 1 \right]=0$.     
Clearly, $D_\text{VCQL}(\bs)$ is a non-negative term and equals zero if and only if $\vpolicy(\ba|\bs) = \behaviorv(\ba|\bs)$.
Consequently, when $\vpolicy$ is inconsistent with $\behaviorv$, each iteration of the value update process inherently introduces some degree of underestimation, that is 
\begin{align}
   \hat{V}^{\tau+1}(\bs) \leq \bellman^\vpolicy \hat{V}^{\tau} (\bs) 
\end{align}
We further compute the fixed point of the recursion in Eq.~\eqref{eqn:value_recursion}, which yields the following approximation of the policy value:
\begin{align}
\label{eq:V-wo-samplingerror}
    \hat{V}^\vpolicy(\bs) = V^\vpolicy(\bs) - \alpha  \underbrace{(I - \gamma P^\vpolicy)^{-1}}_{\text{non-negative entries}}
    \underbrace{D_{\textit{VCQL}}(\bs)}_{\text{
    Eq.~\eqref{eq:DVCQL} 
    }\geq 0}.
\end{align}
The aforementioned conclusion underscores that, given $\alpha\ge 0$, Eq.~\eqref{eq:vcql} provides a rigorous lower bound on the policy value in the absence of sampling error.

Next, we extend the conclusion to the setting with sampling error, that is, without assuming that $\hat{\bellman}^\vpolicy = \bellman^\vpolicy$.
Based on Eq.~\eqref{eq:assu}, \eqref{eq:V-wo-samplingerror} and \cite[Theorem 3.1]{kumar2020conservative}, we have
\begin{align}
   \hat{V}^\vpolicy(\bs) \leq V^\vpolicy(\bs) - \alpha \left[\left(I - \gamma P^\vpolicy \right)^{-1} \E_{\vpolicy}\left[\frac{\vpolicy}{\behaviorv} - 1 \right] \right]&(\bs)\nonumber\\
   + \left[ (I - \gamma P^\vpolicy)^{-1} \frac{C_{r, T, \delta} r_{\max}}{(1- \gamma) \sqrt{|\mathcal{D}|}}\right]&(\bs),\nonumber
\end{align}
where $(I - \gamma P^\vpolicy)^{-1}$ is a matrix with all non-negative entries.
The above result holds (w.h.p.) $\forall \bs \in \mathcal{D}$.
In this case, the value of $\alpha$, that prevents overestimation (w.h.p.) is given by:
\begin{align}
\label{eq:alpha}
    \alpha \geq \max_{\bs, \ba \in \mathcal{D}} \frac{C_{r, T, \delta} r_{\max}}{(1 - \gamma) \sqrt{|\mathcal{D}(\bs)|}} \cdot 
    \max_{\bs \in \mathcal{D}} \left[D_{\textit{VCQL}}(\bs)\right]^{-1}.
\end{align}
Therefore, we complete the proof of Lemma \ref{thm:vcql_underestimates}. 
\end{proof}

\begin{remark}
    \blue{Lemma \ref{thm:vcql_underestimates} shows that VCQL, learned by iterating Eq.~\eqref{eq:vcql}, provides a lower bound of the true Q-function in the absence of sampling error. When sampling error exists, choosing a proper \(\alpha\) further reduces overestimation with high probability \(1-\delta\). 
    Thus, the Q-values of all learned policies will not be overestimated, ensuring that the local policy remains \text{conservative and safe}.}

    \blue{VCQL requires both modifications to CQL (the red terms in Eq.~\eqref{eq:vcql}); omitting either leads to unacceptable overestimation for some out-of-distribution actions. 
    This follows directly from the proof of Lemma \ref{thm:vcql_underestimates}, which we omit due to space limitations.}
\end{remark}
\blue{In summary, VCQL not only enables interaction between the global policy and local behavioral policy during local policy evaluation, but also helps identify better policies to guide Q-function updates. 
More importantly, the modified policy evaluation loss still preserves conservative properties, mitigating extrapolation errors caused by OOD error.}

\subsection{Bounding the Return}
To begin, we derive the form of the resulting Q-function iteration. By setting the derivative of Eq.~\eqref{eq:vcql} to zero, we have the following expression for \( \hat{Q}^{\tau+1} \):
\begin{equation}
    \label{eq:derivative}
     \hat{Q}^{\tau+1}(\bs, \ba) = \hat{\bellman}^\vpolicy \hat{Q}^\tau(\bs, \ba) - {\alpha} \left(\frac{\vpolicy(\ba|\bs)}{\behaviorv(\ba|\bs)} - 1 \right),
\end{equation}  
\blue{
where \( \hat{\bellman}^\vpolicy \) denotes the Bellman operator induced by policy \( \vpolicy \).
$\alpha$ is a positive regularization parameter whose value is rigorously constrained by the theoretical bound in Eq.~\eqref{eq:alpha}.}

Since local updates on client $k$ are fully performed on the fixed dataset $\mathcal{D}_k$, it is important to relate the returns of policy $\pi_k$ in the underlying MDP $M$ to those in the empirical MDP $\widetilde{M} \doteq \langle \mathcal{S}, \mathcal{A}, \widetilde{P}, \widetilde{R}, \mu, \gamma \rangle$, induced by $\mathcal{D}_k$.
Based on Eq.~\eqref{eq:derivative}, the Problem \eqref{eq:awr} can be transformed into the following:
\begin{align}
\label{eq:app:awr}
\max_{\pi_k} ~~ J(\widetilde{M},\pi_k) - \frac{{\alpha}{\nu}(\policy_k)}{1-\gamma} - \lambda \mathrm{KL}(\bar{\pi}^*_k,\pi_k),
\end{align}
where  
$\mathrm{KL}(\bar{\pi}^*_k,\pi_k) = \mathbb{E}_{\rvs \sim \mathcal{D}_k} \left[ \mathrm{D_{KL}} \left(\bar\pi_k^*(\cdot  | \rvs) \middle|\middle| \pi_k(\cdot | \rvs)\right) \right]$.
We sightly abuse notation in 
${\nu}(\policy_k) = \sum_{\bs,\ba}\policy_k(\bs,\ba)\left(\frac{\policy_k(\bs,\ba)}{\behaviorv(\bs,\ba)} - 1 \right)$ and denote $\policy_k(\bs,\ba)$ as state-action distribution of the local policy.
Since the learning process is entirely based on the fixed dataset $\mathcal{D}_k$, it is crucial to establish the relationship between the returns of policy $\pi_k$ on the empirical MDP $\widetilde{M}$ and the underlying MDP $M$. Consequently, we derive the following conclusion based on \cite[Lemma A.2]{yu2021combo} and \cite[Lemma 1]{yue2024federated}: 
\begin{lemma}
\label{lemma:JMDP-new}
For any $\pi_k$, the following statement holds (w.h.p.):
\begin{align}
    \label{eq:eq1-proof}
    J(\widetilde{M},\pi_k)\le J(M,\pi_k) +  \tilde{\xi},
\end{align}
where $\tilde{\xi}$ is a positive parameter, defined as
\begin{align}
    \tilde{\xi} \doteq\;&\frac{2\gamma r_{\max} C_{T,\delta}}{(1-\gamma)^2}\mathbb{E}_{\rvs\sim d_{\pi_k}(\rvs)}\bigg[ \sqrt{D_\mathrm{VCQL}(\pi_k,\pi_{\beta_k})(\rvs)|\mathcal{A}|/|\mathcal{D}|} \bigg]\nonumber\\
    & + \frac{C_{R,\delta}}{1-\gamma}\mathbb{E}_{\bfs,\bfa\sim{\pi_k}}\Big[ 1/\sqrt{|\mathcal{D}(\bs,\ba)|}\Big],
\end{align}
with $D_\mathrm{VCQL}(\pi_1,{\pi_2})(\rvs)\doteq 1 + \sum_a \pi_1(\rva | \rvs)(\pi_1(\rva | \rvs)/{\pi_2}(\rva | \rvs)-1)$. 
\end{lemma}
\begin{proof}
Let $M_1$ and $M_2$ be two Markov Decision Processes (MDPs) with the same state space, action space, and discount factor $\gamma$. Given a fraction $f\in(0,1)$, the $f$-interpolant MDP $M_f$ is defined with dynamics $T_{M_f}=fT_{M_1}+(1 - f)T_{M_2}$ and reward function $r_{M_f}=fr_{M_1}+(1 - f)r_{M_2}$.
$T^{\pi}$ denotes the transition matrix on state-action pairs induced by a stationary policy $\pi$, i.e., $T^{\pi}=T(\bfs'|\bfs,\bfa)\pi(\bfa'|\bfs')$.

According to \cite[Lemma A.2]{yu2021combo}, the returns of policy $\pi$ in any MDP $M$, denoted by $J(M,\pi)$, and in $M_f$, denoted by $J(M_1,M_2,f,\pi)$, satisfy the following conclusion:
\begin{align}
\label{eq:JM1M2}
    J(M,\pi)-\xi\leq J(M_1,M_2,f,\pi)\leq J(M,\pi)+\xi,
\end{align}
where
\begin{align*}
\xi=&\frac{2\gamma(1 - f)}{(1-\gamma)^2}r_{\max}D_{\mathrm{TV}}(T_{M_2},T_M)\\
+&\frac{\gamma f}{1-\gamma}\mathbb{E}_{d^{\pi}_{M_1}}[(T^{\pi}_{M}-T^{\pi}_{M_1})Q^{\pi}_{M_1}]\\
+&\frac{f}{1-\gamma}\mathbb{E}_{\bfs,\bfa\sim d^{\pi}_{M_1}}[|r_{M_1}(\bfs,\bfa)-r_M(\bfs,\bfa)|]\\
+&\frac{1 - f}{1-\gamma}\mathbb{E}_{\bfs,\bfa\sim d^{\pi}_{M_2}}[|r_{M_2}(\bfs,\bfa)-r_M(\bfs,\bfa)|].
\end{align*}
The above conclusion is widely used in reinforcement learning \cite{liu2023micro,lin2022modelbased,yue2024federated}.
The conclusion describes the relationship between policy returns in different MDPs based on the corresponding reward differences and dynamics differences.
Furthermore, the conclusion provides an upper and lower bound on the difference between the returns on the empirical MDP and the true MDP. 
We consider the case where we set \(M_1 = M_2=\widetilde{M}\) and \(f = 1\). This leads to the inequality:
\begin{align}
        J(M,\pi) - \xi \le J(\widetilde{M},\pi) \le J(M,\pi) + \xi,
    \end{align}
    where $\xi$ is denoted as
    \begin{align*}
        \xi\doteq\;&\frac{\gamma }{1-\gamma}\Big|\mathbb{E}_{\bfs,\bfa\sim\pi,\bfa'\sim\pi(\cdot|\bfs')}\Big[ \sum_{s'}\big(\hat{T}(\bfs'|\bfs,\bfa)-T(\bfs'|\bfs,\bfa)\big)\nonumber\\
        &\cdot \hat Q^\pi(\bfs',\bfa')\Big]\Big| + \frac{1}{1-\gamma}\mathbb{E}_{\bfs,\bfa\sim\pi}\Big[\big| r - r(\bs,\ba) \big|\Big].
\end{align*}  
According to Assumption \ref{assu:cql}, we have the following bounds:
\begin{align}
    \label{eq:assu-reward}
        &|r - r(\bfs, \bfa)| \leq 
        C^{\mathrm{vcql}}_{r, \delta} /\sqrt{|\mathcal{D}(\bfs, \bfa)|}, \\
    \label{eq:assu-transition}
        &|\hat{\transitions}(\bfs'|\bfs, \bfa) - \transitions(\bfs'|\bfs, \bfa)| \leq
        C^{\mathrm{vcql}}_{\transitions, \delta}/\sqrt{|\mathcal{D}(\bfs, \bfa)|},\\
    \label{eq:assu-q}
        &|\hat Q^{\pi_k}(\bs,\ba)| \leq \frac{r_{\max}}{1-\gamma},~\forall \rvs, \rva.   
\end{align}
Substituting these inequality into the expression for \(\xi\) and letting \(\pi=\pi_k\), we can obtain an upper bound \(\tilde{\xi}\) such that (w.h.p.):
\begin{align}
    \label{eq:eq1}
    J(\widetilde{M},\pi_k)\le J(M,\pi_k) +  \tilde{\xi},
\end{align}
where $\tilde{\xi}$ is defined as
\begin{align}
    \tilde{\xi} \doteq\;&\frac{2\gamma r_{\max} C_{T,\delta}}{(1-\gamma)^2}\mathbb{E}_{\rvs\sim d_{\pi_k}(\rvs)}\bigg[ \sqrt{D_\mathrm{VCQL}(\pi_k,\pi_{\beta_k})(\rvs)|\mathcal{A}|/|\mathcal{D}|} \bigg]\nonumber\\
    & + \frac{C_{R,\delta}}{1-\gamma}\mathbb{E}_{\bfs,\bfa\sim{\pi_k}}\Big[ 1/\sqrt{|\mathcal{D}(\bs,\ba)|}\Big],
\end{align}
with $D_\mathrm{VCQL}(\pi_1,{\pi_2})(\rvs)\doteq 1 + \sum_a \pi_1(\rva | \rvs)(\pi_1(\rva | \rvs)/{\pi_2}(\rva | \rvs)-1)$. 
Thus, we complete the proof of Lemma \ref{lemma:JMDP-new}.
\end{proof}

\subsection{Policy Improvement on Server Aggregation}
Next, we show the 
global policy achieves improvement over behavior policy.
Adapted from Lemma \ref{lemma:JMDP-new}, we establish a lower bound for the return of $\bar{\pi}$ in the $\widetilde{M}$ relative to its return in the underlying MDP (w.h.p.):
\begin{align}
    J(\widetilde{M},\bar{\pi}^*_k)\ge J(M,\bar{\pi}^*_k) - \bar{\xi},
\end{align}
where $\bar{\xi}$ is a positive parameter, defined as
\begin{align}
    \bar{\xi} \doteq\;&\frac{2\gamma r_{\max} C_{T,\delta}}{(1-\gamma)^2}\mathbb{E}_{s\sim d_{\bar{\pi}^*_k}}\bigg[ \sqrt{D_\mathrm{VCQL}(\bar{\pi}^*_k,\pi_{\beta_k})(\rvs)|\mathcal{A}|/|\mathcal{D}(\rvs)|} \bigg]\nonumber\\
    & + \frac{C_{r,\delta}}{1-\gamma}\mathbb{E}_{\bfs,\bfa\sim\bar{\pi}^*_k}\Big[ 1/\sqrt{|\mathcal{D}(\bs,\ba)|}\Big].
\end{align}
Similarly, the following holds with probability greater than $1-\delta$:
\begin{equation}
    J(\widetilde{M},\pi_{\beta_k})\ge J(M,\pi_{\beta_k}) -  \xi_b,
\end{equation}
where 
$\xi_b$ is a positive parameter, defined as
\begin{align}
    \xi_b &\doteq\frac{2\gamma r_{\max} C_{T,\delta}}{(1-\gamma)^2}\cdot\mathbb{E}_{s\sim d_{\pi_{\beta_k}}}\Big[\sqrt{|\mathcal{A}|/|\mathcal{D}(\rvs)|} \Big] \nonumber\\
    &+ \frac{C_{R,\delta}}{1-\gamma}\cdot\mathbb{E}_{\bfs,\bfa\sim{\pi_{\beta_k}}}\left[ 1/\sqrt{|\mathcal{D}(\bs,\ba)|}\right].
\end{align}
As derivation in Section \ref{subsec:awr}, $\bar\pi^*_k$ is the optimal solution to Problem \eqref{eq:P-Advatage-of-Global}. Thus, we can write
\begin{align}
    &J(M,\bar\pi^*_k) +  \bar{\xi} - \frac{\alpha{\nu}(\bar\pi^*_k)}{1-\gamma} - \beta \mathrm{KL}(\bar\pi^*_k,\pi_{\beta_k})\nonumber\\
    \ge& J(\widetilde{M},\bar\pi^*_k) - \frac{\alpha{\nu}(\bar\pi^*_k)}{1-\gamma} - \beta \mathrm{KL}(\bar\pi^*_k,\pi_{\beta_k}) \nonumber\\
    \ge& J(\widetilde{M},\pi_{\beta_k}) - \frac{\alpha{\nu}(\pi_{\beta_k})}{1-\gamma} - \beta \mathrm{KL}(\pi_{\beta_k},\pi_{\beta_k})\nonumber\\
    \ge&J(M,\pi_{\beta_k}) -  \xi_b - \frac{\alpha{\nu}(\pi_{\beta_k})}{1-\gamma} .
\end{align}
Then, the following fact holds:
\begin{align}
\label{eq:strict-global-on-data-step1}
    J(M,\pi) &\ge J(M,\pi_{\beta_k}) - \bar{\xi}  - \xi_b  + \sigma_b + \beta \mathrm{KL}(\bar\pi^*_k,\pi_{\beta_k}),
\end{align}
where $\sigma_b$ is defined as:
\begin{align}
    \sigma_b
    &\doteq \frac{\alpha\big({\nu}(\bar{\pi}^*_k) - {\nu}(\pi_{\beta_k})\big)}{1-\gamma} \nonumber\\
    &\ge -\frac{\alpha\left|{\nu}(\bar{\pi}^*_k) - {\nu}(\pi_{\beta_k})\right|}{1-\gamma} \nonumber\\
    &\mathop{\ge}^{(i)} -\frac{2\alpha}{\delta(1-\gamma)}\mathbb{E}_{\bs\sim d_k(s)}[D_\mathrm{TV}(\bar{\pi}^*_k(\bfa|\bfs),\pi_{\beta_k}(\bfa|\bfs))]\nonumber\\
    &\mathop{\ge}^{(ii)} -\frac{2\alpha}{\delta(1-\gamma)},
\label{eq:strict-global-on-data-delta}
\end{align}
where $(i)$ holds is because
\begin{align} 
&\left|{\nu}(\bar{\pi}^*_k) - {\nu}(\pi_{\beta_k})\right| \nonumber\\
=&\left|\sum_{\bs,\ba} \bar{\pi}^*_k(\bfs,\bfa) \left(\frac{\bar{\pi}^*_k(\bfa|\bfs)}{\pi_{\beta_k}(\bfa,\bfa)} - 1\right)-\pi_{\beta_k}(\bfs,\bfa) \left(\frac{\pi_{\beta_k}(\bfa|\bfs)}{\pi_{\beta_k}(\bfa|\bfs)} - 1\right)\right|\nonumber\\ 
=&\left|\sum_{\bs,\ba} \frac{\bar{\pi}^*_k(\bfs,\bfa)}{\pi_{\beta_k}(\bfa|\bfs)}  \left(\bar{\pi}^*_k(\bfa|\bfs) - \pi_{\beta_k}(\bfa|\bfs)\right)\right|\nonumber\\ 
\mathop{\leq}^{(a)}&\frac{1}{\delta} \sum_{\bs,\ba} d_k(\bfs) \left|\bar{\pi}^*_k(\bfa|\bfs)-\pi_{\beta_k}(\bfa|\bfs)\right|^2
\nonumber\\
\leq&\frac{2}{\delta} \sum_{\bs,\ba} d_k(\bfs) \left|\bar{\pi}^*_k(\bfa|\bfs)-\pi_{\beta_k}(\bfa|\bfs)\right|
, 
\end{align}
where the inequality (a) holds due to the fact $\pi_{\beta_k}(\bfs,\bfa)=d_k(\bfs)\pi_{\beta_k}(\bfa|\bfs)\approx d_k(\bfs,\bfa)$ and the cardinality $d_k(\bfs,\bfa)=|D_k(\bfs,\bfa)|/|D_k|\ge\delta$ (see Section \ref{subsec:vp}), adapted from \cite[Proof of Lemma 1]{yue2024federated}.
Additionally, the inequality $(ii)$ is held because of the following:
$$
D_{\mathrm{TV}}(\bar\pi_k^*(\cdot\mid s),\,\pi_{\beta_k}(\cdot\mid s))
=\frac12\sum_a\big|\bar\pi_k^*(a\mid s)-\pi_{\beta_k}(a\mid s)\big|.
$$
By the triangle inequality $|x-y|\le |x|+|y|$ (and probabilities are nonnegative), for every state $s$,
$$
D_{\mathrm{TV}}(\bar\pi_k^*(\cdot\mid s),\,\pi_{\beta_k}(\cdot\mid s))
\le \frac12\sum_a\big(\bar\pi_k^*(a\mid s)+\pi_{\beta_k}(a\mid s)\big)
=1,
$$
since $\sum_a \bar\pi_k^*(a\mid s)=\sum_a \pi_{\beta_k}(a\mid s)=1$.
Taking expectation over $s\sim d_k(s)$ gives
$$
\mathbb{E}_{s\sim d_k}\!\left[D_{\mathrm{TV}}(\bar\pi_k^*(\cdot\mid s),\,\pi_{\beta_k}(\cdot\mid s))\right]
\le \mathbb{E}_{s\sim d_k}[1]=1.
$$
Thus (ii) holds.
(Equivalently: $D_{\mathrm{TV}}(\cdot,\cdot)\in[0,1]$ pointwise in $s$, so its $d_k$-weighted average is also $\le 1$.)

\begin{theorem}
    \label{thm:policy-improvement-global-app}
    If Assumption \ref{assu:cql} holds, 
    we have the policy $\bar{\pi}^*_k$ of Problem (\ref{eq:lagrange-barpik}) satisfies with high probability:
    \begin{align}
    \label{eq:strict-global-on-data-step2}
        &J(M,\bar{\pi}^*_k) - J(M,\pi_{\beta_k}) \\
        \ge&  \beta \mathrm{KL}(\bar{\pi}^*_k,\pi_{\beta_k})-\frac{2\alpha}{\delta(1-\gamma)} - {\xi_b} - \bar{\xi}\nonumber ,
    \end{align}
    where ${\xi}_b,\bar{\xi}$, and $\frac{2\alpha}{\delta(1-\gamma)}$ are independent of $\beta$.
\end{theorem}  
\begin{proof}
Combining Eq.~\eqref{eq:strict-global-on-data-step1} and ~\eqref{eq:strict-global-on-data-delta}, 
the result can be easily obtained.
\end{proof}
According to the fact that ${\xi}_b,\bar{\xi}$, and $\frac{2\alpha}{\delta(1-\gamma)}$ are independent of $\beta$ in Theorem \ref{thm:policy-improvement-global-app}, the proper value of $\beta$ results in the policy improvement over local behavior policy $\pi_{\beta_k}$, which is:
\begin{align}
\label{eq:strict-global-on-data}
    J(M,\bar{\pi}^*_k)>J(M,{\pi}_{\beta_k}).
\end{align}

\begin{remark}
    Theorem \ref{thm:policy-improvement-global-app} demonstrates that under Assumption \ref{assu:cql}, the policy $\bar{\pi}^*_k$ strictly improves upon the behavioral policy $\policy_{\beta_k}$ based on the local offline dataset. Moreover, the global policy $\bar\pi$ aligns with the policy $\bar{\pi}^*_k$ in a fully asynchronous federated learning environment, ensuring that $\bar\pi$ also achieves strict improvement over the local behavioral policy. These results suggest that FOVA effectively extracts valuable policies from distributed offline datasets, enabling the creation of a superior global policy and reinforcing its capability to enhance overall policy performance.
\end{remark}

\subsection{Policy Improvement on Local Updates}\label{subsec:pi-local}
Subsequently, we theoretically establish the strict policy improvement guarantee for local policy $\pi_k$.
Similarly, we establish a lower bound for the return of $\bar{\pi}$ in the $\widetilde{M}$ relative to its return in the underlying MDP $M$ with high probability:
\begin{align}
    \label{eq:eq2}
    J(\widetilde{M},\bar{\pi}^*_k)\ge J(M,\bar{\pi}^*_k) - \bar{\xi},
\end{align}
where $\bar{\xi}$ is
\begin{align}
    \bar{\xi} \doteq\;&\frac{2\gamma r_{\max} C_{T,\delta}}{(1-\gamma)^2}\mathbb{E}_{s\sim d_{\bar{\pi}^*_k}}\bigg[ \sqrt{D_\mathrm{VCQL}(\bar{\pi}^*_k,\pi_{\beta_k})(\rvs)|\mathcal{A}|/|\mathcal{D}(\rvs)|} \bigg]\nonumber\\
    & + \frac{C_{R,\delta}}{1-\gamma}\mathbb{E}_{\bfs,\bfa\sim\bar{\pi}^*_k}\Big[ 1/\sqrt{|\mathcal{D}(\bs,\ba)|}\Big].
\end{align}
By combining Eq.~\eqref{eq:eq1} and \eqref{eq:eq2}, and considering that $\pi_k$ is the optimal solution to Problem \eqref{eq:awr}, we obtain the following:
\begin{align}
    &J(M,\pi_k) +  \tilde{\xi} - \frac{\alpha{\nu}(\pi_k)}{1-\gamma} - \lambda \mathrm{KL}(\bar{\pi}^*_k,\pi_k)\nonumber\\
    \ge& J(\widetilde{M},\pi_k) - \frac{\alpha{\nu}(\pi_k)}{1-\gamma} - \lambda \mathrm{KL}(\bar{\pi}^*_k,\pi_k)\nonumber\\
    \ge& J(\widetilde{M},\bar{\pi}^*_k) - \frac{\alpha{\nu}(\bar{\pi}^*_k)}{1-\gamma} - \lambda \mathrm{KL}(\bar{\pi}^*_k,\bar{\pi}^*_k)\tag{from optimality of $\pi_k$}\\
    \ge&J(M,\bar{\pi}^*_k) - \bar{\xi} - \frac{\alpha{\nu}(\bar{\pi}^*_k)}{1-\gamma}.
    \label{eq:eq9}
\end{align}
This provides a lower bound for $J(M,\pi)$ based on $J(M,\bar{\pi}^*_k)$:
\begin{align}
    J(M,\pi_k) &\ge J(M,\bar{\pi}^*_k) - \bar{\xi} - \tilde{\xi}  + \bar{\sigma} + \lambda \mathrm{KL}(\bar{\pi}^*_k,\pi_k),
    \label{eq:eq8}
\end{align}
where $\bar{\sigma}$ is
\begin{align}
    \label{eq:sigma_bar}
    \bar{\sigma}&\doteq \frac{\alpha\big({\nu}(\pi_k) - {\nu}(\bar{\pi}^*_k)\big)}{1-\gamma}\nonumber\\
    &\mathop{\ge}^{(j)} -\frac{4\alpha}{\delta(1-\gamma)}\mathbb{E}_{\bs\sim d_k(\bs)} [D_{TV}(\bar{\pi}^*_k(\bfa|\bfs),{\pi}_k(\bfa|\bfs))]\nonumber\\
    &\mathop{\ge}^{(jj)} -\frac{4\alpha}{\delta(1-\gamma)} 
\end{align}
where $(j)$ is similar with Eq.~\eqref{eq:strict-global-on-data-delta}
\begin{align} 
&\left|{\nu}(\bar{\pi}^*_k) - {\nu}(\pi_{k})\right| \nonumber\\
=&\left|\sum_{\bs,\ba} \bar{\pi}^*_k(\bfs,\bfa) \left(\frac{\bar{\pi}^*_k(\bfa|\bfs)}{\pi_{\beta_k}(\bfa|\bfs)} - 1\right)-\pi_{k}(\bfs,\bfa) \left(\frac{\pi_{k}(\bfs,\bfa)}{\pi_{\beta_k}(\bfs,\bfa)} - 1\right)\right|\nonumber\\ 
\leq&\left|\sum_{\bs,\ba} \bar{\pi}^*_k(\bfs,\bfa) \left(\frac{\bar{\pi}^*_k(\bfa|\bfs)}{\pi_{\beta_k}(\bfa|\bfs)} - 1\right)-\bar{\pi}^*_k(\bfs,\bfa) \left(\frac{{\pi}_k(\bfa|\bfs)}{\pi_{\beta_k}(\bfa|\bfs)} - 1\right)\right|\nonumber\\ 
+&\left|\sum_{\bs,\ba} \bar{\pi}^*_k(\bfs,\bfa) \left(\frac{{\pi}_k(\bfa|\bfs)}{\pi_{\beta_k}(\bfa|\bfs)} - 1\right)-{\pi}_k(\bfs,\bfa) \left(\frac{{\pi}_k(\bfa|\bfs)}{\pi_{\beta_k}(\bfa|\bfs)} - 1\right)\right| \nonumber\\
&\text{(triangle inequality)}\nonumber\\ 
=&\left|\sum_{\bs,\ba}\frac{\bar{\pi}^*_k(\bfs,\bfa)}{\pi_{\beta_k}(\bfa|\bfs)} \cdot (\bar{\pi}^*_k(\bfa|\bfs)-{\pi}_k(\bfa|\bfs)  \right|\nonumber\\ 
+&\left|\sum_{\bs,\ba} (\bar{\pi}^*_k(\bfs,\bfa)-{\pi}_k(\bfs,\bfa)) \cdot \left(\frac{{\pi}_k(\bfa|\bfs)}{\pi_{\beta_k}(\bfa|\bfs)} - 1\right)\right| \nonumber\\
&\text{(arranging terms)}\nonumber\\ 
\leq&\frac{2}{\delta} \sum_{\bs,\ba} d_k(\bs)\left|\bar{\pi}^*_k(\bfa|\bfs)-{\pi}_k(\bfa|\bfs)\right|\nonumber\\ 
\leq&\frac{4}{\delta} \mathbb{E}_{\bs\sim d_k(\bs)} [D_{TV}(\bar{\pi}^*_k(\bfa|\bfs),{\pi}_k(\bfa|\bfs))],  
\end{align}
The inequality $(jj)$ is held because, for any fixed $\bs$,
\begin{align}
&D_{\mathrm{TV}}\big(\bar{\pi}^*_k(\cdot\mid\bs),\,\pi_k(\cdot\mid\bs)\big)\nonumber\\
=&\tfrac12\sum_{\bfa}\big|\bar{\pi}^*_k(\bfa\mid\bs)-\pi_k(\bfa\mid\bs)\big|\nonumber\\
\le &\tfrac12\sum_{\bfa}\big(\bar{\pi}^*_k(\bfa\mid\bs)+\pi_k(\bfa\mid\bs)\big)\nonumber
=1.
\end{align}
Taking expectation over $\bs\sim d_k$ gives the claim.
Accordingly, we obtain the following
\begin{theorem}
    \label{thm:policy-improvement-local-app}
    If Assumption \ref{assu:cql} holds,
    we obtain an optimal solution $\pi^*_k$ of Problem (\ref{eq:awr}) satisfies with high probability:
    \begin{align}
        &J(M,\pi_k) - J(M,\pi_{\beta_k}) \\
        \ge &
         \lambda \mathrm{KL}(\bar{\pi}^*_k,\pi_k) + \beta \mathrm{KL}(\bar{\pi}^*_k,\pi_{\beta_k})
        -\frac{6\alpha}{\delta(1-\gamma)} - {\xi_b} - 2\bar{\xi} - \tilde{\xi},\nonumber
    \end{align}
    where $\bar{\xi},\tilde{\xi}$, ${\xi_b}$ and $\frac{6\alpha}{\delta(1-\gamma)}$ are independent of $\lambda$ and $\beta$.
\end{theorem}  
\begin{proof}
Plugging Eq.~\eqref{eq:sigma_bar} into Eq.~\eqref{eq:eq9}, we have
\begin{align}
    J(M,\pi_k) &\ge J(M,\bar{\pi}^*_k) - \bar{\xi} - \tilde{\xi}  - 
    \frac{4\alpha}{\delta(1-\gamma)}
    + \lambda \mathrm{KL}(\bar{\pi}^*_k,\pi_k),
    \label{eq:eq8-further}
\end{align}
Combining Eq.~\eqref{eq:strict-global-on-data-step2} and \eqref{eq:eq8-further}, we derive
\begin{align}
    &J(M,\pi_k) - J(M,\pi_{\beta_k}) \\
    \ge &
     \lambda \mathrm{KL}(\bar{\pi}^*_k,\pi_k) + \beta \mathrm{KL}(\bar{\pi}^*_k,\pi_{\beta_k})
    -\frac{6\alpha}{\delta(1-\gamma)} - {\xi_b} - 2\bar{\xi} - \tilde{\xi},\nonumber
\end{align}
Consequently, we complete the proof of Theorem \ref{thm:policy-improvement-local-app}.
\end{proof}
Clearly, $\bar{\xi},\tilde{\xi}$, ${\xi_b}$ and $\frac{6\alpha}{\delta(1-\gamma)}$ are independent of $\lambda$ and $\beta$. Thus, proper value of $\lambda$  and $\beta$ leads to the performance improvement over policy $\bar{\pi}^*_k$:
\begin{equation}
\label{eq:strict-local-on-data}
    J(M,\pi_k)>J(M,{\pi}_{\beta_k}).
\end{equation}
\begin{remark}
    Theorem \ref{thm:policy-improvement-local-app} shows that under Assumption \ref{assu:cql}, the local policy $\pi_k$ derived from Problem (\ref{eq:awr}) achieves a strict improvement over the behavioral policy $\policy_{\beta_k}$ based on the local offline dataset $D_k$. This indicates that FOVA is capable of training high-quality local policies, highlighting its strong potential to enhance policy performance.
\end{remark}

\section{Extended Experiment for Fig.~\ref{fig:intro}}\label{app:app-challenges}


We additionally conducted two separate sets of experiments using the Gym-MuJoCo dataset to evaluate the issues of mixed-quality and inconsistent optimization objectives in offline FRL, respectively.

\begin{figure}[h]
    \subfigure[Hopper with medium/expert]{\label{fig:moti-heter}\includegraphics[width=0.210\textwidth]{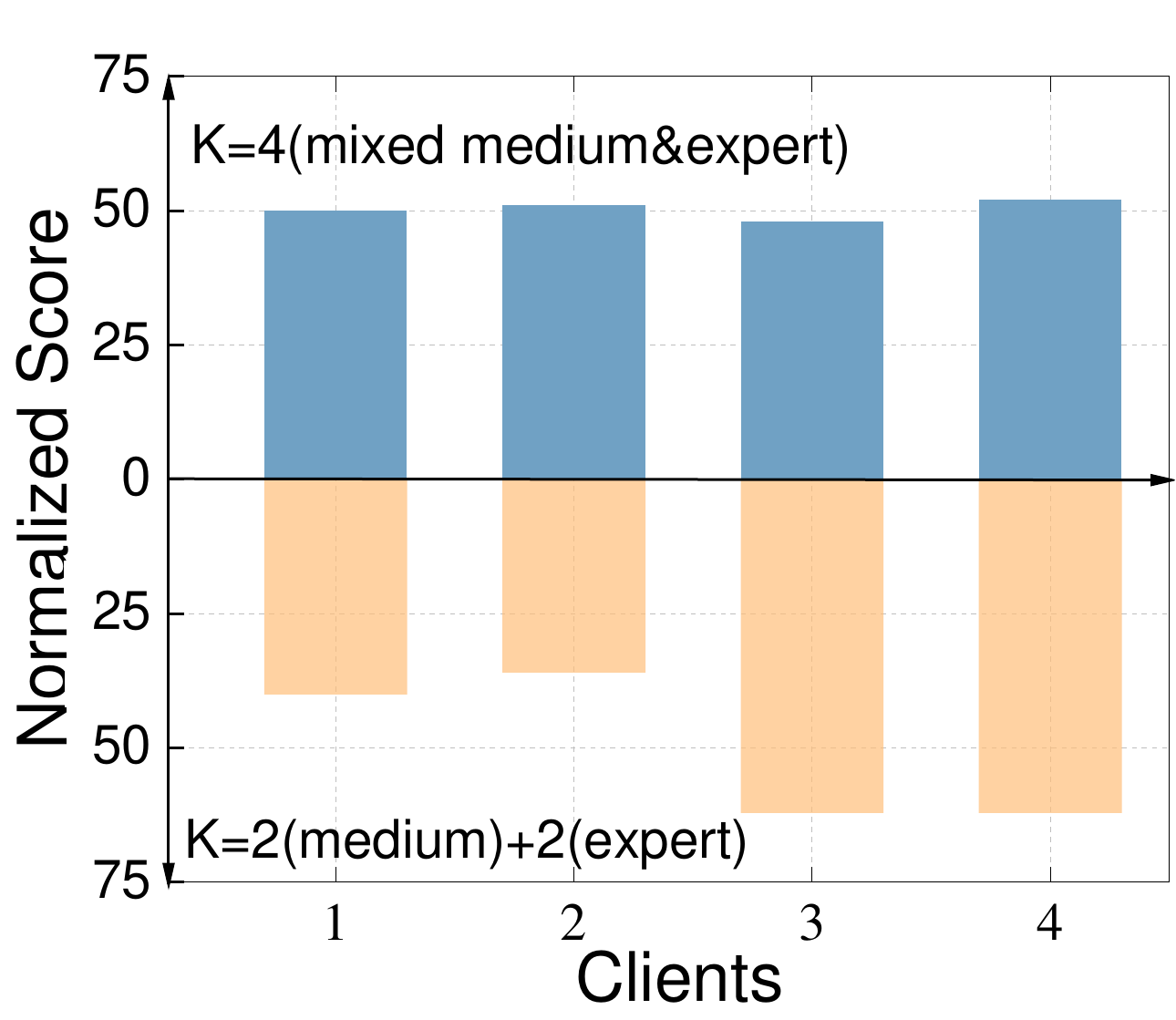}}
    \subfigure[Hopper-medium-replay]{\label{fig:moti-incon}\includegraphics[width=0.238\textwidth]{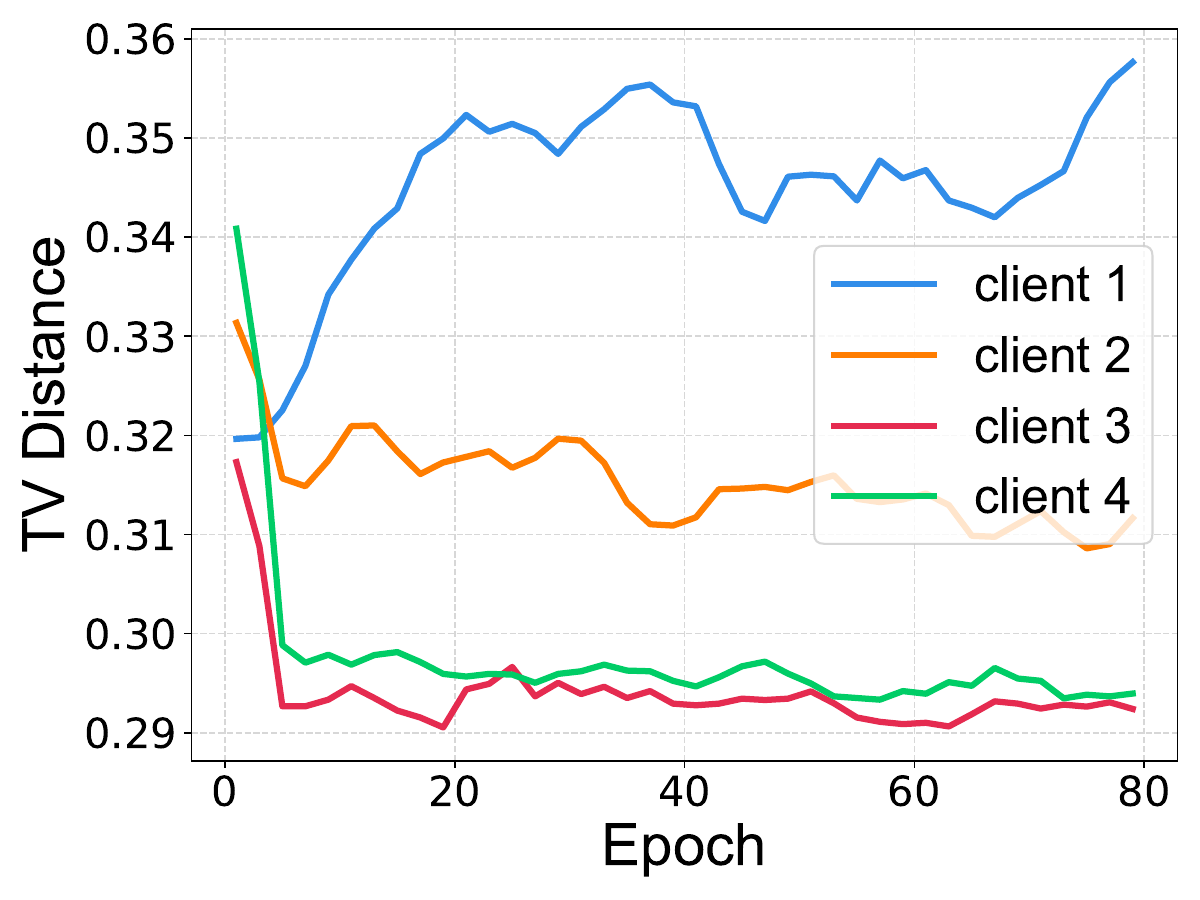}}
    \label{fig:moti}
	\caption{The performance of DRPO in Gym-MuJoCo dataset.
    Left: Four clients with a combined total of 12k Hopper data pairs (6k medium-quality and 6k expert-quality). The blue bar depicts mixed-quality clients (2 medium and 2 expert), and the yellow bar illustrates 4 clients with uniform-quality data pairs. 
    Right: 4 clients, each utilizing 3k Hopper-medium-replay data pairs.
 }
\vspace{-15pt}
\end{figure}

To investigate the impact of mixed-quality data, we addressed this issue in Fig.~\ref{fig:intro-mixed}. With the total dataset remaining constant, we partitioned the data into groups of varying quality. We found that local policies trained on mixed-quality data experienced significant performance degradation, resulting in a decreased and unstable global policy. 
To further explore this problem, we conducted a more extensive analysis based on Fig.~\ref{fig:intro-mixed}, as shown in Fig.~\ref{fig:moti-heter}. We evaluated the final performance of all clients' local policies and discovered that the quality of the local datasets had a substantial impact on the effectiveness of the trained local policies. This decline is primarily due to the fact that local policy training involves a trade-off between the local behavioral policy—some of which may be inherently low-quality—and the global policy. This simple balancing act, without accurately assessing policy quality, led to nearly a 50\% performance drop for clients with low-quality data. 
In summary, existing methods perform poorly in scenarios with mixed-quality data because they lack a mechanism to distinguish harmful data.


Fig.~\ref{fig:intro-incon} illustrates a counterintuitive finding: the global policy can significantly underperform relative to most local policies during the aggregation process. However, existing methods typically include terms in the loss function to minimize \( D(\pi_k, \bar{\pi}) \), thereby constraining the differences between the global and local policies, as shown below:
\begin{align}
\max_{\pi_k}\mathbb{E}_{\bs\sim\mathcal{D}_k,\ba\sim\pi_k(\cdot|\bs)}\big[\hat{Q}^{\pi_k}(\bs,\ba)\big]
    - \alpha_1 D(\pi_k,\policy_{\beta_k}) - \alpha_2 D(\pi_k,\bar{\pi}),
    \nonumber
\end{align}
If effective, this should mitigate the problem of inconsistent optimization objectives. To challenge this notion, we conducted an extensive analysis of TV distances (ranging from 0 to 1) based on the experimental setup in Fig.~\ref{fig:intro-incon}, as shown in Fig.~\ref{fig:moti-incon}. We examined the TV distances between local policies and the global policy from four clients trained on the Hopper-medium-replay dataset using existing methods. The results showed that although regularization reduced the overall TV distance, it remained relatively high (around 0.3), and for one client, the TV distance even slightly increased due to the additional loss terms.
Therefore, the distance constraint regularization in the loss function does not guarantee consistency between local and global policies. As a result, existing methods primarily optimize local policies without effectively optimizing the global policy.

\end{appendices}

\section*{Acknowledge}
We thank Yupeng Yao, Shuning Wang, and Junwei Liao for insightful discussions during the preparation of this work.


\bibliographystyle{IEEEtran}
\bibliography{r}
\newpage

\onecolumn
\section{Extented Experiment of Inconsistent
Optimization Objective 
Problem}\label{app:motivating-inconsistent-extension}
We evaluated the inconsistent optimization objective problems in offline FRL using the Gym locomotion dataset, as shown in Fig.~\ref{fig:app-incon}.
It is evident that the client-side policy consistently exhibits a substantial performance gap when compared to the server-side policy. In most cases, the server-side policy performs worse, which is undesirable, as the global policy of the server-side should ideally be the sole output of the algorithm.

\section{Extented Experiment of Mixed-Quality Problem}\label{app:motivating-mixed-extension}
{
We investigated the mixed-quality issue in offline FRL using the Gym locomotion dataset, including HalfCheetah, Hopper, Walker2d, and Ant.} 
The experiments were conducted with existing algorithms, as illustrated in Fig.~\ref{fig:app-heter}.
As seen in Fig.~\ref{fig:app-heter}, the performance of algorithms under the mixed-quality setting is significantly worse compared to a uniform-quality setting, where the data, though identical, is divided into different quality categories. 
\begin{figure*}[h]
    \centering 
    \vspace{-1pt}
    \subfigure{\label{fig:app-ant-e-incon}\includegraphics[width=0.235\textwidth]{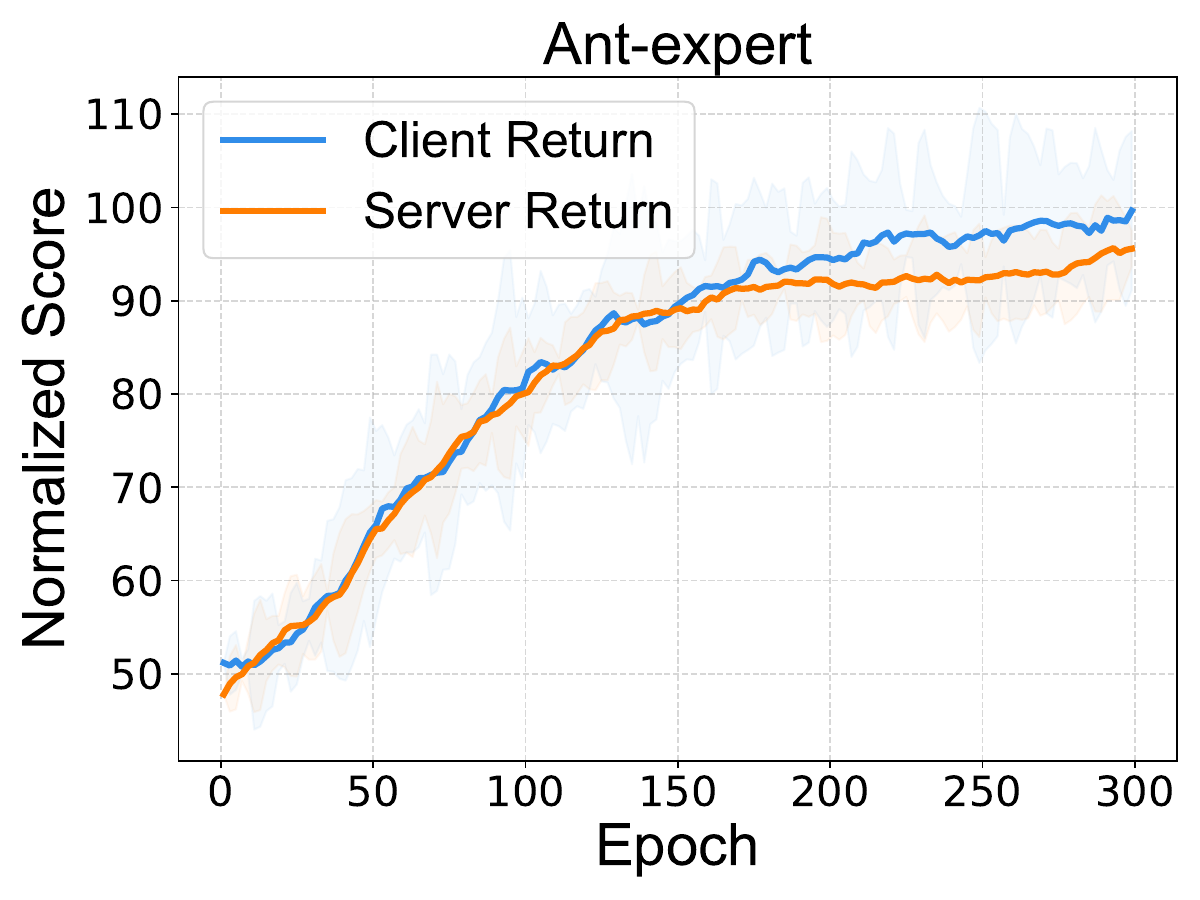}}
    \subfigure{\label{fig:app-ant-me-incon}\includegraphics[width=0.235\textwidth]{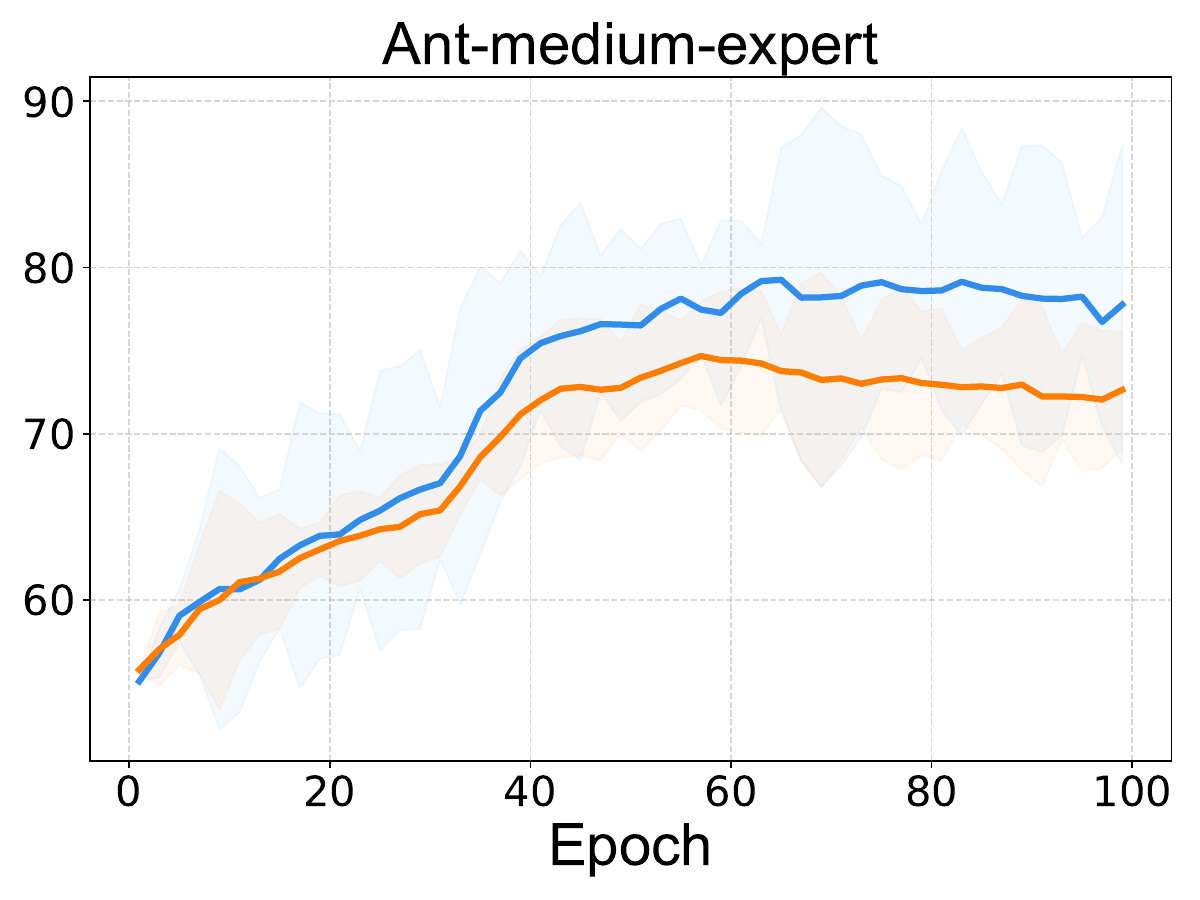}}
    \subfigure{\label{fig:app-ant-m-incon}\includegraphics[width=0.235\textwidth]{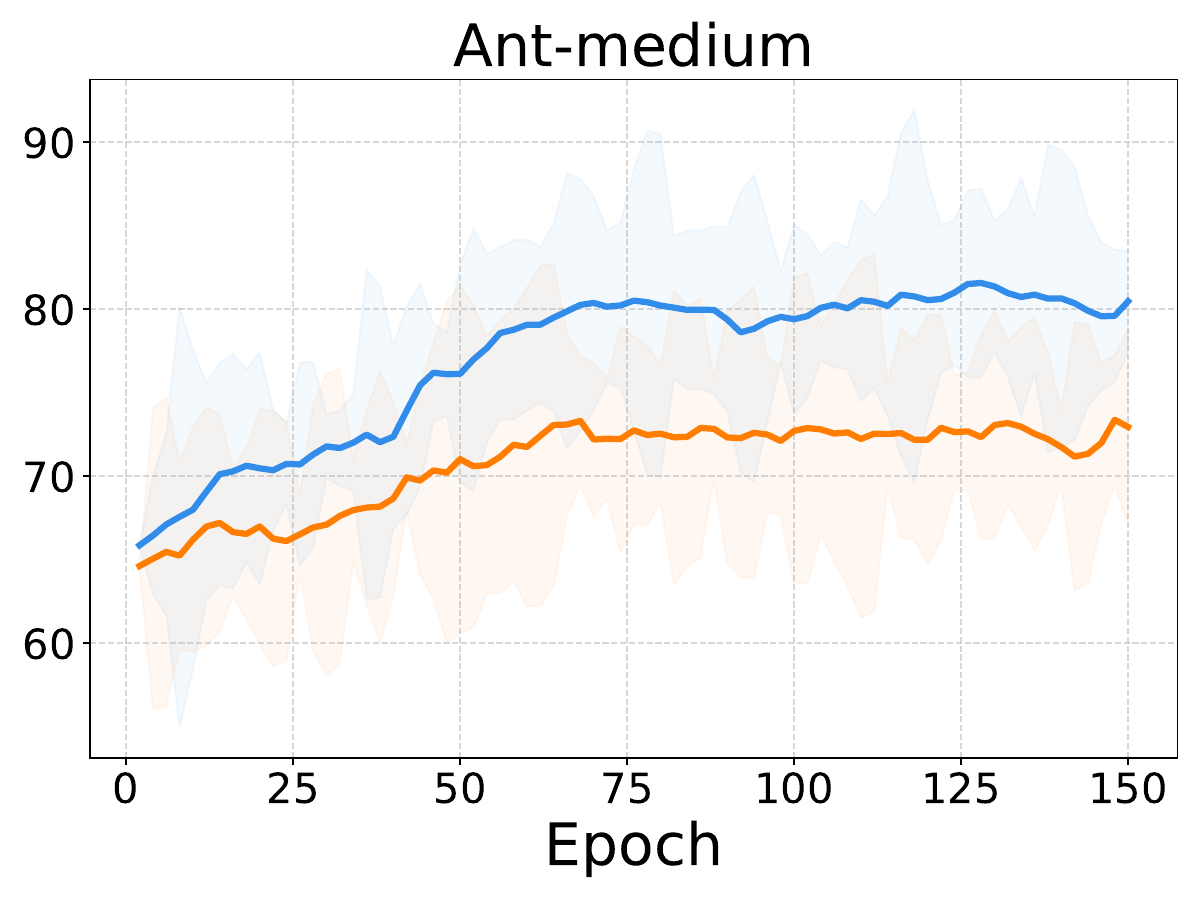}}
    \subfigure{\label{fig:app-ant-mr-incon}\includegraphics[width=0.235\textwidth]{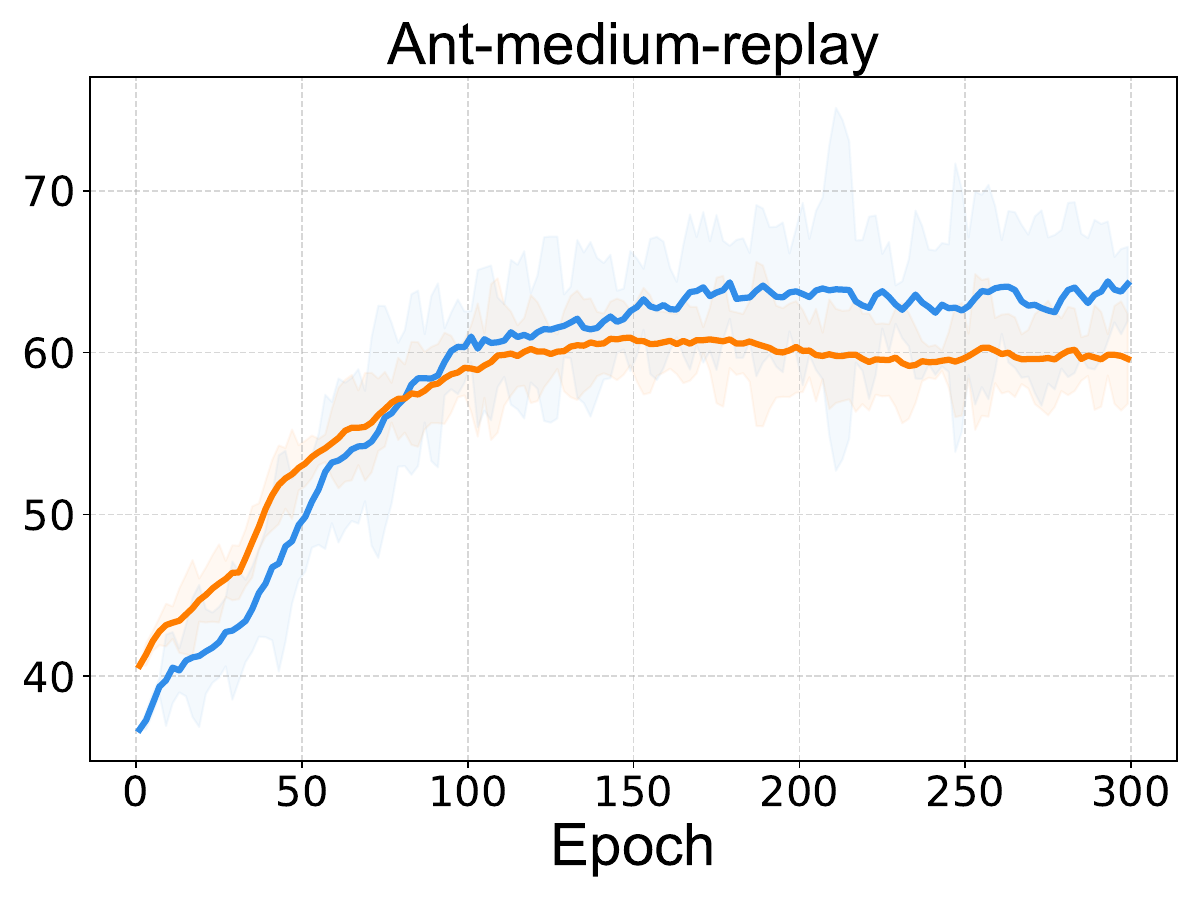}}
    \vspace{-1pt}
    \subfigure{\label{fig:app-halfcheetah-e-incon}\includegraphics[width=0.235\textwidth]{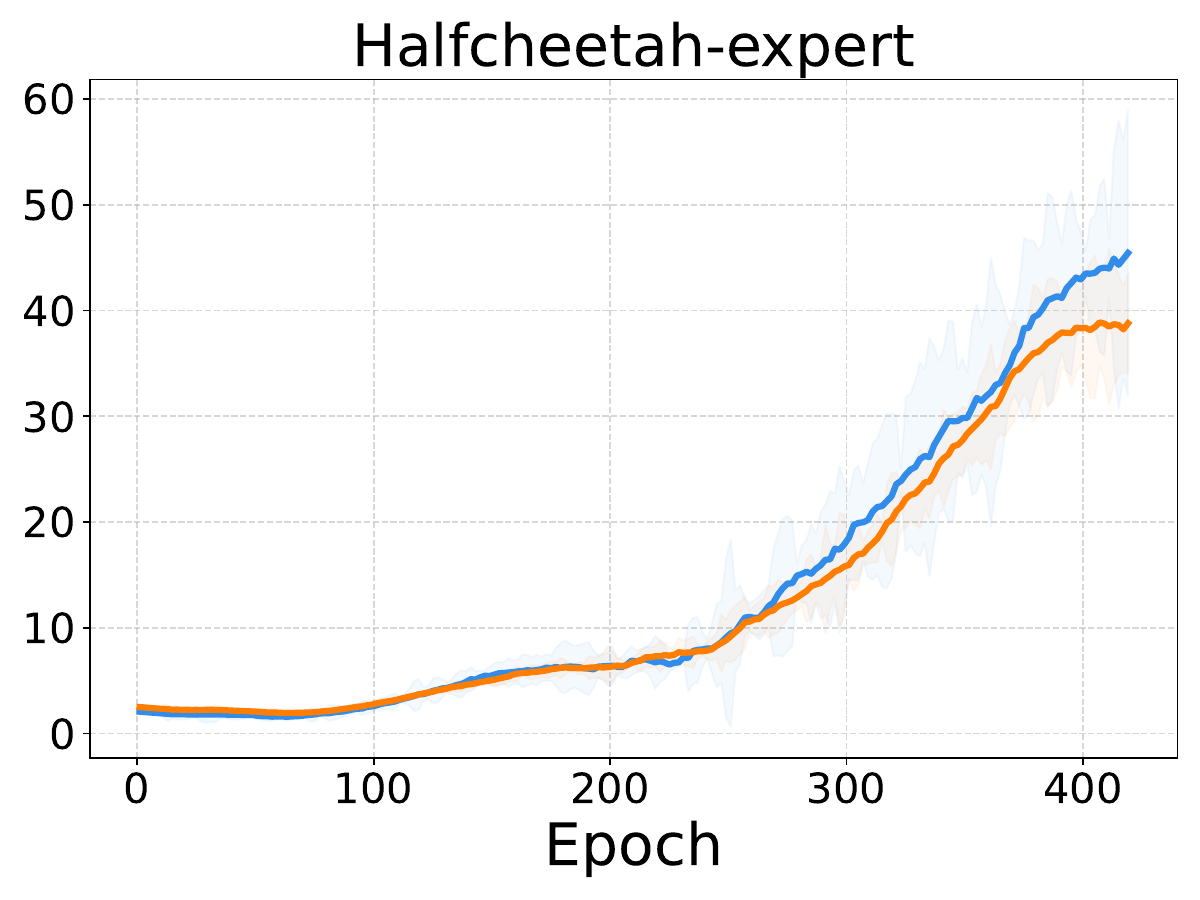}}
    \subfigure{\label{fig:app-halfcheetah-me-incon}\includegraphics[width=0.235\textwidth]{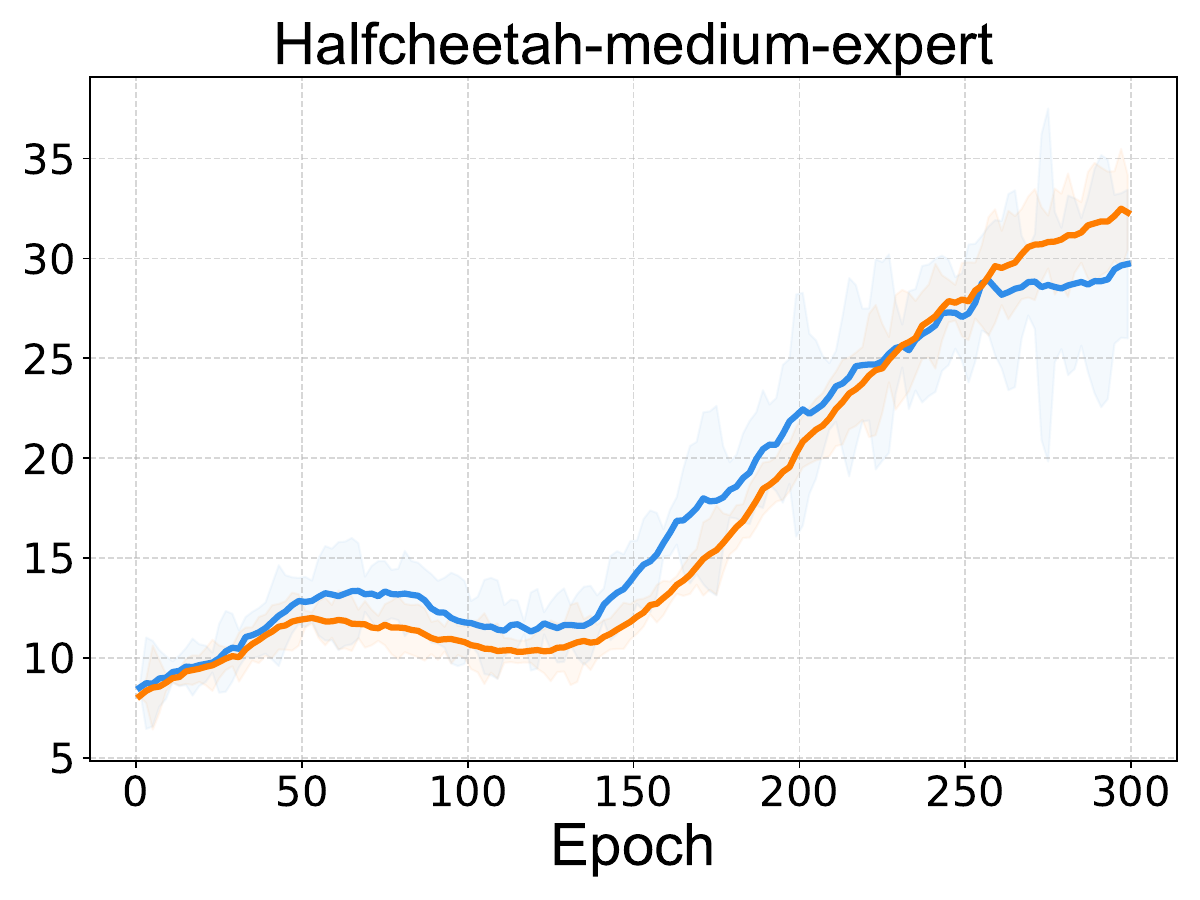}}
    \subfigure{\label{fig:app-halfcheetah-m-incon}\includegraphics[width=0.235\textwidth]{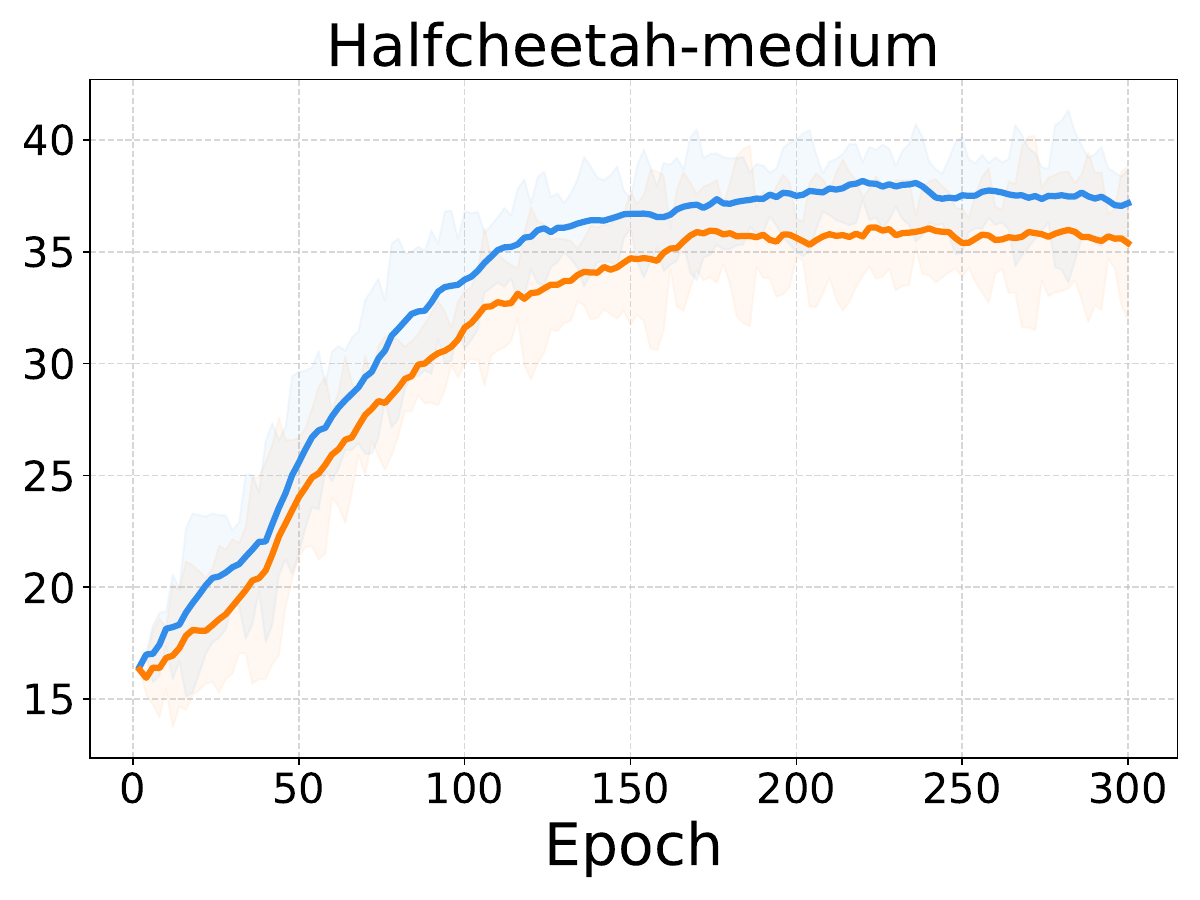}}
    \subfigure{\label{fig:app-halfcheetah-mr-incon}\includegraphics[width=0.235\textwidth]{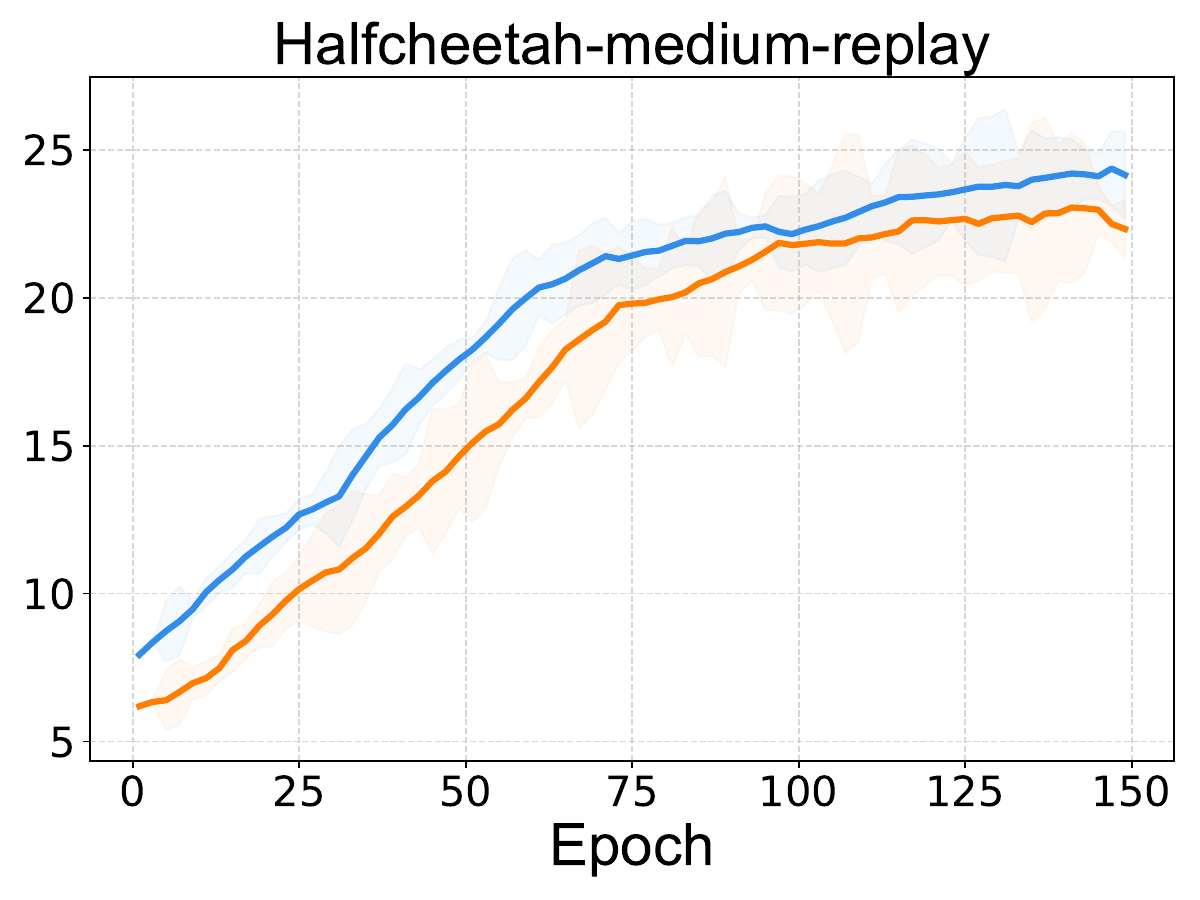}}
    \vspace{-1pt}
    \subfigure{\label{fig:app-hopper-e-incon}\includegraphics[width=0.235\textwidth]{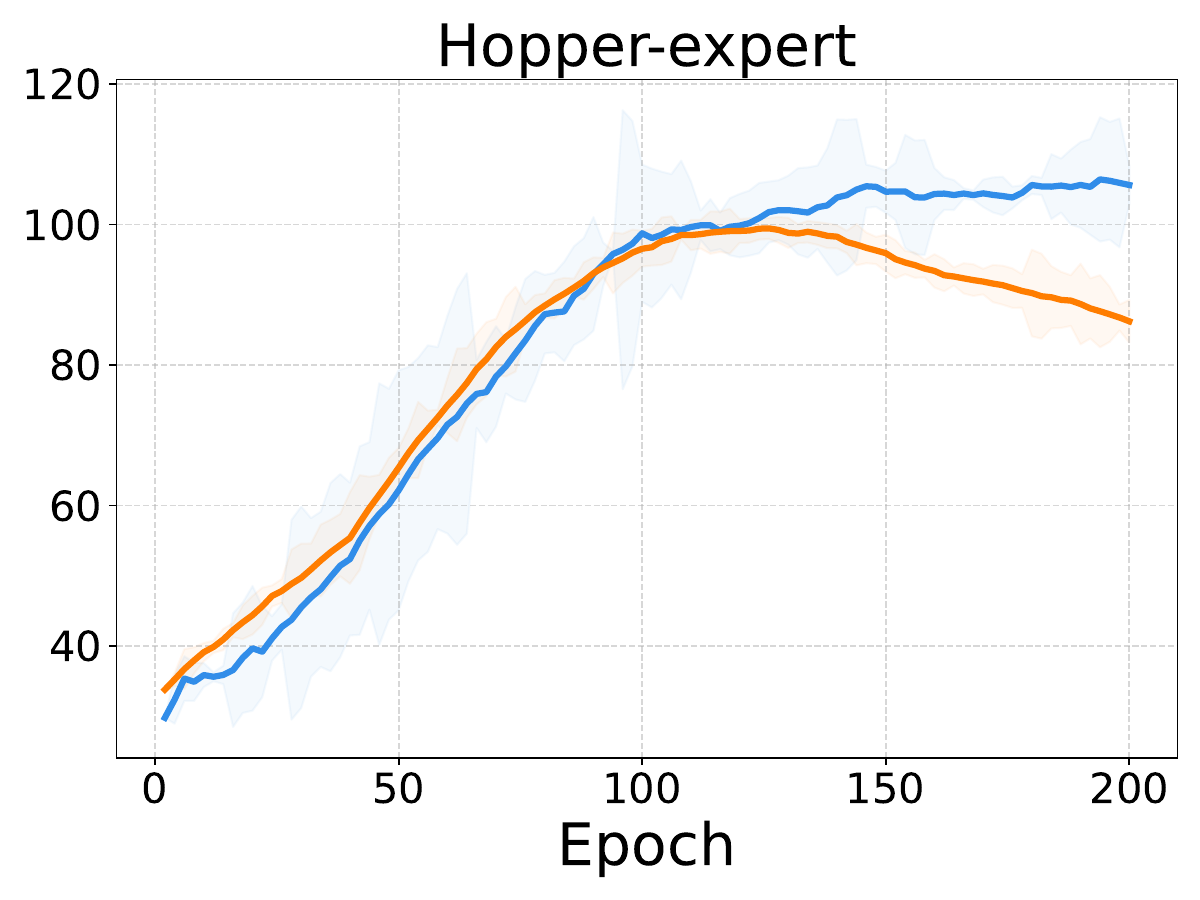}}
    \subfigure{\label{fig:app-hopper-me-incon}\includegraphics[width=0.235\textwidth]{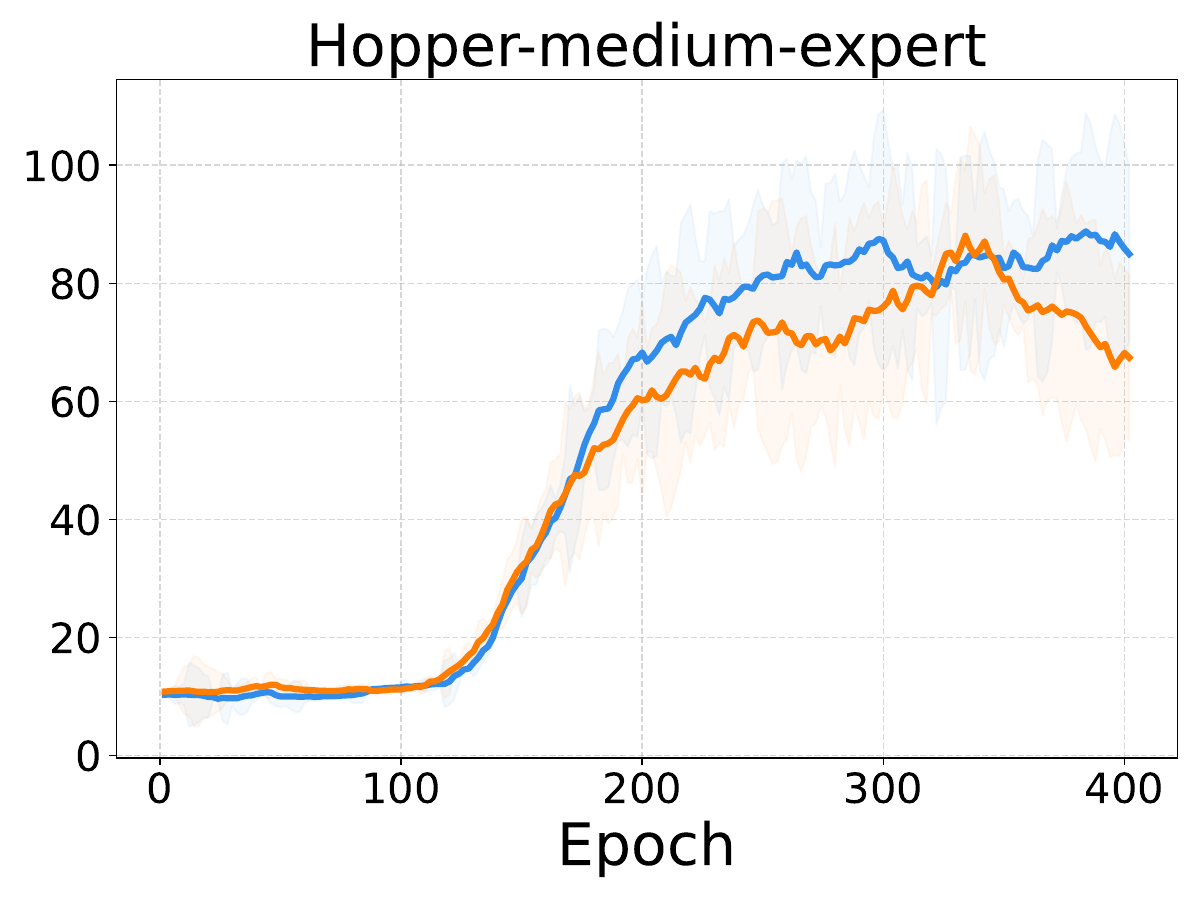}}
    \subfigure{\label{fig:app-hopper-m-incon}\includegraphics[width=0.235\textwidth]{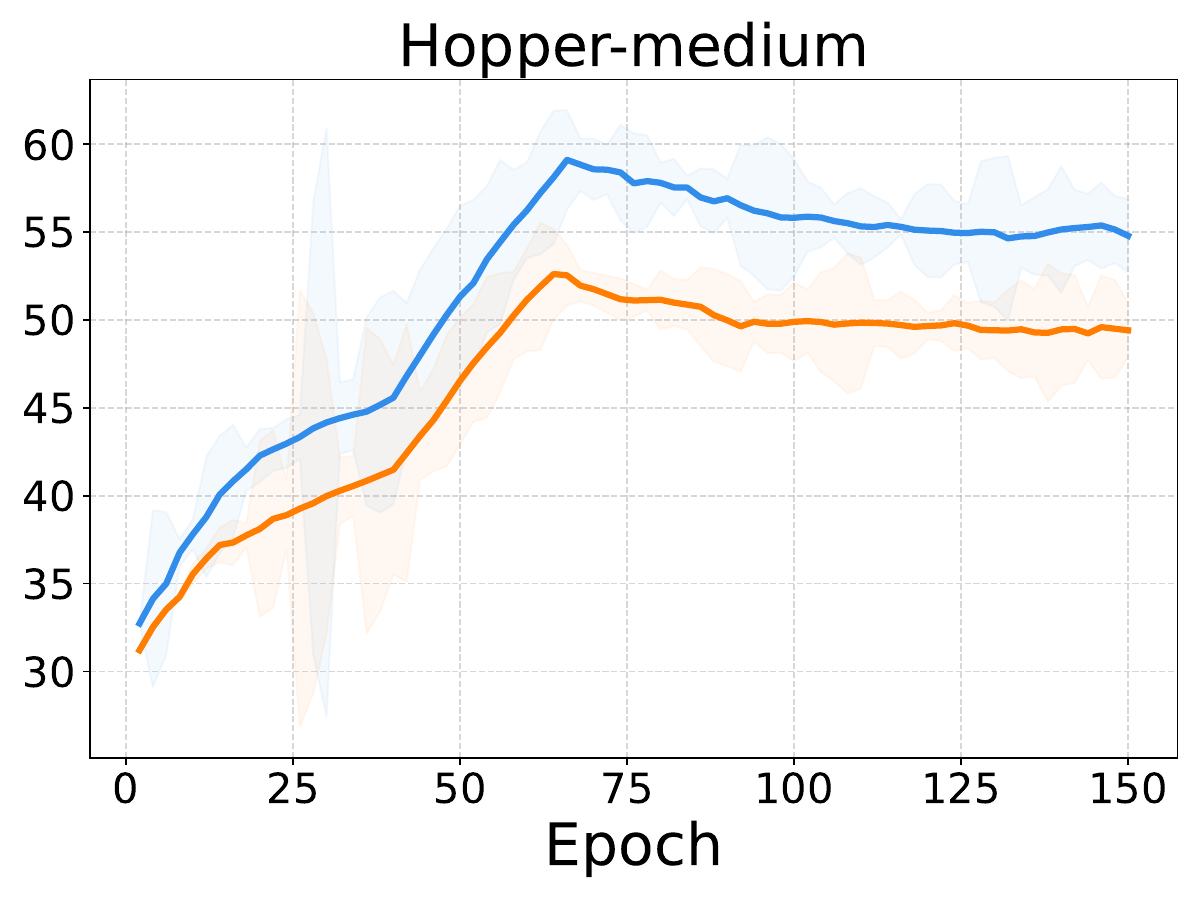}}
    \subfigure{\label{fig:app-hopper-mr-incon}\includegraphics[width=0.235\textwidth]{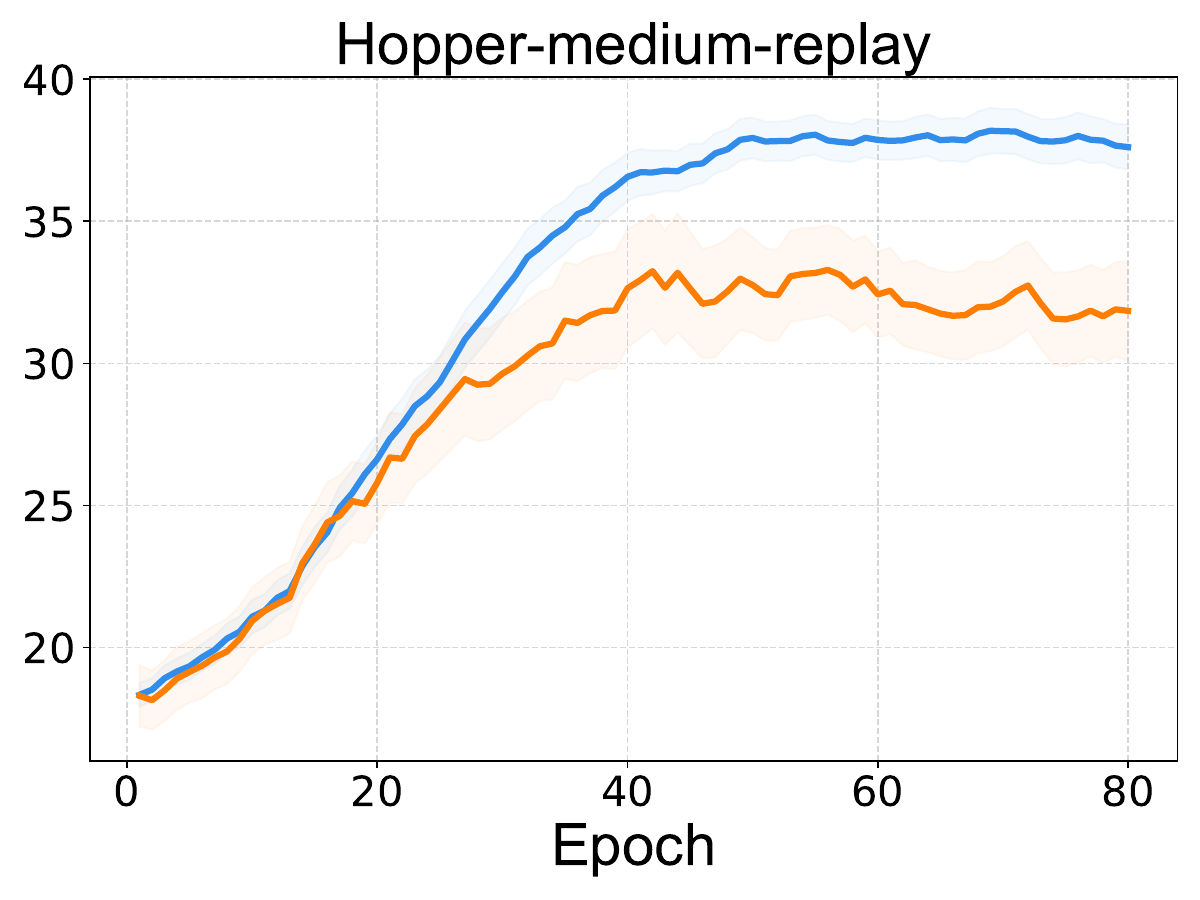}}
    \vspace{-1pt}
    \subfigure{\label{fig:app-walker2d-e-incon}\includegraphics[width=0.235\textwidth]{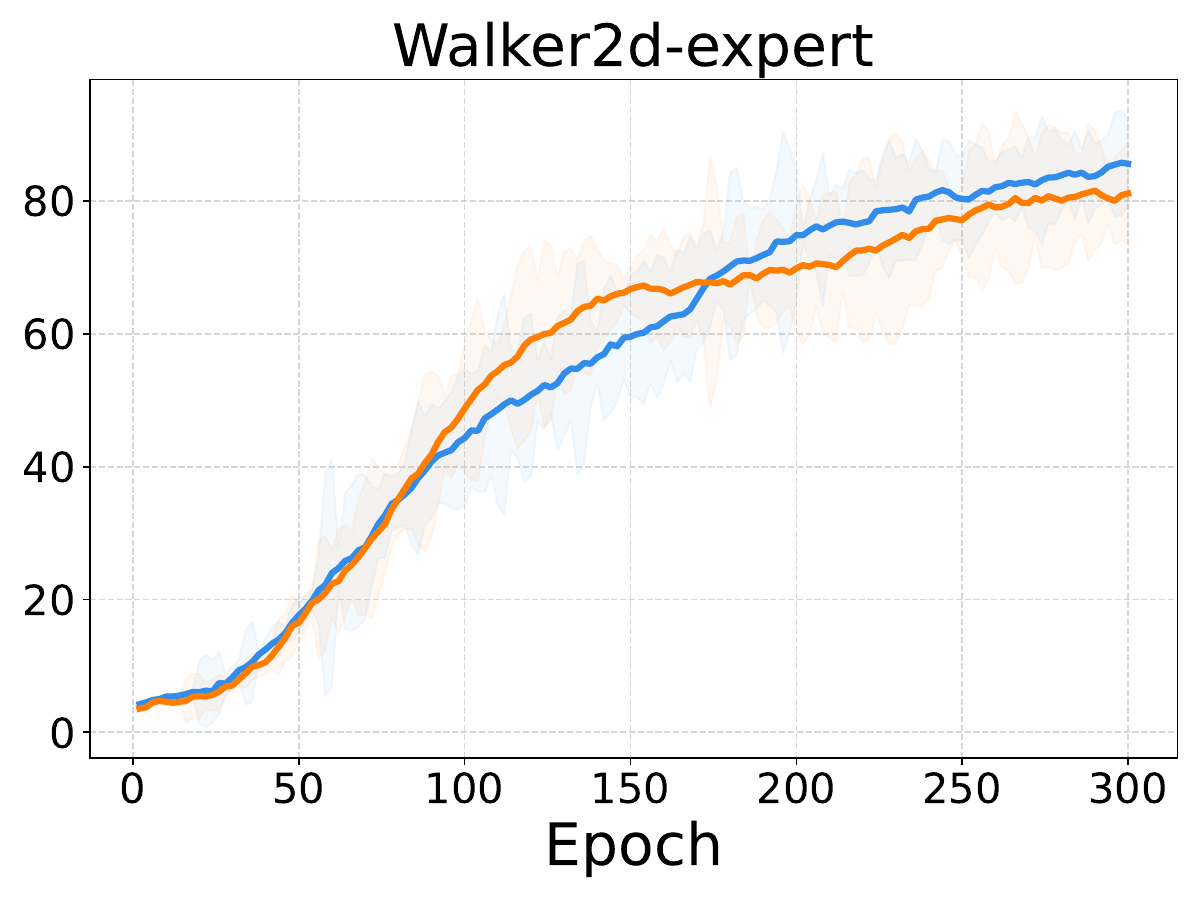}}
    \subfigure{\label{fig:app-walker2d-me-incon}\includegraphics[width=0.235\textwidth]{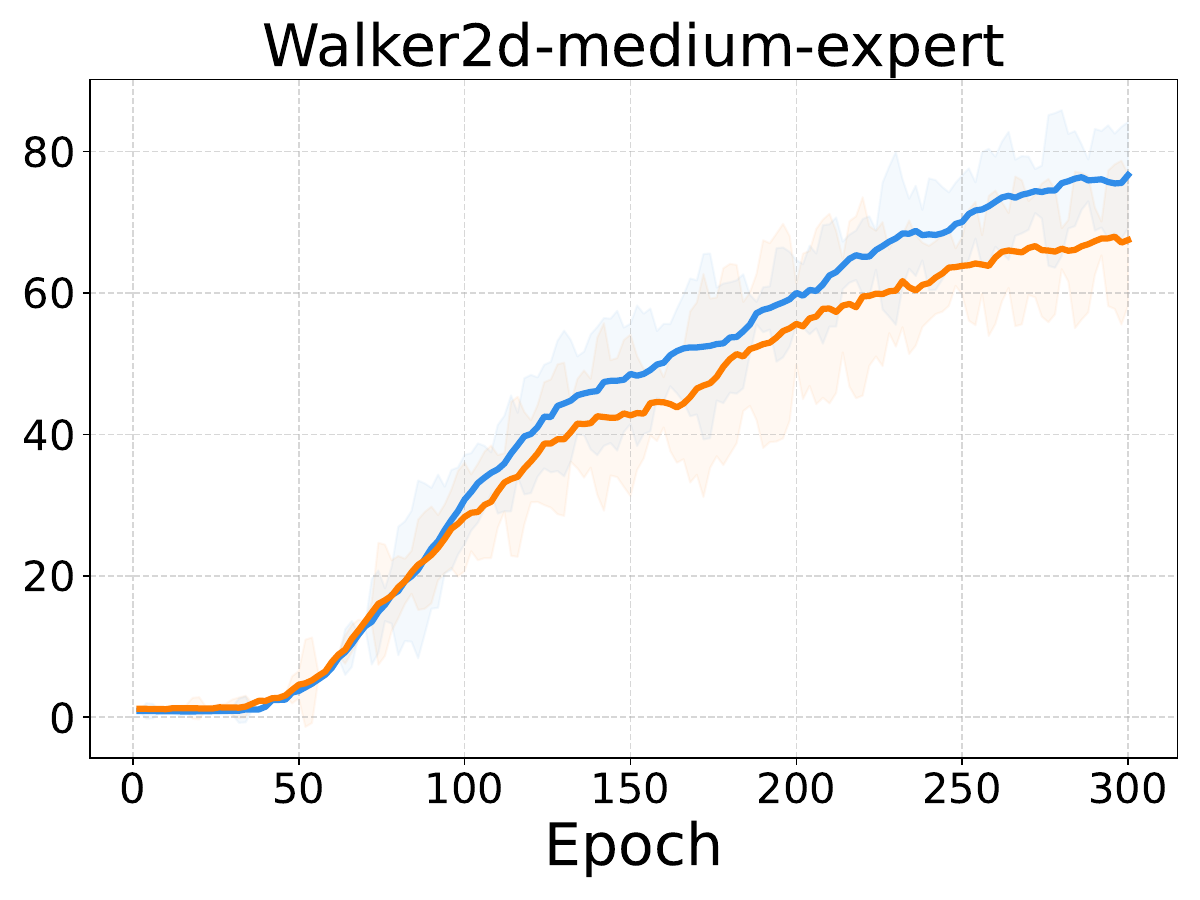}}
    \subfigure{\label{fig:app-walker2d-m-incon}\includegraphics[width=0.235\textwidth]{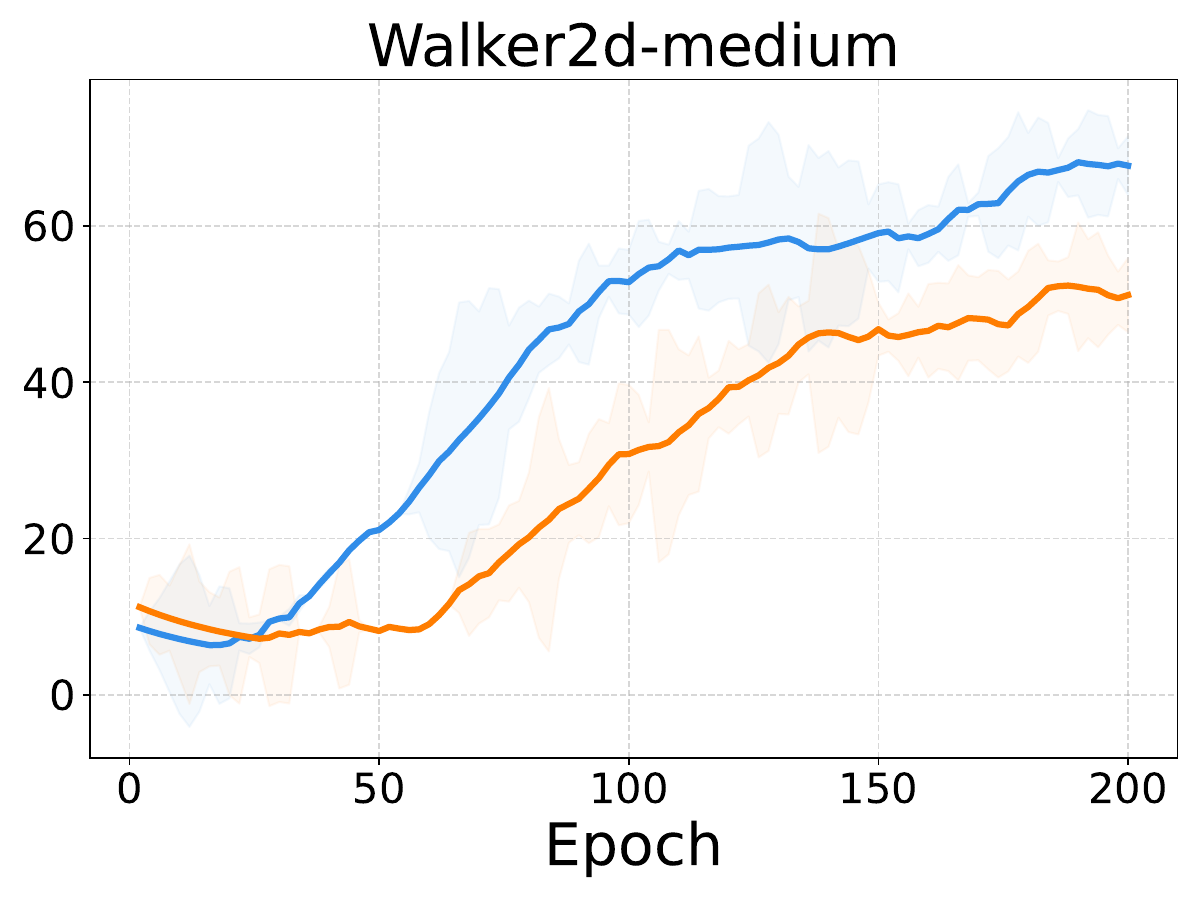}}
    \subfigure{\label{fig:app-walker2d-mr-incon}\includegraphics[width=0.235\textwidth]{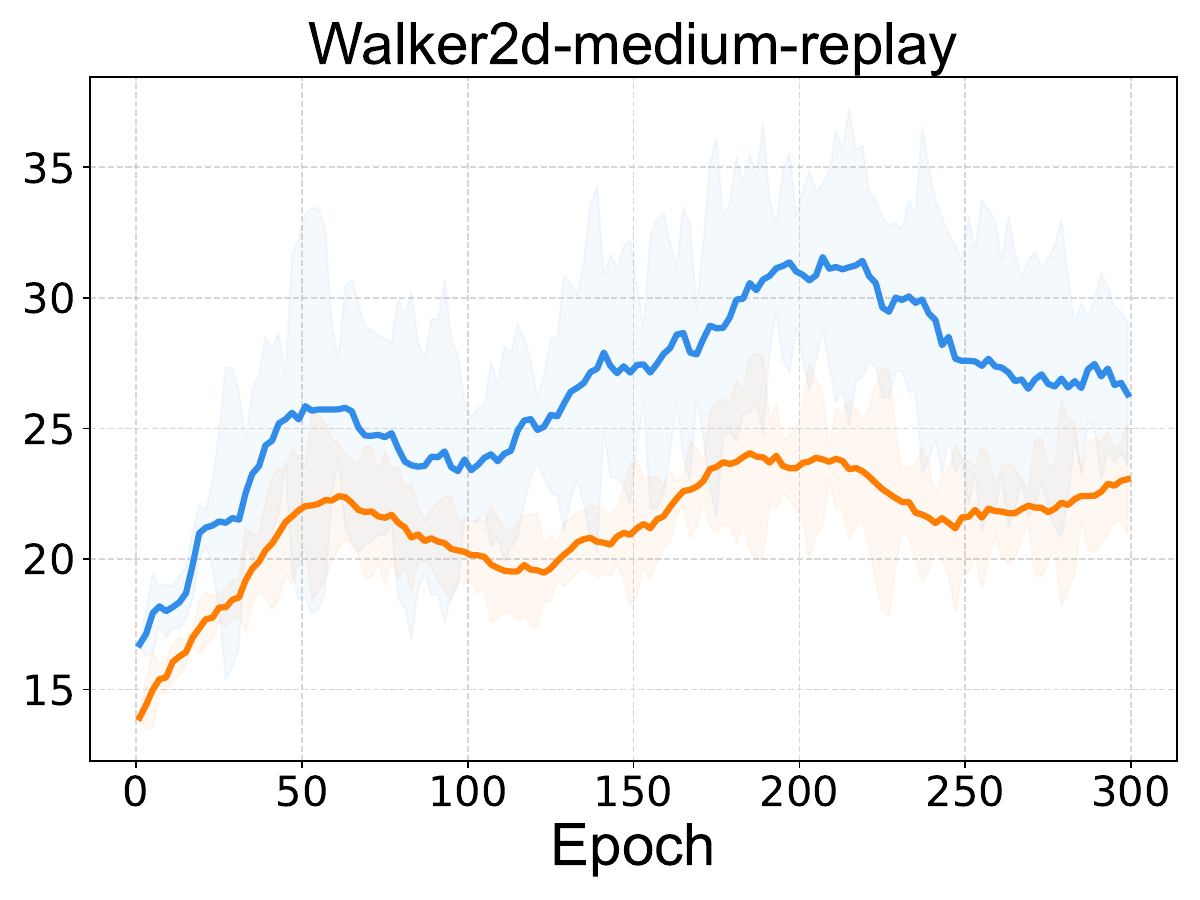}}
    \caption{Extension experiment for inconsistent problem.}
    \label{fig:app-incon}
\end{figure*}
\begin{figure*}[h]
	\centering
	\subfigure{\label{fig:app-ant-heter}\includegraphics[width=0.235\textwidth]{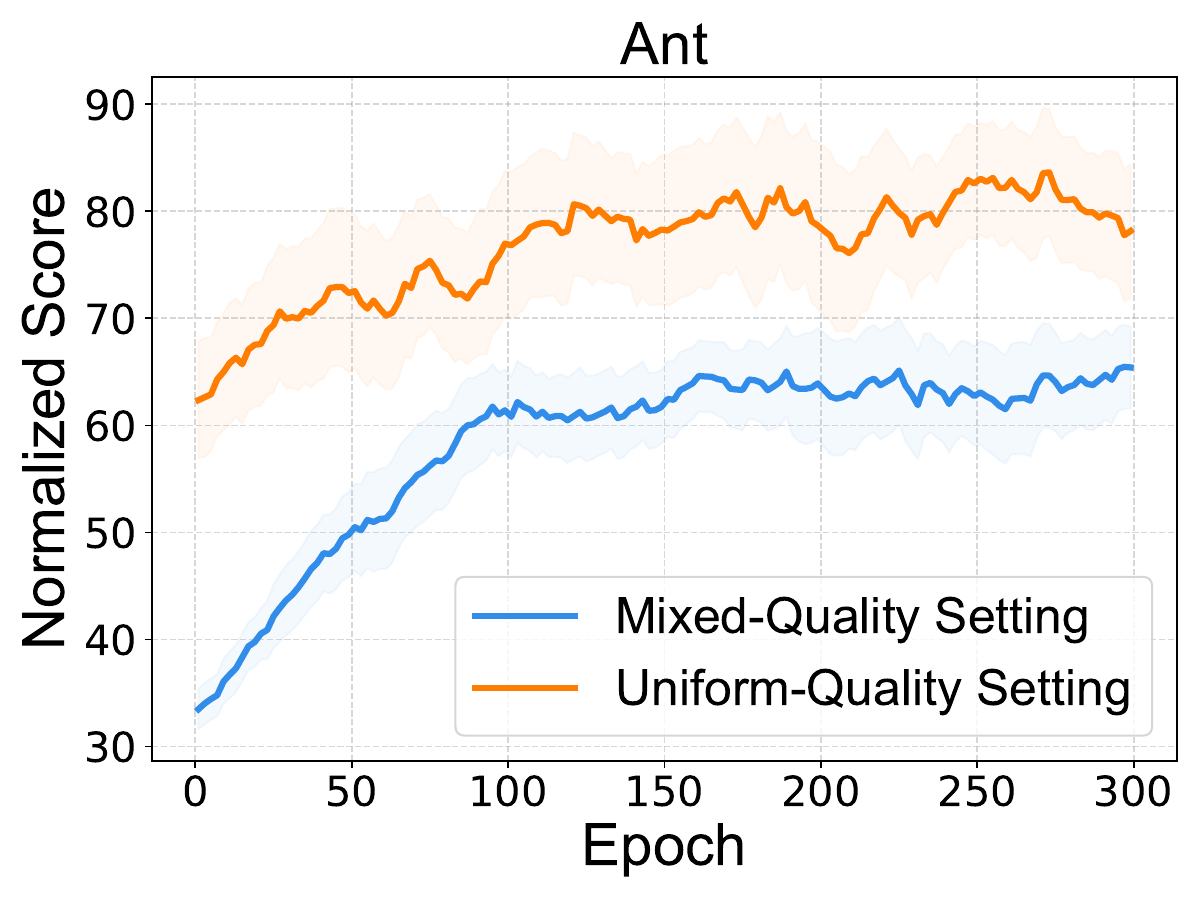}}
	\subfigure{\label{fig:app-halfcheetah-heter}\includegraphics[width=0.235\textwidth]{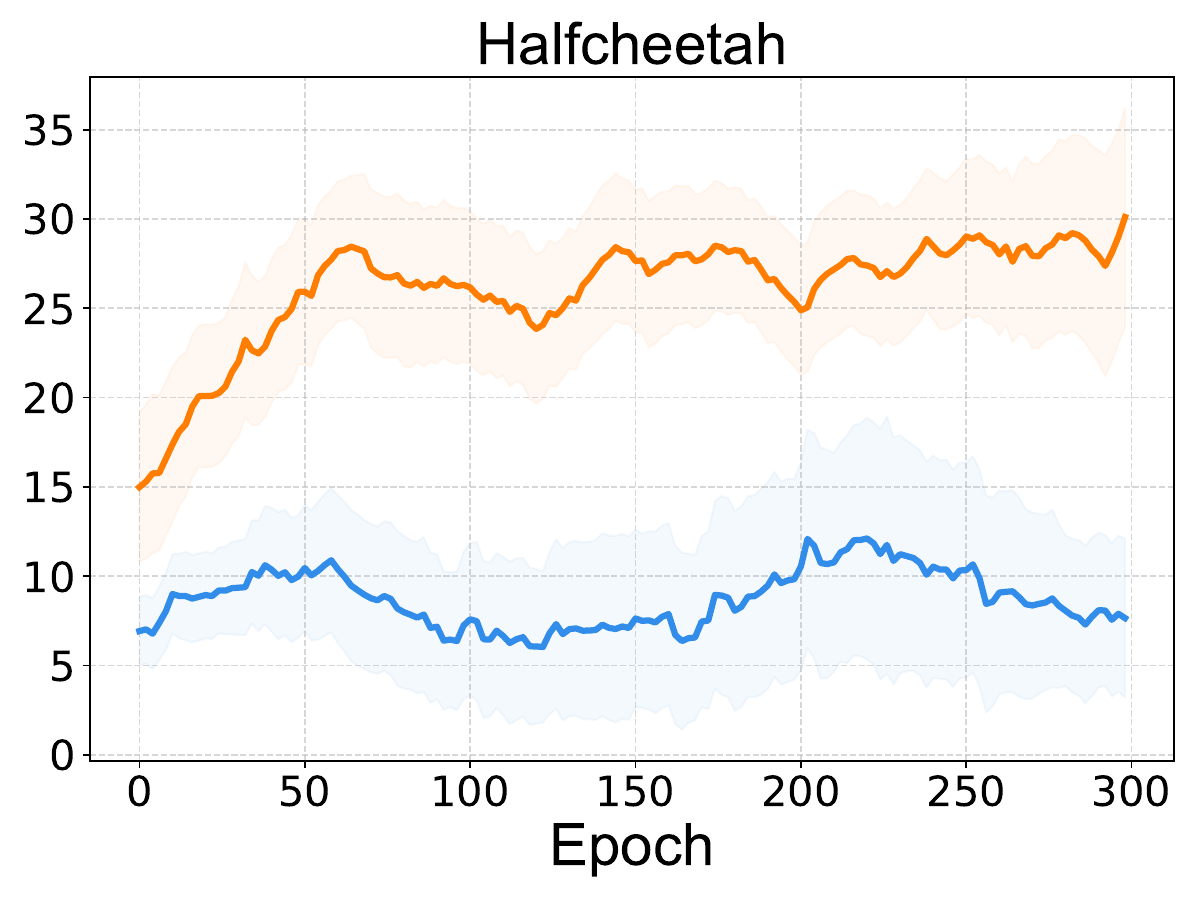}}
	\subfigure{\label{fig:app-hopper-heter}\includegraphics[width=0.235\textwidth]{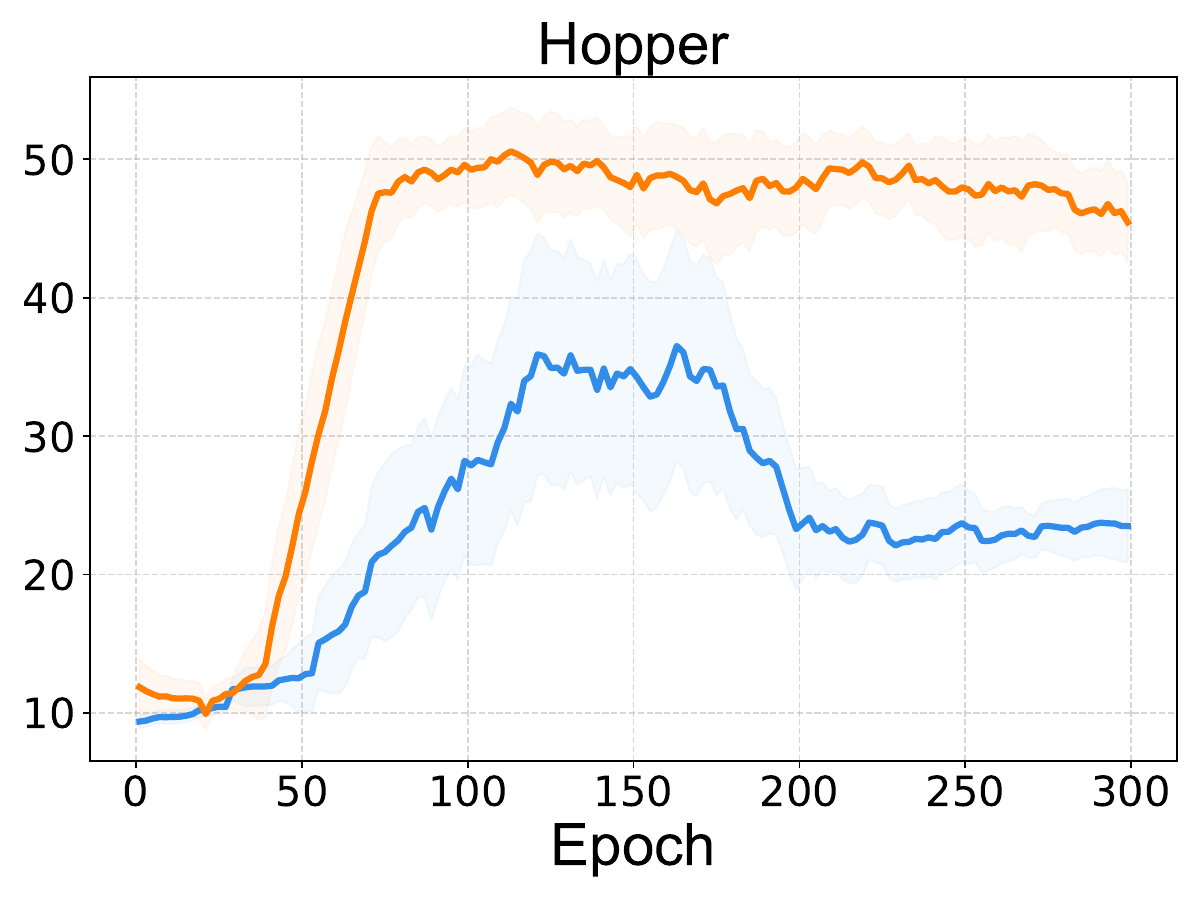}}
	\subfigure{\label{fig:app-walker2d-heter}\includegraphics[width=0.235\textwidth]{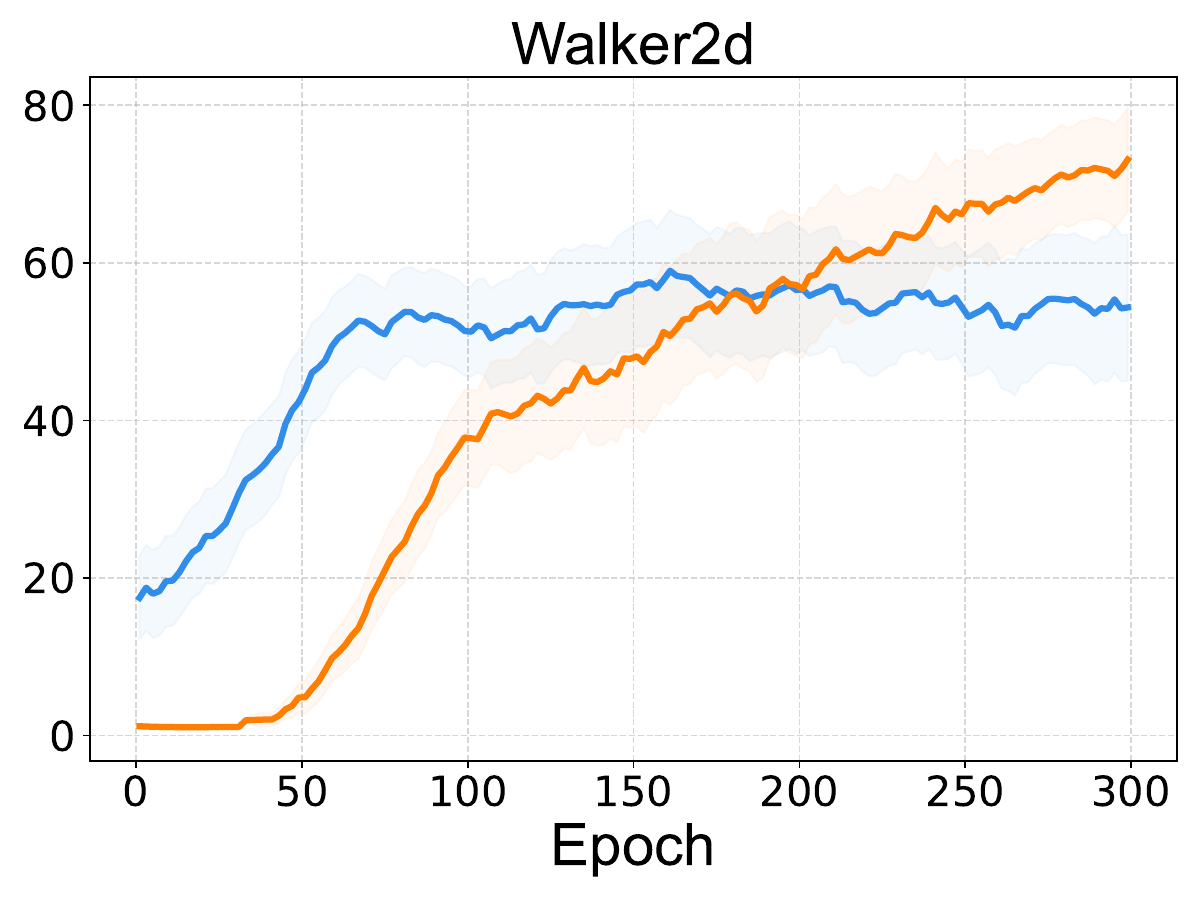}}
    \vspace{-0pt}
	\caption{Extension experiment for mixed-quality problem.}
 \label{fig:app-heter}
    \vspace{-10pt}
\end{figure*}
\let\citep\cite
\blue{
\section{Safe Policy improvement for global policy}
\label{sec:heter}
This section establishes safe policy improvement for FOVA in heterogeneous offline FRL. We first lay out the problem setup and a formal \textit{Preliminary on Heterogeneity}, fixing notation for client MDPs, discounted occupancy, and advantages. We then derive a \textit{global improvement bridge} via advantage decomposition and averaging, and, under an \textit{offline estimation/shift assumption}, transfer bounds from empirical to true MDPs. Finally, we prove a \textit{local regularized policy-improvement lower bound}, and aggregate these pieces into the \textit{global safe-improvement theorem}, where the error terms are explicitly controlled by the regularizers $\lambda$ and $\beta$. 

\subsection{Preliminary on Heterogeneity}
\label{subsec:heter-def}
Offline FRL heterogeneity stems from different initial states and transition dynamics.
The performance of a policy on $\mathcal{M}_k$ depends jointly on the initial state distribution $\mu_k$ and the transition dynamics $T_k$, both of which can be stochastic or deterministic. Variations in either of these factors can lead to different learned policies, resulting in the heterogeneity problem.
For each client $k$, we have the local objective as
$
J_k(\pi)=J(\widetilde{M}_k, \pi),
$
where
$\widetilde{M}_k$ denotes the empirical Markov Decision Process (MDP) constructed from the client’s local dataset $D_k$. Specifically, $\widetilde{M}_k = (\mathcal{S}, \mathcal{A}, \widetilde{T_k}, \widetilde{\mathcal{R}_k}, \gamma, \widetilde{\mu_k})$ consists of the estimated transition kernel $\widetilde{T_k}$, estimated reward function $\widetilde{\mathcal{R}_k}\le\widetilde{\mathcal{R}}^{\max}_k$, and estimated initial state distribution $\widetilde{\mu_k}$, all derived from $D_k$.
The value $J_k(\pi)$ for a given policy $\pi$ can be expressed as:
$$
J_k(\pi) = \frac{1}{1-\gamma} \sum_{\bs} d^{\pi}_{\widetilde{M}_k}(\bs) \sum_{\ba} \pi(\ba|\bs) \mathcal{R}_k(\bs,\ba),
$$
where $d^{\pi}_{\widetilde{M}_k}(\bs)$ represents the normalized discounted state distribution, defined as:
$$
d^{\pi}_{\widetilde{M}_k}(\bs) = (1-\gamma) \sum_{t=0}^{\infty} \gamma^t \Pr(\bs_t = \bs \mid \pi, \widetilde{\mu_k}, \widetilde{T_k}).
$$
In offline federated reinforcement learning, the global objective aggregates the local objectives uniformly across $K$ clients. The primary goal is to learn a policy based on the set of offline datasets $\{\widetilde{M_k}\}_{k=1}^{K}$:
$$
\max_{\pi} \; J(\pi) = \frac{1}{K} \sum_{k=1}^{K} J_k(\pi).
$$
To quantify the heterogeneity between MDPs, we introduce the following quantity on the
level of heterogeneity. Similar definitions are widely used in the FL/FRL work
\citep{li2020federated,wang2023federated,xie2022fedkl,xie2023client}

\begin{definition}[Level of Heterogeneity]
\label{def:heterogeneity_measure}
For any policy $\pi$, the heterogeneity level of client $n$ is measured by the Frobenius norm $\lVert \mathbf{H}_k \rVert_F$. The heterogeneity matrix is constructed as 
\begin{align}
    \mathbf{H}_k = \sum_{n=1}^{K} \frac{1}{K} (\mathbf{\Lambda}_k)^{-1} \mathbf{\Lambda}_n \mathbf{A}_n - \mathbf{A}_k,
\end{align}
where $\mathbf{\Lambda}_k$ denotes the diagonal state visitation frequency matrix for dataset $k$, while $\mathbf{A}_k$ represents the advantage matrix under transition dynamic $T_k$.
\end{definition}
This matrix $\mathbf{H}_k$ captures the deviation between client $n$'s local performance and the federated average performance. Since policy and Q-value directly determine gradient directions, significant deviations indicate potential optimization conflicts in federated learning. In IID settings, $\mathbf{H}_k$ reduces to the zero matrix and $\lVert \mathbf{H} \rVert_F = 0$. Under non-IID conditions, the norm's magnitude quantifies the client's environmental divergence. 
The matrix $\mathbb{A}_{k}$ is a $\left|\mathcal{S}\right| \times \left|\mathcal{A}\right|$ matrix, with the $(i, j)$-th entry denoted by $A_k(\bs_i, \ba_j)$. This entry represents the advantage function for the $i$-th state and $j$-th action, capturing the difference between the action-value function and the state-value function. It quantifies the benefit of taking action $\ba_j$ at state $\bs_i$ within the local environment of client $k$.
The policy advantage $\mathbb{A}_k(\pi')$ can be expressed as:
$$
\mathbb{A}_k(\pi) = \sum_{\bs} d^{\pi}_{\widetilde{M}_k}(\bs) \sum_{\ba} \pi(\ba|\bs) A_k(\bs, \ba)
.
$$
\subsection{Advantage Decomposition and Averaging}
\begin{lemma}
\label{lemma:advantage-bound}
For all clients $k = 1, \dots, K$, the following bound holds:
\begin{align}
\sum_{n=1}^{K} \frac{1}{K} \mathbb{A}_{n} (\pi_k^{t+1}) 
&\mathop{\geq} \mathbb{A}_{k} (\pi_k^{t+1}) - \frac{2 \|\mathbf{H}_k\|_F}{1-\gamma} \mathbb{E}_{\bs \sim d^{\pi_{\beta_k}}_{\widetilde{M}_k}} \left[ D_{TV} (\pi_k^{t+1} (\cdot | \bs) \| \pi_{\beta_k} (\cdot | \bs)) \right],\nonumber\\
&\geq \mathbb{A}_{k} (\pi_k^{t+1}) - \frac{\sqrt{2} }{1-\gamma}\|\mathbf{H}_k\|_F 
\sqrt{\mathbb{E}_{\bs \sim d^{\pi_{\beta_k}}_{\widetilde{M}_k}} \left[ D_{KL} (\pi_k^{t+1} (\cdot | \bs) \| \pi_{\beta_k} (\cdot | \bs)) \right]},
\end{align}
\end{lemma}
\begin{proof}
    This proof uses some useful lemmas and corollaries from \citep{xie2022fedkl}. By substituting $q_k = \frac{1}{K}$ and $\pi_{\beta_k}$ into \citep{xie2022fedkl}[Theorem I], we easily derive the following:
    \begin{align}
        \left|\sum_{n=1}^{K} \frac{1}{K} \mathbb{A}_{n} (\pi_k^{t+1}) - \mathbb{A}_{k} (\pi_k^{t+1})\right| \leq \mathbb{A}_{k} (\pi_k^{t+1}) - \|\mathbf{H}_k\|_F \sqrt{\sum_{s,a}[\frac{1}{1-\gamma}d^{\pi_{\beta_k}}_{\widetilde{M}_k}(\pi_k^{t+1} (\cdot | \bs), \pi_{\beta_k} (\cdot | \bs))]^2}
    \end{align}
Thus, we have
\begin{align}
\sum_{n=1}^{K} \frac{1}{K} \mathbb{A}_{n} (\pi_k^{t+1}) 
&\geq \mathbb{A}_{k} (\pi_k^{t+1}) - \|\mathbf{H}_k\|_F \sqrt{\sum_{s,a}[\frac{1}{1-\gamma}d^{\pi_{\beta_k}}_{\widetilde{M}_k}(\pi_k^{t+1} (\cdot | \bs), \pi_{\beta_k} (\cdot | \bs))]^2},\nonumber\\
&\mathop{\geq}_{\text{(a)}} \mathbb{A}_{k} (\pi_k^{t+1}) - \frac{2 \|\mathbf{H}_k\|_F}{1-\gamma} \mathbb{E}_{\bs \sim d^{\pi_{\beta_k}}_{\widetilde{M}_k}} \left[ D_{TV} (\pi_k^{t+1} (\cdot | \bs) \| \pi_{\beta_k} (\cdot | \bs)) \right],
\end{align}
where (a) follows from norm inequalities. Let $d_k(\bs) = d^{\pi_{\beta_k}}_{\widetilde{M}_k}(\bs)$ and $\Delta_s = \pi_k^{t+1}(\cdot|\bs) - \pi_{\beta_k}(\cdot|\bs)$. Then
$$
\sqrt{\sum_s \Big[c d_k(\bs)\,\|\Delta_s\|_2\Big]^2}
\;\leq\; c\sum_s d_k(\bs)\,\|\Delta_s\|_2
\;\leq\; c\sum_s d_k(\bs)\,\|\Delta_s\|_1
= 2c\,\mathbb{E}_{\bs\sim d_k}\!\left[D_{TV}\big(\pi_k^{t+1}\,\|\pi_{\beta_k}\big)\right],
$$
where we used $\sqrt{\sum x_s^2} \leq \sum |x_s|$, the inequality $\|\cdot\|_2 \leq \|\cdot\|_1$, $c = \frac{1}{1-\gamma}$, and the identity $\|p-q\|_1 = 2D_{TV}(p\|q)$.
\end{proof}
The lemma impolies that after client $k$ updates to $\pi_k^{t+1}$, the federated average advantage is bounded below by the client’s own advantage minus a penalty that scales with heterogeneity $\sqrt{2}\|\mathbf{H}_k\|_F/(1-\gamma)$ and the average TV shift from the behavior policy (scaled by $\alpha$). Thus, in IID settings ($\|\mathbf{H}_k\|_F=0$) or with conservative/ trust-region updates (small TV), the penalty is negligible, so local improvements carry over to global gains.
Further, we show that the average policy advantage is not a direct measure of the final performance, as it is a proxy for the improvement of the global objective. A more useful result should directly measure
the improvement in terms of the global expected discounted reward. For that purpose, we have the following Lemma.


\begin{lemma}[Global performance improvement via averaged local advantages]
The next global policy are the uniform mixture 
$
\bar\pi^{t+1} := \tfrac{1}{K}\sum_{k=1}^{K}\pi_k^{t+1}
$
generated by the global aggregation in Algorithm~FOVA.  
The performance gap of global policy enjoys the following lower bound:
\begin{align}
J(\bar\pi^{t+1}) - J(\bar\pi^t) \;
\ge\;\frac{1}{K} 
\sum_{k=1}^{K} \left\{ \frac{\mathbb{A}_k(\pi_k^{t+1})}{(1-\gamma)}
- l_k \,\mathrm{TV}\!\left(\pi_k^{t+1},\,\bar\pi^t\right) 
- h_k \,\mathrm{TV}\!\left(\pi_k^{t+1},\,\pi_{\beta_k}\right) \right\},
\label{eq:improvement-gap}
\end{align}
where 
$\mathrm{KL}(\pi_1,\pi_2) = \mathbb{E}_{\bs \sim d_k} \left[ \mathrm{D_{KL}} \left(\pi_1(\cdot  | \bs) \middle|\middle| \pi_2(\cdot | \bs)\right) \right]$,
$l_k = \frac{{2}\gamma\max_{s,a} |\mathbb{E}[A_k(\bs,\ba)]|}{(1-\gamma)^2}$,
$h_k = \frac{{2}}{(1-\gamma)^2}\|\mathbf{H}_k\|_{F}$,
and
$\mathbf{H}_k$ be defined as in Definition~\ref{def:heterogeneity_measure}.
\end{lemma}

\begin{proof}
We first recall two useful facts, following the \citep{xie2022fedkl}[Lemma II and III].  
For any policies $\pi,\hat{\pi},\pi'$, and $\gamma\le 1$, the surrogate objective and advantage satisfy
\begin{align}
\mathbb{A}_{\pi}(\gamma\hat{\pi}+(1-\gamma)\pi') 
&= \gamma \mathbb{A}_{\pi}(\hat{\pi})+(1-\gamma)\mathbb{A}_{\pi}(\pi').
\label{eq:linearity-a}
\end{align}
For any state $s$ and any policy $\pi$, we have the TV distance bound
\begin{align}
D_{\mathrm{TV}}\!\big(\pi(\cdot|\bs),\bar\pi^{t+1}(\cdot|\bs)\big)
&:= \tfrac12 \,\big\|\pi(\cdot|\bs)-\bar\pi^{t+1}(\cdot|\bs)\big\|_1 \nonumber\\
&= \tfrac12 \left\| \tfrac1K \sum_{k=1}^K \big(\pi(\cdot|\bs)-\pi_k^{t+1}(\cdot|\bs)\big) \right\|_1 \nonumber\\
&\le \tfrac1{2K} \sum_{k=1}^K \big\|\pi(\cdot|\bs)-\pi_k^{t+1}(\cdot|\bs)\big\|_1 \\
&= \tfrac1K \sum_{k=1}^K D_{\mathrm{TV}}\!\big(\pi(\cdot|\bs),\pi_k^{t+1}(\cdot|\bs)\big).
\label{eq:TV-bound}
\end{align}
We analyze the performance difference.  
Starting from the performance difference bound, we obtain
\begin{align}
J(\bar\pi^{t+1}) 
&\mathop{=}^{\text{(a)}} \sum_{n=1}^{K}J_n(\bar\pi^{t+1}) \nonumber\\
&\mathop{\ge}^{\text{(b)}} \sum_{n=1}^{K}\left\{J_n(\pi^{t}) +\frac{1}{K(1-\gamma)}\mathbb{A}_{n}\!\Big(\tfrac{1}{K}\sum_{k=1}^{K}\pi_k^{t+1}\Big)\right.
\left.-\frac{2c_{n}}{1-\gamma}\,\mathbb{E}_{\bs\sim d_{n}}\!\left[D_{TV}(\bar\pi^t(\cdot|\bs)\,\|\,\bar\pi^{t+1}(\cdot|\bs))\right]\right\}\nonumber\\
&\mathop{=}^{\text{(c)}} J(\bar\pi^t)+\frac{1}{K(1-\gamma)}\sum_{n=1}^{K}\left\{\mathbb{A}_{n}\!\Big(\tfrac{1}{K}\sum_{k=1}^{K}\pi_k^{t+1}\Big)\right.
\left.-\frac{2c_{n}}{1-\gamma}\,\mathbb{E}_{\bs\sim d_{n}}\!\left[D_{TV}(\bar\pi^t(\cdot|\bs)\,\|\,\bar\pi^{t+1}(\cdot|\bs))\right]\right\}\nonumber
\end{align}
where $\bar\pi^{t+1} = \tfrac{1}{K}\sum_{k=1}^{K}\pi_k^{t+1}$, $c_n = \gamma \max_{s,a} |\mathbb{E}[A_n(\bs,\ba)]|$, and (b) holds by the definition of $\mathbb{A}_{n}$ and the result from \citep{achiam2017constrained,xie2022fedkl}
, while equalities (a) and (c) hold by the definition of $J$.
Applying linearity of $\mathbb{A}_{n}$ in Eq.~\eqref{eq:linearity-a} and the TV dsitance bound in Eq.~\eqref{eq:TV-bound} yields
\begin{align}
J(\bar\pi^{t+1}) 
\ge J(\bar\pi^t)+\frac{1}{K^2(1-\gamma)}\sum_{n=1}^{K}\sum_{k=1}^{K}\Big\{\mathbb{A}_{n}(\pi_k^{t+1}) 
 -\frac{2c_{n}}{1-\gamma}\,\mathbb{E}_{\bs\sim d_{n}}\!\left[D_{TV}(\bar\pi^t(\cdot|\bs)\,\|\,\pi_n^{t+1}(\cdot|\bs))\right]\Big\}.
\end{align}
Exchanging the order of summation gives
\begin{align}
J(\bar\pi^{t+1}) 
\ge J(\bar\pi^t)+\frac{1}{K^2(1-\gamma)}\sum_{n=1}^{K}\sum_{k=1}^{K}\mathbb{A}_{n}(\pi_k^{t+1})
-\sum_{n=1}^{K}\frac{2c_{n}}{K(1-\gamma)^2}\mathbb{E}_{\bs\sim d_{n}}\!\left[D_{TV}(\bar\pi^t(\cdot|\bs)\,\|\,\pi_n^{t+1}(\cdot|\bs))\right].\nonumber
\end{align}
By applying the bound provided in \text{Lemma \ref{lemma:advantage-bound}}, we can express the value function update as follows:
\begin{align}
J(\bar\pi^{t+1}) - J(\bar\pi^t) \ge \frac{1}{K(1-\gamma)}& \sum_{n=1}^{K} \Bigg\{ \mathbb{A}_n(\pi_n^{t+1}) \notag 
- \frac{2c_{n}}{1-\gamma} \mathbb{E}_s \left[ D_{TV}(\bar\pi^t(\cdot|\bs) \| \pi_n^{t+1}(\cdot|\bs)) \right] \notag \\
&\quad - \frac{{2}}{1-\gamma} \|\mathbf{H}_n\|_F {\mathbb{E}_{\bs \sim d_{n}} \left[ D_{TV} (\pi_n^{t+1} (\cdot | \bs) \| \pi_{\beta_n} (\cdot | \bs)) \right]} \Bigg\}.
\end{align}
Therefore, we obtain the conclusion:
\[
J(\bar\pi^{t+1})-J(\bar\pi^t)\;\ge\;\sum_{k=1}^{K}\!\left\{\frac{\mathbb{A}_k(\pi_k^{t+1})}{K(1-\gamma)}
-\tfrac{l_k}{K} {(\mathrm{TV}(\pi_k^{t+1},\bar\pi^t))} 
-\tfrac{h_k}{K} {(\mathrm{TV}(\pi_k^{t+1},\pi_{\beta_k}))} \right\},
\]
where 
$\mathrm{TV}(\pi_1,\pi_2) = \mathbb{E}_{\bs \sim d_k} \left[ \mathrm{D_{TV}} \left(\pi_1(\cdot  | \bs) \middle|\middle| \pi_2(\cdot | \bs)\right) \right]$,
$l_k = \frac{2c_k}{(1-\gamma)^2}$,
and
$h_k = \frac{2}{(1-\gamma)^2}\,\|\mathbf{H}_k\|_{F}$.
\end{proof}

\subsection{Offline Estimation Error under Heterogeneity}
Next, we reformulate original Assumption by introducing the heterogeneity, and some derive the some conclusions of main paper as Lemma \ref{lemma:JMDP-heter} and Lemma \ref{lemma:policy-improvement-local-convergence}, respectively.
\begin{assumption}
\label{assu:cql-heter}
For any client \(k\) and any \((\bs,\ba)\) in the support of \(D_k\), with probability at least \(1-\delta\),
\[
\big|\,\widetilde{\mathcal R}_k(\bs,\ba)- r(\bs,\ba)\,\big|
\;\le\;
\frac{C^{\mathrm{vcql}}_{r,\delta,k}}{\sqrt{N_k(\bs,\ba)}},
\qquad
\big\|\,\widetilde T_k(\cdot\mid \bs,\ba)-T(\cdot\mid \bs,\ba)\,\big\|_{1}
\;\le\;
\frac{C^{\mathrm{vcql}}_{T,\delta,k}}{\sqrt{N_k(\bs,\ba)}},
\]
where \(N_k(\bs,\ba)\) is the number of samples of \((\bs,\ba)\) in \(D_k\), and \(C^{\mathrm{vcql}}_{r,\delta,k},C^{\mathrm{vcql}}_{T,\delta,k}= \widetilde{\mathcal O}(\sqrt{\log(1/\delta)})\).
\end{assumption}
\begin{lemma}
\label{lemma:JMDP-heter}
Fix any client $k$ and any policy $\pi$. With probability at least $1-\delta$ (over the data used to build $\widetilde{M}_k$), for every MDP $M$ we have
\[
\big|\,J(M,\pi)-J(\widetilde{M}_k,\pi)\,\big| \;\le\; \tilde{\xi}_k(\pi),
\]
where
$
\tilde{\xi}_k(\pi)\;\doteq\; 
\frac{2\gamma r_{\max} C_{T,\delta,k}}{(1-\gamma)^2}\;
\mathbb{E}_{\bs\sim d^{\pi}_{\widetilde{M}_k}}\!\left[ \sqrt{ D_{\mathrm{VCQL}}(\pi,\pi_{\beta_k})(\bs)\,\frac{|\mathcal{A}|}{|\mathcal{D}_k|} } \right]
\;+\;
\frac{C_{r,\delta,k}}{1-\gamma}\;
\mathbb{E}_{(\bs,\ba)\sim (d^{\pi}_{\widetilde{M}_k},\pi)}\!\left[\frac{1}{\sqrt{|\mathcal{N}_k(\bs,\ba)|}}\right],
$
and
$
D_{\mathrm{VCQL}}(\pi_1,\pi_2)(\bs)\;\doteq\;1+\sum_{\ba}\pi_1(\ba\mid \bs)\Big(\frac{\pi_1(\ba\mid \bs)}{\pi_2(\ba\mid \bs)}-1\Big).
$
\end{lemma}
\begin{proof}
Applying Lemma 1 in the main paper with $\widetilde{M}=\widetilde{M}_k$ yields the one-sided bound
$J(\widetilde{M}_k,\pi)\le J(M,\pi)+\tilde{\xi}_k(\pi)$ for every $M$.
Repeating the same argument while swapping the roles of $M$ and $\widetilde{M}_k$ (the proof depends only on the occupancy measure and coverage terms, which are evaluated under $d^{\pi}_{\widetilde{M}_k}$ as written above) gives
$J(M,\pi)\le J(\widetilde{M}_k,\pi)+\tilde{\xi}_k(\pi)$.
Combining the two inequalities proves the claim.
\end{proof}

\begin{corollary}
\label{cor:transfer}
Let $\bar\pi^{t}$ and $\bar\pi^{t+1}$ be two successive global policies. Define
$\tilde{\xi}_k^{t}\doteq \tilde{\xi}_k(\bar\pi^{t})$ and
$\tilde{\xi}_k^{t+1}\doteq \tilde{\xi}_k(\bar\pi^{t+1})$.
Then, with probability at least $1-\delta$, for every MDP $M$ and client $k$,
\[
J(M,\bar\pi^{t+1})-J(M,\bar\pi^{t})
\;\ge\;
\Big(J(\widetilde{M}_k,\bar\pi^{t+1})-J(\widetilde{M}_k,\bar\pi^{t})\Big)
\;-\;\tilde{\xi}_k^{t+1}\;-\;\tilde{\xi}_k^{t}.
\]
\end{corollary}
\begin{proof}
By Lemma~\ref{lemma:JMDP-heter},
$J(M,\bar\pi^{t+1}) \ge J(\widetilde{M}_k,\bar\pi^{t+1}) - \tilde{\xi}_k^{t+1}$
and
$J(M,\bar\pi^{t}) \le J(\widetilde{M}_k,\bar\pi^{t}) + \tilde{\xi}_k^{t}$.
Subtract the second inequality from the first.
\end{proof}


\begin{lemma}
\label{lemma:policy-improvement-local-convergence}
If Assumption~\ref{assu:cql-heter} holds, then with high probability the following performance improvement bound holds:
\begin{align}
    &J(\widetilde{M}_k,\pi_k^{t+1}) - J(\widetilde{M}_k,\pi_{\beta_k}) \\
    \;\ge\;
    &\lambda\, \mathrm{KL}(\bar{\pi}^*_k,\pi_k^{t+1}) 
    + \beta\, \mathrm{KL}(\bar{\pi}^*_k,\pi_{\beta_k})
    -\frac{4\alpha}{\delta(1-\gamma)}\mathrm{TV}(\bar{\pi}^*_k,{\pi}_k)
    -\frac{2\alpha}{\delta(1-\gamma)}\mathrm{TV}(\bar{\pi}^*_k,\pi_{\beta_k}).\nonumber
\end{align}
\end{lemma}
\begin{proof}
By optimality of the Problem (17), $\pi_k^{t+1}$ maximizes
$$
J(\widetilde M_k,\pi)-\frac{\alpha\,\nu(\pi)}{1-\gamma}-\lambda\,\mathrm{KL}(\bar\pi_k^*,\pi),
$$
hence
\begin{equation}
\label{eq:awr-opt}
J(\widetilde M_k,\pi_k^{t+1})
;\ge;
J(\widetilde M_k,\bar\pi_k^*)
+\lambda,\mathrm{KL}(\bar\pi_k^*,\pi_k^{t+1})
+\frac{\alpha}{1-\gamma}\Big(\nu(\pi_k^{t+1})-\nu(\bar\pi_k^*)\Big).
\end{equation}
By optimality of Problem (4), $\bar\pi_k^*$ maximizes
$$
J(\widetilde M_k,\pi)-\frac{\alpha\,\nu(\pi)}{1-\gamma}-\beta\,\mathrm{KL}(\pi,\pi_{\beta_k}),
$$
so
\begin{equation}
\label{eq:global-opt}
J(\widetilde M_k,\bar\pi_k^*)
;\ge;
J(\widetilde M_k,\pi\_{\beta_k})
+\beta,\mathrm{KL}(\bar\pi_k^*,\pi\_{\beta_k})
+\frac{\alpha}{1-\gamma}\Big(\nu(\bar\pi_k^*)-\nu(\pi\_{\beta_k})\Big).
\end{equation}
Combining \eqref{eq:awr-opt} and \eqref{eq:global-opt} gives

$$
\begin{aligned}
J(\widetilde{M}_k,\pi_k^{t+1}) - J(\widetilde{M}_k,\pi_{\beta_k})
\;\ge\;&
\lambda\, \mathrm{KL}(\bar{\pi}^*_k,\pi_k^{t+1}) 
+ \beta\, \mathrm{KL}(\bar{\pi}^*_k,\pi_{\beta_k}) \\
&\quad
+\frac{\alpha}{1-\gamma}\!\left[\nu(\pi_k^{t+1})-\nu(\bar\pi_k^*)\right]
+\frac{\alpha}{1-\gamma}\!\left[\nu(\bar\pi_k^*)-\nu(\pi_{\beta_k})\right].
\end{aligned}
$$
We now upper-bound the $\nu$-differences via total variation. Under the $\delta$-coverage condition (as in the main text), a standard manipulation yields, for policies $\pi,\pi'$,
$$
\big|\nu(\pi)-\nu(\pi')\big|
\;\le\;
\frac{2}{\delta}\,\mathbb{E}_{s\sim d_k}\!\left[\sum_{a}\big|\pi(a\mid s)-\pi'(a\mid s)\big|\right]
=
\frac{4}{\delta}\,\mathbb{E}_{s\sim d_k}\!\left[D_{\mathrm{TV}}\big(\pi(\cdot\mid s),\pi'(\cdot\mid s)\big)\right].
$$
Applying this bound (with the sharper constant $2/\delta$ in the $\nu(\bar\pi_k^*)-\nu(\pi_{\beta_k})$ term due to the $\pi_{\beta_k}$-normalization in the algebra used in the main paper) gives
$$
\frac{\alpha}{1-\gamma}\big(\nu(\pi_k^{t+1})-\nu(\bar\pi_k^*)\big)\ \ge\ -\frac{4\alpha}{\delta(1-\gamma)}\,\mathrm{TV}(\bar\pi_k^*,\pi_k^{t+1}),
\qquad
\frac{\alpha}{1-\gamma}\big(\nu(\bar\pi_k^*)-\nu(\pi_{\beta_k})\big)\ \ge\ -\frac{2\alpha}{\delta(1-\gamma)}\,\mathrm{TV}(\bar\pi_k^*,\pi_{\beta_k}),
$$
where $\mathrm{TV}(\pi,\pi'):=\mathbb{E}_{s\sim d_k}\!\left[D_{\mathrm{TV}}(\pi(\cdot\mid s),\pi'(\cdot\mid s))\right]$.
Substituting yields
$$
\begin{aligned}
J(\widetilde{M}_k,\pi_k^{t+1}) - J(\widetilde{M}_k,\pi_{\beta_k})
\;\ge\;
\lambda\, \mathrm{KL}(\bar{\pi}^*_k,\pi_k^{t+1}) 
+ \beta\, \mathrm{KL}(\bar{\pi}^*_k,\pi_{\beta_k})
-\frac{4\alpha}{\delta(1-\gamma)}\,\mathrm{TV}(\bar{\pi}^*_k,{\pi}_k^{t+1})
-\frac{2\alpha}{\delta(1-\gamma)}\,\mathrm{TV}(\bar{\pi}^*_k,\pi_{\beta_k}).
\end{aligned}
$$
This is exactly the desired inequality.
\end{proof}

\begin{remark}
    Building upon our prior analysis, particularly that of local policy improvement, 
    we emphasize that the results pertain exclusively to the local level. 
    Although heterogeneity introduces variations in local data distributions, 
    we find that the local policy $\pi_k$ still  
    achieves policy improvement. 
\end{remark}
By contrast, the situation is more intricate 
for the global policy, as it necessitates the aggregation of multiple local policies. 
In this case, distributional discrepancies among local datasets may exert a 
substantial influence on the effectiveness of aggregation. 
A detailed discussion of this issue will be presented in the subsequent subsection.


\subsection{Safe Policy Improvement Guarantee}
\label{subsec:global-spi}


Given quantified heterogeneity and per-client analysis, we prove that FOVA guarantees safe policy improvement of the global policy in heterogeneous offline FRL.
\begin{theorem}
Under Assumption~\ref{assu:cql-heter}, Algorithm FOVA achieves a
\(\mathcal{B}(\lambda,\beta)\)-safe policy improvement for the global policy in the true MDP \(M\); specifically,
\[
    J(M,\bar\pi^{t+1}) \;\ge\; J(M,\bar\pi^{t}) \;-\; \mathcal{B}(\lambda,\beta),
\]
where
\[
\mathcal{B}(\lambda,\beta)
= \,\overline{\tilde{\xi}_k^{t+1}+\tilde{\xi}_k^{t}}
+\,\overline{\,l_k\,\mathrm{TV}(\pi_{\beta_k},\bar\pi^{t})\,}
+\,\frac{1}{8\lambda}\,\overline{(l_k+h_k+2\sigma)^2}
+\,\frac{1}{8\beta}\,\overline{(l_k+h_k+\sigma)^2},
\]
and \(\overline{x_k}\) denotes the client-wise average, i.e.,
\(\overline{x_k} \triangleq \tfrac{1}{K}\sum_{k=1}^K x_k\).
\end{theorem}
\begin{proof}
We start from part of the right-hand side in Eq.~\eqref{eq:improvement-gap}.
\begin{align}
    &\frac{\mathbb{A}_k(\pi_k^{t+1})}{(1-\gamma)}
- l_k \,\mathrm{TV}\!\left(\pi_k^{t+1},\,\bar\pi^t\right) 
- h_k \,\mathrm{TV}\!\left(\pi_k^{t+1},\,\pi_{\beta_k}\right) \nonumber\\
=&\big(J(\widetilde{M}_k,\pi)-J(\widetilde{M}_k,\pi_{\beta_k})\big)
- l_k \,\mathrm{TV}\!\left(\pi_k^{t+1},\,\bar\pi^t\right) 
- h_k \,\mathrm{TV}\!\left(\pi_k^{t+1},\,\pi_{\beta_k}\right) ,
\label{eq:J-A}
\end{align}
because for any $\pi$, it follows
\begin{align}
\big(J(\widetilde{M}_k,\pi)-J(\widetilde{M}_k,\pi_{\beta_k})\big)
&\mathop{=}\Big(\mathbb{E}_{s_0\sim\mu_{k}}[V^{\pi}(\bs_0)]-\mathbb{E}_{s_0\sim\mu_{k}}[V^{\pi_{\beta_k}}(\bs_0)]\Big)\\
&\mathop{=}\,\mathbb{E}_{s_0\sim\mu_{k},\pi}\!\left[\sum_{t=0}^{\infty}\gamma^t r_t- V^{\pi_{\beta_k}}(\bs_0)\right]\\
&\mathop{=}\sum_{t=0}^{\infty}\gamma^t\,\mathbb{E}_{\pi}\!\left[r_t+\gamma V^{\pi_{\beta_k}}(s_{t+1})-V^{\pi_{\beta_k}}(\bs_t)\right]\\
&\mathop{=}\sum_{t=0}^{\infty}\gamma^t\,\mathbb{E}_{\pi}\!\left[A_k(\bs_t,\ba_t)\right]\\
&\mathop{=}\sum_{t=0}^{\infty}\gamma^t\sum_{\bs}\Pr(\bs_t=\bs\mid \widetilde{M}_k,\pi)\sum_{\ba}\pi(\ba\mid \bs)\,A_k(\bs,\ba)\\
&\mathop{=}\limits^{(vi)}\sum_{\bs}\Big(\sum_{t=0}^{\infty}\gamma^t\Pr(\bs_t=\bs\mid \widetilde{M}_k,\pi)\Big)\sum_{\ba}\pi(\ba\mid \bs)\,A_k(\bs,\ba)\\
&=\frac{1}{1-\gamma}\sum_{\bs} d^{\pi}_{\widetilde{M}_k}(\bs)\sum_{\ba}\pi(\ba\mid \bs)\,A_k(\bs,\ba)\;
=\frac{1}{1-\gamma}\;\mathbb{A}_k(\pi).
\end{align}
from the definition of $J$ and $\mathbb{A}_k(\pi)$.
Substituting Eq.~\eqref{eq:J-A} and Lemma \ref{lemma:policy-improvement-local-convergence} into Eq.~\eqref{eq:improvement-gap}, we derive
\begin{align}
J(\bar\pi^{t+1})-J(\bar\pi^t)\;\ge\;
\frac{1}{K}\sum_{k=1}^{K}
f_k(\pi_k^{t+1},\bar{\pi}^*_k,\pi_{\beta_k},\bar\pi^t),
\label{eq:global-gap-step1}
\end{align}
where $f_k(\pi_k^{t+1},\bar{\pi}^*_k,\pi_{\beta_k},\bar\pi^t)={\lambda}\, \mathrm{KL}(\bar{\pi}^*_k,\pi_k^{t+1}) 
+ {\beta}\, \mathrm{KL}(\bar{\pi}^*_k,\pi_{\beta_k})
- l_k {\mathrm{TV}(\pi_k^{t+1},\bar\pi^t)} 
- h_k {\mathrm{TV}(\pi_k^{t+1},\pi_{\beta_k})}
-2\sigma\mathrm{TV}(\bar{\pi}^*_k,{\pi}_k)
-\sigma\mathrm{TV}(\bar{\pi}^*_k,\pi_{\beta_k})$
and $\sigma = \frac{2\alpha}{\delta(1-\gamma)}$.
Using triangle inequalities to re-center the linear TV terms on \(\bar\pi_k^*\), and then absorbing them into KL via Pinsker+Young, we obtain
\begin{align}
&\quad f_k(\pi_k^{t+1},\bar{\pi}^*_k,\pi_{\beta_k},\bar\pi^t) \nonumber\\
&=\lambda\, \mathrm{KL}(\bar{\pi}^*_k,\pi_k^{t+1}) 
+ \beta\, \mathrm{KL}(\bar{\pi}^*_k,\pi_{\beta_k})
- l_k \operatorname{TV}(\pi_k^{t+1},\bar\pi^t) 
- h_k \operatorname{TV}(\pi_k^{t+1},\pi_{\beta_k})
-2\sigma\,\operatorname{TV}(\bar{\pi}^*_k,{\pi}_k^{t+1})
-\sigma\,\operatorname{TV}(\bar{\pi}^*_k,\pi_{\beta_k}) \nonumber\\
&\ge 2\lambda\,\mathrm{TV}^2(\bar\pi^*_k,\pi_k^{t+1})
+2\beta\,\mathrm{TV}^2(\bar\pi^*_k,\pi_{\beta_k})
-(l_k+h_k+2\sigma)\,\mathrm{TV}(\bar\pi^*_k,\pi_k^{t+1})
-(l_k+h_k+\sigma)\,\mathrm{TV}(\bar\pi^*_k,\pi_{\beta_k})
- l_k\,\mathrm{TV}(\pi_{\beta_k},\bar\pi^t) \nonumber\\
&\ge -\frac{(l_k+h_k+2\sigma)^2}{8\lambda}
-\frac{(l_k+h_k+\sigma)^2}{8\beta}
- l_k\,\mathrm{TV}(\pi_{\beta_k},\bar\pi^t). 
\label{eq:lower_bound_fx}
\end{align}
Substituting Eq.~\eqref{eq:lower_bound_fx} into Eq.~\eqref{eq:global-gap-step1}, we obtain
\begin{align}
J(\bar\pi^{t+1})-J(\bar\pi^t)\;\ge\;
-\frac{1}{K}\sum_{k=1}^{K}
\Big\{\frac{(l_k+h_k+2\sigma)^2}{8\lambda}
+\frac{(l_k+h_k+\sigma)^2}{8\beta}
+ l_k\,\mathrm{TV}(\pi_{\beta_k},\bar\pi^t)\Big\},
\label{eq:global-gap-step2}
\end{align}
Combining Corollary \ref{cor:transfer} and Eq.~\eqref{eq:global-gap-step2}, we derive
\begin{align}
    &J(M,\bar\pi^{t+1})-J(M,\bar\pi^{t})\nonumber\\
    =&\frac{1}{K}\sum_{k=1}^{K}
    \{J(M,\bar\pi^{t+1})-J(M,\bar\pi^{t})\}\nonumber\\
    \;\ge&\;
    \frac{1}{K}\sum_{k=1}^{K}\Big\{\Big(J(\widetilde{M}_k,\bar\pi^{t+1})-J(\widetilde{M}_k,\bar\pi^{t})\Big)
    \;-\;\tilde{\xi}_k^{t+1}\;-\;\tilde{\xi}_k^{t}\Big\}\nonumber\\
    \;=&\; J(\bar\pi^{t+1})-J(\bar\pi^t)
    -\frac{1}{K}\sum_{k=1}^{K}\Big\{
    \;\tilde{\xi}_k^{t+1}\;+\;\tilde{\xi}_k^{t}\Big\}\nonumber\\
    \ge&\;
    -\frac{1}{K}\sum_{k=1}^{K}\Big\{
    \;\tilde{\xi}_k^{t+1}\;+\;\tilde{\xi}_k^{t}+ \frac{(l_k+h_k+2\sigma)^2}{8\lambda}
+\frac{(l_k+h_k+\sigma)^2}{8\beta}
+ l_k\,\mathrm{TV}(\pi_{\beta_k},\bar\pi^t)\Big\}\nonumber\\
     \;\ge&\;
-\overline{\tilde{\xi}_k^{t+1}+\tilde{\xi}_k^{t}}
-\overline{l_k\,\mathrm{TV}(\pi_{\beta_k},\bar\pi^t)}
-\frac{1}{8\lambda}\,\overline{(l_k+h_k+2\sigma)^2}
-\frac{1}{8\beta}\,\overline{(l_k+h_k+\sigma)^2}.
\end{align}
Under Definition \ref{def:heterogeneity_measure} and Assumption \ref{assu:cql-heter}, our analysis establishes that FOVA guarantees safe global-policy improvement. Specifically, with high probability and for sufficiently large $\beta$ and $\lambda$, the heterogeneity-induced loss $(h_k)$ and aggregation-induced loss $(l_k)$ are uniformly controlled, ensuring that global updates do not reduce performance.
\end{proof}
\begin{remark}
The theorem certifies that the global update is safely improving up to
explicit penalty terms tied to heterogeneity. In particular, the
aggregation-induced deviation $l_k$ is damped by the aggregation
regularizer $\lambda$, while the behavior-mismatch (dataset heterogeneity)
term $h_k$ is damped by the behavior regularizer $\beta$. Thus:
 when cross-client aggregation dominates ($l_k$ large), increase
$\lambda$ to tighten the bound; when dataset mismatch dominates
($h_k$ large), increase $\beta$ to tighten the bound. 
In summary, larger $(\lambda,\beta)$ yield a stronger safety
global policy improvement.
\end{remark}

\subsection{Convergence of FOVA}
Subsequently, we provide the compact properties and continuity properties in offline FRL.
\begin{lemma}[Local compact policy set]\label{lem:local-compact}
Fix finite state/action sets $\mathcal S,\mathcal A$ and $\eta\in(0,1/|\mathcal A|]$.
For client $k$, define
\[
\Pi_k \;:=\; \Big\{\pi_k\in\mathbb{R}^{|\mathcal S||\mathcal A|}:\ \forall s\in\mathcal S,\ \sum_{a\in\mathcal A}\pi_k(\ba|\bs)=1,\ \pi_k(\ba|\bs)\ge \zeta \Big\},
\]
and endow $\Pi_k$ with the $\ell_\infty$ norm. Then $\Pi_k$ is nonempty, convex, and compact.
\end{lemma}

\begin{proof}
Take any $\pi_k,\tilde\pi_k\in\Pi_k$ and $\alpha\in[0,1]$. For all $s,a$ we have
\[
\alpha\pi_k(\ba|\bs)+(1-\alpha)\tilde\pi_k(\ba|\bs)\ \ge\ \alpha\zeta+(1-\alpha)\zeta=\zeta,
\]
and for every $s$,
\[
\sum_a\!\big[\alpha\pi_k(\ba|\bs)+(1-\alpha)\tilde\pi_k(\ba|\bs)\big]
=\alpha\sum_a\pi_k(\ba|\bs)+(1-\alpha)\sum_a\tilde\pi_k(\ba|\bs)=\alpha+(1-\alpha)=1,
\]
hence $\alpha\pi_k+(1-\alpha)\tilde\pi_k\in\Pi_k$, which proves convexity.  
Now let $(\pi_k^{(m)})_{m\ge1}\subset\Pi_k$ converge to $\pi_k$ in $\ell_\infty$.
Then $\pi_k^{(m)}(\ba|\bs)\to\pi_k(\ba|\bs)$ for each $s,a$, and since $\pi_k^{(m)}(\ba|\bs)\ge\zeta$ for all $m$, the limit satisfies $\pi_k(\ba|\bs)\ge\zeta$. Moreover,
\[
\sum_a \pi_k(\ba|\bs)
=\sum_a \lim_{m\to\infty}\pi_k^{(m)}(\ba|\bs)
=\lim_{m\to\infty}\sum_a \pi_k^{(m)}(\ba|\bs)
=\lim_{m\to\infty}1=1,
\]
thus $\pi_k\in\Pi_k$, which shows that $\Pi_k$ is closed.  
For any feasible $\pi_k$ and any $s,a$, the normalization condition together with nonnegativity implies $0\le \pi_k(\ba|\bs)\le 1$, so $\|\pi_k\|_\infty\le 1$, hence $\Pi_k$ is bounded.  
In finite-dimensional spaces, closed and bounded sets are compact (Heine--Borel), therefore $\Pi_k$ is compact. Finally, since $\eta\le 1/|\mathcal A|$, the uniform policy $\pi_k(\ba|\bs)=1/|\mathcal A|$ is feasible, hence $\Pi_k$ is nonempty.
\end{proof}

\begin{lemma}[Global compact feasible set]\label{lem:global-compact}
Let $\Pi:=\prod_{k=1}^K \Pi_k$ and equip it with the max product norm
$\|\pi\|_{\max}:=\max_{k\in[K]}\|\pi_k\|_\infty$. If each $\Pi_k$ is compact (resp.\ convex), then $\Pi$ is compact (resp.\ convex).
\end{lemma}

\begin{proof}
If each $\Pi_k$ is convex, then their Cartesian product is also convex, so $\Pi$ is convex.  
To prove compactness, take any sequence $(\pi^{(m)})_{m\ge1}\subset\Pi$. Since $\Pi_1$ is compact, there is a subsequence along which the first component converges. From this subsequence, compactness of $\Pi_2$ ensures the existence of a further subsequence with a convergent second component. Repeating this argument up to $k=K$ yields a diagonal subsequence $(\pi^{(m_j)})_j$ such that $\pi^{(m_j)}_k\to\bar\pi_k\in\Pi_k$ for every $k$. Therefore
\[
\|\pi^{(m_j)}-\bar\pi\|_{\max}
=\max_{k\in[K]}\|\pi^{(m_j)}_k-\bar\pi_k\|_\infty \;\longrightarrow\;0,
\]
so $\pi^{(m_j)}\to\bar\pi=(\bar\pi_1,\ldots,\bar\pi_K)\in\Pi$, hence $\Pi$ is sequentially compact, and thus compact since it is a metric space. This completes the proof.
\end{proof}

\begin{assumption}[Feasibility-preserving update]\label{ass:feasible-update}
The update operator $\mathcal U$ satisfies $\mathcal U(\Pi)\subseteq \Pi$.
\end{assumption}
\begin{remark}
Assumption~\ref{ass:feasible-update} is met if, e.g., (i) each round solves a constrained subproblem over $\Pi$,
(ii) the server aggregates by convex combination of feasible local policies,
or (iii) one projects/clips and renormalizes back to $\Pi$ after local updates.
\end{remark}

\begin{lemma}[Continuity Properties]
\label{lemma:continuity_properties}
The following continuity properties hold for the offline FRL framework:
\begin{enumerate}
    \item (State visitation) The discounted visitation distribution $d^{\pi}_{M}$ is continuous in the policy $\pi$.
    \item (Q-function) For any policy $\pi$, the state--action value function $Q^{\pi}(\bs,\ba)$ is Lipschitz-continuous in $\pi$.
    \item (Derived quantities) As a consequence of \emph{(2)}, the state value $V^{\pi}(\bs)$, the advantage $A^{\pi}(\bs,\ba)$, and the expected return $J(M,\pi)$ are all Lipschitz-continuous in $\pi$.
    \item (Cross-policy value) For any policies $\pi$ and $\hat{\pi}$, the functional ${Q}^{\hat{\pi}}(\pi)$ is continuous in $(\pi,\hat{\pi})$.
    \item (Heterogeneity) For every client $k=1,\dots,K$, the heterogeneity level $\|\mathbf{H}_k\|_F$ is continuous in $\pi$.
    \item (Local objective) For every client $k=1,\dots,K$, the local policy-improvement objective is continuous in $\pi$.
\end{enumerate}
\end{lemma}

\begin{proof}
Items \emph{(1)}–\emph{(4)} follow directly from \citep[ Lemmas 4–6 and Corollary 1]{kuba2021trust}. In particular, these results give continuity (and, in fact, Lipschitz continuity where stated) of $d_M^\pi$, $Q^\pi$, and the derived quantities, as well as continuity of cross-policy value functionals.
Let $\|\mathbf{H}_k\|_F$ denote the Frobenius norm of the client-$k$ heterogeneity matrix. For concreteness, write
\[
\|\mathbf{H}_k\|_F
= \frac{1}{K} \sum_{n=1}^K \sum_{\bs}\sum_{\ba}
\Bigg(
\frac{d^{\pi_{\beta_n}}_{\widetilde{M}_n}(\bs)}{d^{\pi_{\beta_k}}_{\widetilde{M}_k}(\bs)}
A^{\pi_{\beta_n}}_{T_n}(\bs,\ba)
-
A^{\pi_{\beta_k}}_{T_k}(\bs,\ba)
\Bigg)^2,
\]
where the precise form of $\mathbf{H}_k$ is built from ratios of discounted visitation distributions and advantage functions across clients. By Item \emph{(1)}, $d^{\pi}_{M}$ varies continuously in $\pi$; by Item \emph{(3)}, so do the advantages $A^{\pi}$. Under the standing reachability assumption (so that $d^{\pi_{\beta_k}}_{\widetilde{M}_k}(\bs)>0$ on the relevant support), the ratio of two continuous maps with nonzero denominator is continuous. Finite sums and products preserve continuity, hence $\|\mathbf{H}_k\|_F$ is continuous in $\pi$. This proves \emph{(5)}.
Consider the local policy-improvement objective
\[
f(\pi)
= \lambda\, \mathbb{E}_{(\bs,\ba)\sim \mathcal{D}_k}\!\left[
\log \pi(\ba \mid \bs)\;
\exp\!\left(\frac{1}{\beta}\big(\hat{Q}(\bs,\bar{\pi}(\bs))-\hat{V}(\bs)\big)\right)
\right]
\;+\;
\mathbb{E}_{\bs\sim \mathcal{D}_k,\; \ba\sim \pi(\cdot\mid \bs)}
\big[\hat{Q}(\bs,\ba)\big].
\]
By Items \emph{(2)}–\emph{(4)}, the quantities $\hat{Q}$, $\hat{V}$, and any cross-policy evaluation such as $\hat{Q}(\bs,\bar{\pi}(\bs))$ vary continuously with $\pi$ (and with any auxiliary policy arguments). The functions $\exp(\cdot)$ and $\log(\cdot)$ are continuous on their domains; we work on the support where $\pi(\ba\mid \bs)>0$ (so $\log\pi(\ba\mid \bs)$ is well-defined). The first expectation is taken w.r.t.\ a fixed dataset distribution $\mathcal{D}_k$, hence it is a finite (or dominated) linear functional of a continuous integrand, and therefore continuous in $\pi$. For the second term, the map
\[
\pi \;\mapsto\; \mathbb{E}_{\ba\sim \pi(\cdot\mid \bs)}[\hat{Q}(\bs,\ba)]
= \sum_{\ba}\pi(\ba\mid \bs)\,\hat{Q}(\bs,\ba)
\quad(\text{or the corresponding integral})
\]
is continuous whenever $\hat{Q}$ is continuous and bounded on the relevant domain; linearity (or dominated convergence) then implies continuity after averaging over $\bs\sim\mathcal{D}_k$. The scalar weights $\lambda$ and $\beta$ are fixed (hence trivially continuous); if they are chosen adaptively as continuous functions of $\pi$, the conclusion remains unchanged. Consequently, $f(\pi)$ is continuous in $\pi$. This proves \emph{(6)}.
\end{proof}

We further show the lower bound of $J(\bar\pi^{t+1})- J(\bar\pi^{t})$.
\begin{lemma}
\label{lemma:improving}
    For all $t$, Algorithm FOVA is guaranteed to generate a monotonically improving global policy, \ie
\[
J(\bar\pi^{t+1})- J(\bar\pi^{t})\ge \frac{1}{K}\sum_k \eta_k, \quad
\forall t,\quad \eta_k>0, \forall k,
\]
with
$
{\;
\lambda \;\ge\;
\frac{\,w\big(C+\eta_k\big)\;+\;\sqrt{\,w^2\big(C+\eta_k\big)^2+\tfrac12 (l_k+2\sigma)^2\,\mathrm{KL}_1\,}\,}{\mathrm{KL}_1},
}
$
$
{\;
\beta \;\ge\;
\frac{\,\big(1-w\big)\big(C+\eta_k\big)\;+\;\sqrt{\,\big(1-w\big)^2\big(C+\eta_k\big)^2+\tfrac12 \sigma^2\,\mathrm{KL}_2\,}\,}{\mathrm{KL}_2},
}
$
$
\mathrm{KL}_1:=\KL\!\big(\bar\pi_k^{*}, \pi_k^{t+1}\big),
\mathrm{KL}_2:=\KL\!\big(\bar\pi_k^{*}, \pi_{\beta_k}\big),
$
$C=h_k\varepsilon_k+l_k D_t$,
and
$w\in[0,1]$.
\end{lemma}
\begin{proof}
Recall the definition
\[
\begin{aligned}
f_k
=&\ \lambda\,\KL\!\big(\bar\pi_k^{*}, \pi_k^{t+1}\big)
+\beta\,\KL\!\big(\bar\pi_k^{*}, \pi_{\beta_k}\big)
-l_k\,\mathrm{TV}\!\big(\pi_k^{t+1},\bar\pi^{\,t}\big)
-h_k\,\mathrm{TV}\!\big(\pi_k^{t+1},\pi_{\beta_k}\big)\\
&\ -2\sigma\,\mathrm{TV}\!\big(\bar\pi_k^{*},\pi_k^{t+1}\big)
-\sigma\,\mathrm{TV}\!\big(\bar\pi_k^{*},\pi_{\beta_k}\big).
\end{aligned}
\]
Suppose $\mathrm{TV}(\pi_k^{t+1},\pi_{\beta_k})=\varepsilon_k>0$ and both
$\KL(\bar\pi_k^{*}, \pi_k^{t+1})$ and $\KL(\bar\pi_k^{*},\pi_{\beta_k})$ are strictly positive.  
By the triangle inequality,
\[
\mathrm{TV}\!\big(\pi_k^{t+1},\bar\pi^{\,t}\big)
\le \mathrm{TV}\!\big(\pi_k^{t+1},\bar\pi_k^{*}\big)+\mathrm{TV}\!\big(\bar\pi_k^{*},\bar\pi^{\,t}\big).
\]
Writing $D_t:=\mathrm{TV}(\bar\pi_k^{*},\bar\pi^{\,t})$ and substituting $\mathrm{TV}(\pi_k^{t+1},\pi_{\beta_k})=\varepsilon_k$ into the definition of $f_k$, we obtain
\[
\begin{aligned}
f_k \ \ge\ &
\lambda\,\KL\!\big(\bar\pi_k^{*}, \pi_k^{t+1}\big)
+\beta\,\KL\!\big(\bar\pi_k^{*}, \pi_{\beta_k}\big)\\
&-(l_k+2\sigma)\,\mathrm{TV}\!\big(\bar\pi_k^{*},\pi_k^{t+1}\big)
-\sigma\,\mathrm{TV}\!\big(\bar\pi_k^{*},\pi_{\beta_k}\big)
-\big(h_k\varepsilon_k+l_k D_t\big).
\end{aligned}
\]
By Pinsker’s inequality $\mathrm{TV}(p,q)\le \sqrt{\tfrac12 \KL(p, q)}$ and Young’s inequality, for any $\epsilon_1,\epsilon_2>0$,
\[
\begin{aligned}
(l_k+2\sigma)\,\mathrm{TV}\!\big(\bar\pi_k^{*},\pi_k^{t+1}\big)
&\le \frac{\epsilon_1}{2}\,\KL\!\big(\bar\pi_k^{*}, \pi_k^{t+1}\big)
+\frac{(l_k+2\sigma)^2}{4\epsilon_1},\\[1ex]
\sigma\,\mathrm{TV}\!\big(\bar\pi_k^{*},\pi_{\beta_k}\big)
&\le \frac{\epsilon_2}{2}\,\KL\!\big(\bar\pi_k^{*}, \pi_{\beta_k}\big)
+\frac{\sigma^2}{4\epsilon_2}.
\end{aligned}
\]
Choosing $\epsilon_1=\lambda$ and $\epsilon_2=\beta$ yields
\begin{align}
f_k \ \ge\
\frac{\lambda}{2}\,\KL\!\big(\bar\pi_k^{*}, \pi_k^{t+1}\big)
+\frac{\beta}{2}\,\KL\!\big(\bar\pi_k^{*}, \pi_{\beta_k}\big)
-\Bigg(\frac{(l_k+2\sigma)^2}{4\lambda}
+\frac{\sigma^2}{4\beta}
+h_k\varepsilon_k+l_k D_t\Bigg).
\label{eq:star}
\end{align}
Let $C:=h_k\varepsilon_k+l_k D_t$. For any target margin $\eta_{k}>0$, condition
\begin{align}
\frac{\lambda}{2}\,\KL\!\big(\bar\pi_k^{*}, \pi_k^{t+1}\big)
+\frac{\beta}{2}\,\KL\!\big(\bar\pi_k^{*}, \pi_{\beta_k}\big)
\ \ge\ \frac{(l_k+2\sigma)^2}{4\lambda}
+\frac{\sigma^2}{4\beta}+C+\eta_{k}
\label{eq:tag1}
\end{align}
guarantees $f_k\ge\eta_{k}$ by Eq.~\eqref{eq:star}. To meet Eq.~\eqref{eq:tag1}, fix any weight $w\in[0,1]$, decompose $C+\eta_{k}$ as
\[
a:=w(C+\eta_{k}),\qquad b:=(1-w)(C+\eta_{k}),
\]
and impose the pair of inequalities
\begin{align}
\frac{\lambda}{2}\,\KL\!\big(\bar\pi_k^{*}, \pi_k^{t+1}\big)\ \ge\ \frac{(l_k+2\sigma)^2}{4\lambda}+a,\qquad
\frac{\beta}{2}\,\KL\!\big(\bar\pi_k^{*}, \pi_{\beta_k}\big)\ \ge\ \frac{\sigma^2}{4\beta}+b.
\label{eq:tag2}
\end{align}
Each inequality in Eq.~\eqref{eq:tag2} is quadratic in the corresponding parameter. Solving the first inequality yields
\[
\lambda \ \ge\
\frac{a+\sqrt{\,a^2+\tfrac12 (l_k+2\sigma)^2\,\KL(\bar\pi_k^{*}, \pi_k^{t+1})}}{\KL(\bar\pi_k^{*}, \pi_k^{t+1})}.
\]
Similarly, the second inequality is satisfied for
\[
\beta \ \ge\
\frac{b+\sqrt{\,b^2+\tfrac12 \sigma^2\,\KL(\bar\pi_k^{*},\pi_{\beta_k})}}{\KL(\bar\pi_k^{*},\pi_{\beta_k})}.
\]
Thus, by selecting any $w\in[0,1]$ and choosing $\lambda,\beta$ not smaller than the respective minimal values above, condition (1) holds and hence
\[
f_k \ \ge\ \eta_{k}\ >0.
\]
Because both $\KL(\bar\pi_k^{*}, \pi_k^{t+1})$ and $\KL(\bar\pi_k^{*},\pi_{\beta_k})$ are assumed strictly positive, the minimal feasible $\lambda,\beta$ given above are finite. Therefore, one can always tune $(\lambda,\beta)$ to guarantee $f_k\ge\eta_{k}$ for any desired $\eta_{k}>0$, in particular ensuring $f_k>0$.
\end{proof}

Next, we show the convergence of policy updates in Algorithm FOVA. 

\begin{theorem}
Suppose Assumption~\ref{assu:cql-heter} and Assumption~\ref{ass:feasible-update} hold, with parameters chosen as in Lemma~\ref{lemma:improving}. 
Then the algorithm converges in terms of the performance values $\{J(\bar\pi^{t})\}_{t\ge0}$. 
Moreover, every accumulation point of the policy sequence $\{\bar\pi^{t}\}$ corresponds to an optimal solution.
\label{thm:convergence}
\end{theorem}


\begin{proof}
By Lemma~\ref{lemma:improving}, for all $t$ there exist admissible local parameters that yield per-round client improvements $\{\eta_k^{(t)}\}_{k=1}^K$ with
\[
J(\bar\pi^{t+1}) \;\ge\; J(\bar\pi^{t}) + \frac{1}{K}\sum_{k=1}^K \eta_k^{(t)} \;\ge\; J(\bar\pi^{t}),
\]
so $\{J(\bar\pi^{t})\}$ is nondecreasing. Using $\widetilde R_k \le \widetilde R_k^{\max}$ and $\sum_{t\ge 0}\gamma^t = (1-\gamma)^{-1}$, for any $\pi$,
\[
J(\pi) \;\le\; \frac{1}{K}\sum_{k=1}^K \sum_{t=0}^\infty \gamma^t \widetilde R_k^{\max}
\;=\; J_{\max}.
\]
Thus, $\{J(\bar\pi^{t})\}$ is bounded above and hence converges to some $J^\star\le J_{\max}$. Denote $\Delta_t := J(\bar\pi^{t+1})-J(\bar\pi^{t})\ge 0$; by convergence we have $\Delta_t\to 0$.
By Lemma~\ref{lem:local-compact} each $\Pi_k$ is compact (and nonempty), and by Lemma~\ref{lem:global-compact} $\Pi=\prod_{k=1}^K\Pi_k$ is compact. Assumption~\ref{ass:feasible-update} ensures $\bar\pi^{t}\in\Pi$ for all $t$ when initialized with $\bar\pi^0\in\Pi$, so $\{\bar\pi^{t}\}\subset\Pi$ admits accumulation points. Let $\bar\pi^\star$ be any accumulation point and suppose, for contradiction, that it is not a solution. Then at $\bar\pi^\star$ there exists an admissible update that achieves a strict improvement $\delta>0$. By Lemma~\ref{lemma:continuity_properties} (continuity of the improvement value in $\bar\pi$), there is a neighborhood $U$ of $\bar\pi^\star$ and a constant $\underline\eta\in(0,\delta)$ such that for every $\bar\pi\in U$, applying the same improvement construction yields
$
\frac{1}{K}\sum_{k=1}^K \eta_k \;\ge\; \underline\eta.
$
Since $\bar\pi^\star$ is an accumulation point, $\bar\pi^{t}\in U$ for infinitely many $t$, hence $\Delta_t\ge \underline\eta$ infinitely often, contradicting $\Delta_t\to 0$. Therefore every accumulation point is a solution.
In particular, if the solution set is a singleton $\{\bar\pi^\star\}$, then the whole sequence $\{\bar\pi^{t}\}$ converges to $\bar\pi^\star$.
Finally, if there exists a uniform $\underline\eta>0$ such that at every round one can select admissible local parameters with $\eta_k^{(t)}\ge \underline\eta$, $\forall k$,
\[
J(\bar\pi^{t+1})-J(\bar\pi^{t}) \;\ge\; \underline\eta >0 \quad \forall t.
\]
We see that the process must end in at most
\[
N \;\le\; \frac{J_{\max}-J(\bar\pi^{0})}{\underline\eta}
\]
rounds; otherwise $J(\bar\pi^{t})$ would eventually exceed $J_{\max}$, a contradiction.
\end{proof}
\begin{remark}
Theorem~\ref{thm:convergence} highlights the practical implication that the algorithm converges once the regularization parameters $(\lambda,\beta)$ are chosen sufficiently large. 
Intuitively, this condition enforces that the induced global policy from local updates remains sufficiently close to the local behavioral datasets (measured in KL divergence), and that local policies stay close to the induced global policy. 
Under these proximity constraints, the iterative process guarantees monotone improvement of $J(\bar\pi^t)$ and thus ensures convergence.
\end{remark}

}

\end{document}